\documentclass[nonblindrev]{informs3} 

\OneAndAHalfSpacedXI

\usepackage[colorlinks,
linkcolor=red,
anchorcolor=blue,
citecolor=blue
]{hyperref}
\usepackage{smile}
\usepackage{tikz}

\def\A{\textsf{A}}
\def\B{\textsf{B}}
\def\r{\textsf{r}}
\def\Q{\textsf{Q}}
\def\ci{\textsf{conf}}
\def\gap{\textsf{Gap}}

\def\bbeta{\bm{\beta}}
\def\btheta{\bm{\theta}}
\def\bvartheta{\bm{\vartheta}}
\def\bff{\bm{f}}
\def\ps{\textsf{ps}}
\newcommand{\indp}{\perp\!\!\!\!\perp} 
\newcommand{\notindp}{\not\!\perp\!\!\!\perp}

\usepackage{natbib}
 \bibpunct[, ]{(}{)}{,}{a}{}{,}%
 \def\bibfont{\small}%
\TheoremsNumberedThrough     
\ECRepeatTheorems

\EquationsNumberedThrough    

\MANUSCRIPTNO{MS-0001-1922.65}

\begin{document}



\RUNTITLE{Fu et al.}

\TITLE{Offline Reinforcement Learning for Human-Guided Human-Machine Interaction with Private Information}

\ARTICLEAUTHORS{%
\AUTHOR{Zuyue Fu}
\AFF{Department of Industrial Engineering and Management Sciences, Northwestern University, \EMAIL{zuyuefu2022@u.northwestern.edu}} 
\AUTHOR{Zhengling Qi}
\AFF{Department of Decision Sciences, George Washington University, \EMAIL{qizhengling@email.gwu.edu}}
\AUTHOR{Zhuoran Yang}
\AFF{Department of Statistics and Data Science, Yale University, \EMAIL{zhuoran.yang@yale.edu}}
\AUTHOR{Zhaoran Wang}
\AFF{Department of Industrial Engineering and Management Sciences, Northwestern University, \EMAIL{zhaoranwang@gmail.com}}
\AUTHOR{Lan Wang}
\AFF{Department of Management Science, University of Miami, \EMAIL{lanwang@mbs.miami.edu}}
} 

\ABSTRACT{%
Motivated by the human-machine interaction such as training chatbots for improving customer satisfaction, we study human-guided human-machine interaction involving private information.  We model this interaction as a two-player turn-based game, where one player (Alice, a human) guides the other player (Bob, a machine) towards a common goal. Specifically, we focus on offline reinforcement learning (RL) in this game, where the goal is to find a policy pair for Alice and Bob that maximizes their expected total rewards based on an offline dataset collected a priori. The offline setting presents two challenges: (i) We cannot collect Bob's private information, leading to a confounding bias when using standard RL methods, and (ii) a distributional mismatch between the behavior policy used to collect data and the desired policy we aim to learn. To tackle the confounding bias, we treat Bob's previous action as an instrumental variable for Alice's current decision making so as to adjust for the unmeasured confounding.  We develop a novel identification result and use it to propose a new off-policy evaluation (OPE) method for evaluating policy pairs in this two-player turn-based game.  To tackle the distributional mismatch,  we leverage  the idea of pessimism and use our OPE method to develop an off-policy learning algorithm for finding a desirable policy pair for both Alice and Bob.  Finally, we prove that under mild assumptions such as partial coverage of the offline data, the policy pair obtained through our method converges to the optimal one at a satisfactory rate. 

}%



\maketitle

\section{Introduction}\label{sec: intro}

Data-driven machines using algorithms have become more and more prevailing throughout daily activities and business decision making. In recent years, there is a surging interest in studying the interaction between human and machines, which is essential in improving the use of machine learning for solving real-world problems \citep{hoc2000human, gorecky2014human}. Since machines rely on humans to create and control, a central question is how humans can leverage their knowledge to achieve a better decision-making under the coexistence of humans and machines. In this paper, we study a two-player turn-based cooperative game as an example of human-machine interaction. Leveraging the pre-collected offline data, our goal is to develop a data-driven method to learn an optimal policy when some private information of one particular player cannot be accessed during the training procedure.

Our motivation mainly comes from how to improve the customer satisfaction in machine-driven communication, e.g., Chat-GPT \citep{chatgpt}. To train a chatbot for solving customers' problems, it often requires enormous data with tremendous human efforts. Efficient interplay between human and machine is necessary to improve the performance of chatbot. Therefore it is important for humans to find the most useful questions (decisions) to improve the knowledge of chatbot. However, due to the blackbox machine learning models used in the chatbot, human may not access the private information (i.e., ``weakness") of this machine-driven communication system. Humans must rely on the feedback given by the machine to learn its weakness and find most suitable questions to further improve the capability of the chatbot and ultimately use this machine to improve the customer satisfaction. Such an important business and machine learning problem can be generalized as a two-player turn-based cooperative game with private information. Our work is also  motivated by the strategic reaction of humans to algorithms. As in many business activities, algorithms play an important role in providing business solutions such as auction, pricing and trading. For example, hedge fund relies on machine to collect market data and provide real-time suggestions on trading stock, based on which traders make the final decision. The return of traders' decisions will further feed the machine to improve its learning ability. Therefore it is essential to study how to improve the algorithm's ability in discovering important pattern from big data and lead to a better decision-making as a human. This is related to the aforementioned two-player turn-based game as well.

\subsection{Brief Problem Statement and Challenges}
In this paper, we formulate a two-player game as a reinforcement learning (RL) problem, where the machine as a black box may have private information that humans may not understand; human as a complex organism may also have private information that we fail to deliver to the machine. The forbidden use of the private information is the main obstacle
for a successful human-machine interaction. 
Specifically, there are two main sources that the private information affects the game structure: (i) actions taken by the players; (ii) rewards collected and the next states received after taking actions. 
Such a game with private information has been widely considered in practice: See examples in economics \citep{laffont2009theory}, social science \citep{sabater2005review}, robotics/control \citep{farinelli2004multirobot, yan2013survey,mertikopoulos2019learning}, cognitive science \citep{sun2006cognition}, human-machine negotiation \citep{lewis2017deal}, etc.

Specifically, we model such a scenario as a two-player turn-based game, where we call two players as Alice (representing a human) and Bob (representing a machine). In such a game, Alice first takes an action, which may potentially depend on the current state, her private information, and Bob's previous action. Then Alice collects her reward and reveals her public information to Bob and transits to the next state. After observing the new state, Bob takes an action, which may potentially depend on the new current state, his private information, and Alice's previous action. Similarly, Bob collects his reward and reveals his public information to Alice and transits to the next state. The goal for Alice and Bob is then to maximize the sum of their expected total rewards. In other words, we (as Alice) aim to find a policy pair for Alice and Bob to maximize their rewards together. We remark that such a turn-based game exactly characterizes a typical human-guided human-machine interaction, where the Alice (as a human) first takes an action to guide Bob (as a machine) how to behave and collect reward, and Bob potentially ``mimics'' Alice's behavior to take an action to maximize the sum of their rewards. 


In this work, 
we aim to solve an offline RL problem in a two-player turn-based game, where the goal is to learn the optimal policy pair of the principal based on an offline dataset collected a priori. 
The challenge in the offline setting is that Alice, as a human, though can access her own private information, cannot access the private information of Bob, as a machine. As a result, Bob's private information is the so-called unobserved confounder in this two-player turn-based game. If the existence of unmeasured confounding is not carefully factored in the design of offline learning algorithms, directly applying standard RL methods could potentially lead to biased and
inconsistent estimation of value functions and, consequently learn sub-optimal policies. 
In the meanwhile, as our offline data follow a fixed behavior policy and we cannot interact with the environment to further collect more data, it may result in a distributional mismatch between the behavior policy and the optimal one. This is the fundamental challenge behind offline RL \citep{levine2020offline}. 

\subsection{Main Results and Contributions}

To address the above two challenges, we \textit{theoretically} study the offline RL problem in such a two-player turn-based game using the instrumental variable (IV) method \citep{angrist1995identification}, which has been widely used in the literature of causal inference (e.g., \cite{pearl2009causality,hernan2010causal}) to identify the causal effect of a treatment under unmeasured confounding. In this work, we innovatively treat Bob's previous action as an instrumental variable (IV) for Alice's current decision making so as to adjust the unmeasured confounding. Such an instrumental variable arises naturally since (i) given all the information of the current stage, Bob's previous action does not depend on his current private information; (ii) Bob's previous action will affect Alice's current action. Condition (i) is usually referred to as IV independence and Condition (ii) is referred to as IV relevance, which are two commonly required validity conditions for an IV. Moreover, a valid IV additionally requires that there is no direct effect of Bob's previous action on the reward obtained by Alice at the current decision-making unless through Alice's action itself. This is called IV exclusion restriction in the causal inference literature \citep{wang2018bounded}. However, such a requirement could be easily violated in our two-player turn-based game. For example, a chatbot's response (i.e., Bob's action) to the question would directly affect the  customer satisfaction (Alice's reward). Motivated by these, we treat Bob's previous action as \textit{an invalid IV}, where the IV exclusion restriction fails to hold. Relying on the orthogonal conditions on the reward and action's model structure to some extent, we establish a novel identification result using offline data for evaluating policies and learning an optimal one in our two-player turn-based game with private information. 

Our new identification results lead to an estimation problem of solving a system of nonlinear integral equations for evaluating policies in two-player turn-based game. We propose to use a sieve minimum distance estimator, which is shown to achieve the known minimax optimal rate for the non-parametric mean instrumental variable regression \citep{chen2012estimation}. In terms of policy learning, to mitigate the distributional mismatch, instead of directly plugging in the sieve minimum distance estimator, we propose to use conservative estimators as proxies of the action-value functions and maximize such conservative action-value functions to learn an optimal policy pair. Such a method is also termed pessimism in related literature \citep{jin2021pessimism, lu2022pessimism, fu2022offline}. Theoretically, we provide a finite sample regret analysis of our offline learning algorithm, showing that the learned policy pair converges to the optimal one asymptotically as the
sample size of the offline data increases under only the partial coverage assumption, which does not require our offline data to explore the state-action space fully. 


\vskip5pt
\noindent\textbf{Main Contributions.}
Our contribution is threefold. First, to the best of our knowledge, we are among the first in the literature to study offline policy learning in a two-player turn-based game under unmeasured confounding through the lens of causal inference. This serves as a first step to bridge the gap between causal inference and the cooperative game. 

Second, to address the unmeasured confounding (i.e., private information of Bob) during the game, we leverage the structure of the turn-based game and propose to use the other player's action in the previous step as a potential IV for adjusting the confounding bias resulted from the missing private information of Bob. Since we allow the previous action to have a direct effect on the current reward, standard IV approaches in the literature of causal inference cannot be applied. Motivated by the literature on the causal inference with an invalid IV \citep[e.g.,][]{lewbel2012using,lewbel2018identification,tchetgen2021genius}, we first study a standard causal inference setting and generalize the existing identification results of invalid IVs, which serves as a foundation for solving our two-player turn-based game. Specifically, in the causal inference under unmeasured confounding, where an IV can have a direct effect on the outcome, the identification is often restricted to the marginal causal effects of the treatment and IV. Nevertheless, in our problem setting, at each decision-point, there may exist an interaction effect between Alice's and Bob's actions on the reward, making existing results hard to apply.
Motivated by this, we consider a more general setting where there is an interaction effect between the treatment and IV on the outcome, and establish a novel non-parametric identification result for identifying both marginal and interaction effects of
the action and IV. Our result relies on the orthogonal restriction between the effect of unmeasured confounders on the treatment and that on the outcome up to the secord-order information, which could be of independent interest. We then adapt these novel identification results in our two-player turn-based game setting and propose a sieve minimum distance estimator for off-policy evaluation, which is shown to achieve the known minimax optimal rate for the standard non-parametric instrumental variable regression.


Last but not least, by our off-policy evaluation procedure and leveraging the recent idea of using pessimism in the policy learning, we propose a policy pair estimator for both human and machine, which is shown to converge to the optimal policy pair under mild assumptions such as the partial coverage of our offline data. The proposed algorithm addresses the issue of the distributional mismatch. To the best of our knowledge, this is the first established regret result for offline policy learning in a two-player turn-based game under unmeasured confounding.



\vskip5pt
\noindent\textbf{Related Work.} 
Our work is related to the line of works that study RL under the presence of unobserved confounders. \cite{chen2021estimating} utilize a time-varying instrumental variable to establish a novel estimation framework for DTR. Also, to incorporate the observational data into the finite-horizon RL, \cite{wang2021provably} propose an online method termed deconfounded optimistic value iteration. To ensure the identifiability through offline data, they impose a backdoor criterion \citep{pearl2009causality, peters2017elements} when confounders are partially observed, and a frontdoor criterion when confounders are unobserved. Our work is also closely related to \cite{liao2021instrumental} and \cite{fu2022offline}. Specifically, \cite{liao2021instrumental} propose an IV-aided value iteration algorithm to study confounded MDPs in the offline setting. It is worth mentioning that they only consider a parametric model, where the transition dynamics is a linear function of some known feature map. To ensure identifiability, they assume that the unobserved confounders are Gaussian noise. In contrast, we consider a nonparametric saturated model without such restrictive assumptions, which is thus more general and applicable. On the other hand, \cite{fu2022offline} establish point identification results with the help of valid IVs in an infinite-horizon confounded MDP. As discussed before, standard IV approaches cannot solve our problem as there exists a direct effect of IV on the outcome/reward.
Besides the aforementioned literature, there is a growing body of literature using causal inference tools to study off-policy evaluation in confounded MDPs, see \cite{kallus2020confounding, namkoong2020off, bennett2021off,shi2021minimax,lu2022pessimism,miao2022off} and the references therein. Also we highlight that the aforementioned works all focus on the single-agent setting, while our work focus on the two-player setting. 

Our work is also related to the line of works that study two-player games. The majority of this line of works focus on the two-player zero-sum stochastic games. 
See \citet{littman2001value,perolat2015approximate,perolat2018actor,srinivasan2018actor,guo2021decentralized,lagoudakis2012value,perolat2016softened,perolat2016use,luo2019accelerating,jin2020provably,cai2020provably,zhong2021can,hambly2021policy} and the reference therein.  
We remark that all of these works consider the settings where the agents have complete information, and thus do not face the issue of unobserved confounders as in our work. 

The rest of our paper is organized as follows. In \S\ref{sec:model}, we describe the offline reinforcement learning problem for the two-player turn-based game. In \S\ref{sec:id-all}, we study a causal inference problem under unmeasured confounding, where an invalid IV is used to identify the causal effects of both treatment/action and IV on the outcome. In \S\ref{sec:uiehuier}, we adapt the identification result developed in \S\ref{sec:id-all} to a single-stage two-player game and propose an offline data-driven algorithm to compute an optimal policy pair for Alice and Bob. We then generalize our method for solving a multi-stage game in \S\ref{sec: multi-game}. Lastly, we conclude the paper in \S\ref{sec: conclusion}. All technical proofs can be found in the Appendix.

\vskip5pt
\noindent\textbf{Notations.} We denote by $[H] = \{1, 2, \ldots, H\}$ and $[H]_{1/2} = \{3/2, 5/2, \ldots, H+1/2\}$. We denote by $\sigma_{\min}(\Sigma)$ the minimum singular value for any matrix $\Sigma$.

\section{Problem Formulation}\label{sec:model}

We consider a data-driven two-player turn-based game with private information, where two players called Alice and Bob play the game sequentially based on the observed states and rewards and their private information. 
Let $h$ be some positive integer. 
We assume that starting from $h=1$, at the $h$-th step of the game, Alice takes an action $A_h$ following a policy $\pi^\A_h$, which possibly depends on the current state $S_h$, her private information $U_h$, and Bob's previous action $B_{h-1/2}$. After taking such an action, Alice receives a reward $R_h^\A$ and the (next) state $S_{h+1/2}$ is revealed to both players. In the meanwhile, Alice and Bob receive their new private information $U_{h+1/2}$ and $V_{h+1/2}$, respectively.
Then at the $(h + 1/2)$-th step, Bob takes an action $B_{h+1/2}$ following a policy  $\pi^\B_{h+1/2}$ which possibly depends on the current state $S_{h+1/2}$, his private information $V_{h+1/2}$, and Alice's previous action $A_h$.  After taking such an action, Bob receives a reward $R_{h+1/2}^\B$ and the state $S_{h+1}$ is revealed to both players. {In the meanwhile, Alice and Bob receive their new private information $U_{h+1}$ and $V_{h+1}$, respectively.}
We consider a MDP with a finite horizon $H$ (i.e., $h=1, \cdots, H$) to model the trajectory $\{S_{h}, U_{h}, V_{h}, A_{h},R^\A_{h}, S_{h+1/2}, U_{h+1/2}, V_{h+1/2}, B_{h+1/2}, R^\B_{h+1/2}\}_{h \in [H]}$, where the MDP terminates once the state $S_{H+1}$ is revealed. 
For convenience, we assume that Bob reveals an action $B_{1/2}$ following $\pi^\B_{1/2}$ before the state $S_1$ is revealed. Moreover, for the simplicity of presentation, we only consider binary actions, i.e., $\cA = \cB = \{0,1\}$, where $\cA$ and $\cB$ are action spaces for Alice and Bob, respectively. 
Further, we denote by $\cS$, $\cU$ and $\cV$ the state space and private information spaces, respectively.  We remark that our framework can be easily extended to the multiple-action scenario as seen later. 

In this paper, we consider a fully cooperative game of Alice and Bob. Specifically, if all information is available at each time-step, then the goal of the two-player turn-based game considered here is to find a policy pair that {maximizes the sum of Alice's and Bob's expected total rewards, as we will formally introduce in \S\ref{sec:val}}. As discussed in the \S\ref{sec: intro}, we consider the scenario such as training a chatbot or algorithmic trading with the help of traders, where human interacts with the machine to achieve a common goal. In such an application, Alice plays a role of a human, while Bob plays a role of a machine. Our goal is to leverage the offline data collected from the past experience to help Alice and Bob learn a good policy pair that maximizes the total expected rewards. The challenge behind this task is the existence of the private information owned by Bob that hinders Alice from learning a desirable policy. Without accessing such private information, the offline data may not ensure the identification of the policy value and thus cannot guarantee to learn a good policy. In this paper, we develop a novel IV approach to address such a challenge. Finally, we assume that the private information $V_h$ and $V_{h+1/2}$ for $h \in [H]$ are memoryless so that given other observed variables at the same decision point, private information is independent of the past variables. Memoryless assumption has been widely used in the literature \citep[e.g.,][]{kallus2018confounding,shi2022off,fu2022offline}. In this case, the observed sequence by Alice can also be modelled by MDP.

\subsection{Value Functions}\label{sec:val}
We use the policy value to evaluate the performance of any policies. Specifically, for any policy pair $\pi = (\pi^\A, \pi^\B)$, and the timestamp $h\in [H]$, we denote by $Q_h^{\A,\pi}(S_h, A_h , B_{h-1/2}, U_h, V_h)$ and $Q_{h+1/2}^{\A,\pi}(S_{h+1/2}, A_h , B_{h+1/2}, U_{h+1/2}, V_{h+1/2})$ the action-value functions of Alice for the expected cumulative rewards starting at times $h$ and $h +1/2$, i.e.,  
\$
& Q_h^{\A,\pi}(S_h, A_h , B_{h-1/2}, U_h, V_h) = \EE_\pi\left[\sum_{j=h}^{H} R_j^\A \Biggiven S_h, A_h, B_{h-1/2}, U_h, V_h \right], \\
& Q_{h+1/2}^{\A,\pi}(S_{h+1/2}, A_h , B_{h+1/2}, U_{h+1/2}, V_{h+1/2}) = \EE_\pi\left[\sum_{j=h+1}^{H} R_j^\A \Biggiven S_{h+1/2}, A_h, B_{h+1/2}, U_{h+1/2}, V_{h+1/2} \right], 
\$
where the expectation $\EE_\pi[\cdot]$ are taken following the policy pair $\pi = (\pi^\A, \pi^\B)$. 
Similarly, we define the action-value functions of Bob at times $h$ and $h+1/2$ as $Q_h^{\B,\pi}(S_h, A_h , B_{h-1/2}, U_h, V_h)$ and $Q_{h+1/2}^{\B,\pi}(S_{h+1/2}, A_h , B_{h+1/2}, U_{h+1/2}, V_{h+1/2})$ respectively.  Finally, we denote by 
\$
& J^\A(\pi^\A, \pi^\B) = \EE_\pi\left[ Q_1^{\A,\pi}(S_1, A_1, B_{1/2}, U_1, V_1) \right], \qquad J^\B(\pi^\A, \pi^\B) = \EE_\pi\left[ Q_1^{\B,\pi}(S_1, A_1, B_{1/2}, U_1, V_2) \right]
\$
the expected total rewards (i.e., policy value) for Alice and Bob, respectively. 

In a fully cooperative turn-based game of Alice and Bob, we aim to solve the following optimization problem together, 
\#\label{eq:opt-prob}
& \max_{\pi^\A\in \Pi^\A, \pi^\B\in \Pi^\B} J(\pi^\A, \pi^\B), \quad \text{where } J(\pi^\A, \pi^\B) = J^\A(\pi^\A, \pi^\B) + J^\B(\pi^\A, \pi^\B).
\#
Since we are on the side of Alice, we consider Alice's policy space $\Pi^\A$ consisting of policies $\pi^\A=\{\pi_h^\A\}_{h\in [H]}$ that depends on the current state $S_h$, Bob's previous action $B_{h-1/2}$, and Alice's private information $U_h$; while the policy space of Bob $\Pi^\B$ consists of policies $\pi^\B=\{\pi_{h+1/2}^\B\}_{h\in [H]}$ that only depends on the current state $S_{h+1/2}$ and Alice's previous action $A_h$. Recall that we only construct or execute policies using the knowledge of Alice without accessing Bob's private information. 
See Figure \ref{fig:frwgegrg} for a graphical illustration of the two-player turn-based game described in \S\ref{sec:model}, where we aim to solve in \eqref{eq:opt-prob}.

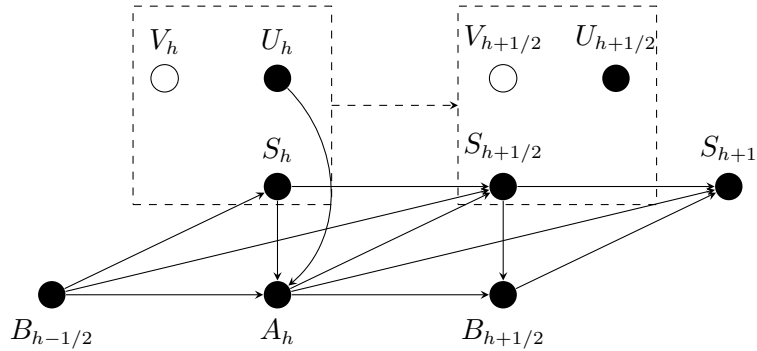
\begin{figure}[htbp]
    \centering
\begin{tikzpicture}[scale=1.2]
        \node (bo) at (-1.5,0) [label=below:{$B_{h-1/2}$}, circle, fill=black]{};
        \node (a) at (1,0) [label=below:{$A_h$}, circle, fill=black]{};
        \node (s) at (1,1.2) [label=above:{$S_h$}, circle, fill=black]{};
        \node (u) at (1,2.4) [label=above:{$U_h$}, circle, fill=black]{};
        \node (v) at (-0.25,2.4) [label=above:{$V_h$}, circle, draw]{};
        \node (bn) at (3.5,0) [label=below:{$B_{h+1/2}$}, circle, fill=black]{};
        \node (sn) at (3.5,1.2) [label=above:{$S_{h+1/2}$}, circle, fill=black]{};
        \node (un) at (4.75,2.4) [label=above:{$U_{h+1/2}$}, circle, fill=black]{};
        \node (vn) at (3.5,2.4) [label=above:{$V_{h+1/2}$}, circle, draw]{};
        \node (snn) at (6,1.2) [label=above:{$S_{h+1}$}, circle, fill=black]{};
    \draw[-stealth] (bo) edge (s) edge (a) edge (sn);
    \draw[-stealth] (s) edge (a) edge (sn);
    \draw[-stealth] (a) edge (sn) edge (snn) edge (bn);
    \draw[-stealth] (sn) edge (bn) edge (snn);
    \draw[-stealth] (bn) edge (snn);
    \draw[-stealth] (u) [out=-45,in=40] edge (a);
    \draw[dashed] (-0.6,1) rectangle ++(2.2,2.2);
    \draw[dashed] (3.0,1) rectangle ++(2.2,2.2);
    \draw [-stealth, dashed](1.6,2.1) -- (3,2.1);
\end{tikzpicture}
\caption{A graphical illustration of a two-player turn-based game described in \S\ref{sec:model} that we aim to solve in \eqref{eq:opt-prob}. 
Here $U_h$ is the private information of Alice, which can be used as an input of the policy $\pi^\A_h$ when solving \eqref{eq:opt-prob}; while $V_{h+1/2}$ is the private information of Bob, which cannot be used as an input of the policy $\pi^\B_{h+1/2}$ when solving \eqref{eq:opt-prob}. In the meanwhile, the dependence of $(S_{h+1/2}, U_{h+1/2}, V_{h+1/2})$ on $(S_h, U_h, V_h)$ is characterized by the dashed arrow in the figure for the simplicity of the presentation.}
\label{fig:frwgegrg}
\end{figure}

\subsection{Data Collection Process}
We aim to leverage the offline data collected from the previous experience to solve Problem \eqref{eq:opt-prob}. In this subsection. we describe how the offline data are collected in a two-player turn-based game as follows. For any $h\in [H]$, at the $h$-th step, Alice observes the state $S_h$ and her private information $U_h$, and then takes an action $A_h$ following a behavior policy $b_h^\A$ and receives a reward $R_h^\A$. At the $(h+1/2)$-th step, Bob observes the state $S_{h+1/2}$ and his hidden state $V_{h+1/2}$, and then takes an action $B_{h+1/2}$ following a behavior policy $b_{h+1/2}^\B$ and receives and reveals a reward $R_{h+1/2}^\B$. The game terminates once the state $S_{H+1}$ is revealed to both players. In this way, we can collect data as the form of $\cD = \{S_h, U_h, R_h^\A, S_{h+1/2}, B_{h+1/2}, R_{h+1/2}^\B\}_{h\in[H]}\cup\{B_{1/2}, S_{H+1}\}$ from Alice's position. We assume that we in total collect $n$ independent and identically distributed copies of the form in $\cD$. 

We consider the following model for such a two-player turn-based game. Specifically, we impose the structure of the action and reward for each player and the structure of state transitions at the $h$-th step as follows. 

\vskip5pt
\noindent\textbf{Alice's model.}
At the $h$-th step, the \textit{saturated} model for Alice takes the following form, 
\#\label{eq:alice}
& \EE\left[ A_h  \given S_h, B_{h-1/2}, U_h, V_h \right] = \alpha_z^\A(U_h, V_h, S_h) B_{h-1/2} + \alpha_{u,s}^\A(U_h, V_h, S_h), \\
& \EE\left[ R_h^\A \given S_h, A_h , B_{h-1/2}, U_h, V_h \right] = \beta_a^\A(U_h, V_h, S_h) A_h  + \beta_z^\A(U_h, V_h, S_h)  B_{h-1/2} \\
& \qquad\qquad\qquad\qquad\qquad\qquad\qquad + \beta_{az}^\A(U_h, V_h, S_h) A_h   B_{h-1/2} + \beta_{u,s}^\A(U_h, V_h, S_h).  
\#
Here in the structural formulation of Alice's action, the function $\alpha_z^\A(U_h, V_h, S_h)$ measures the impact of Bob's previous $B_{h-1/2}$ on Alice's policy, while $\alpha_{u,s}^\A(U_h, V_h, S_h)$ measures other effects from the state and all related private information. We remark that the functions $\alpha_z^\A$ and $\alpha_{u,s}^\A$ are determined by Alice's behavior policy $b^\A_h$.  In the meanwhile, in the structural formulation of Alice's reward, $\beta_a^\A(U_h, V_h, S_h)$ measures the impact completely from Alice's action, $\beta_z^\A(U_h, V_h, S_h)$ measures the impact completely from Bob's previous action, $\beta_{az}^\A(U_h, V_h, S_h)$ measures the impact from the interaction of  Bob's previous and Alice's actions, while the last term $\beta_{u,s}^\A(U_h, V_h, S_h)$ measures the random effect from the state and all related private information. In our applications such as training a chatbot by the human or algorithmic trading, it is important to consider the interaction effect between Alice's action and Bob's action. Take the chatbot as an example. A customer's overall satisfaction is clearly affected by how she asks the questions and whether the machine solves the questions correctly, which can be quantified by the interaction effect to some extent. Another example is related to the algorithmic trading. When the trader (i.e., Alice) wants to leverage the option flow (i.e., Bob's action), it is important to consider the joint effect of both players' actions on the final return if they are institutional investors. 
We remark that because the action spaces are binary, all models can be written in the form of Model \eqref{eq:alice}. Thus our Model \eqref{eq:alice} is saturated and \textit{general}.

\vskip5pt
\noindent\textbf{State Transition from $h$ to $h+1/2$.}
We assume that
\#\label{eq:model-trans-h}
& \EE\left[ f(S_{h+1/2}, U_{h+1/2}, V_{h+1/2}) \given S_h, A_h , B_{h-1/2}, U_h, V_h \right]   \\
& \qquad = \theta_h^f(U_h, V_h, S_h) A_h  + \gamma_h^f(U_h, V_h, S_h) B_{h-1/2} + \omega_h^f(U_h, V_h, S_h) A_h  B_{h-1/2} + \iota_h^f(U_h, V_h, S_h) 
\#
for any function $f\colon \cS\times \cU \times \cU\to \RR$. In such a formulation, $\theta_h^f(U_h, V_h, S_h)$ measures the impact on the future state-transition completely from Alice's action, $\gamma_h^f(U_h, V_h, S_h)$ measures the impact completely from Bob's previous action, $\omega_h^f(U_h, V_h, S_h)$ measures the impact from the interaction of  Bob's previous and Alice's actions, while the last term $\iota_h^f(U_h, V_h, S_h)$ measures the random effect from the state and all related private information. Again we emphasize that the state-transition Model \eqref{eq:model-trans-h} is \textit{general} as our action spaces are all binary. Similar as Alice's model, we consider the following Bob's model.
\vskip5pt
\noindent\textbf{Bob's model.}
At the $(h + 1/2)$-th step, the model for Bob takes the following form, 
\#
& \EE\left[ B_{h+1/2} \given S_{h+1/2}, A_h , U_{h+1/2}, V_{h+1/2} \right]   \\
& \qquad = \alpha_z^\B(U_{h+1/2}, V_{h+1/2}, S_{h+1/2}) A_h + \alpha_{u,s}^\B(U_{h+1/2}, V_{h+1/2}, S_{h+1/2}), \\
& \EE\left[ R_{h+1/2}^\B \given S_{h+1/2}, A_h , B_{h+1/2}, U_{h+1/2}, V_{h+1/2} \right]   \\
& \qquad = \beta_a^\B(U_{h+1/2}, V_{h+1/2}, S_{h+1/2}) B_{h+1/2} + \beta_z^\B(U_{h+1/2}, V_{h+1/2}, S_{h+1/2}) A_h  \\
& \qquad \qquad + \beta_{az}^\B(U_{h+1/2}, V_{h+1/2}, S_{h+1/2}) A_h  B_{h+1/2} + \beta_{u,s}^\B(U_{h+1/2}, V_{h+1/2}, S_{h+1/2}).
\#
The interpretation for Bob's reward and action model is the same as Alice's and thus is omitted. We remark that the functions $\alpha_z^\B$ and $\alpha_{u,s}^\B$ are determined by Bob's adopted policy $b^\B_{h+1/2}$.

\vskip5pt
\noindent\textbf{Transition from $h+1/2$ to $h+1$.}
We have
\#\label{eq:hfeur}
& \EE\left[ f(S_{h+1}, U_{h+1}, V_{h+1}) \given S_{h+1/2}, A_h , B_{h+1/2}, U_{h+1/2}, V_{h+1/2} \right]   \\
& \qquad = \theta_{h+1/2}^f(U_{h+1/2}, V_{h+1/2}, S_{h+1/2}) A_h  + \gamma_{h+1/2}^f(U_{h+1/2}, V_{h+1/2}, S_{h+1/2}) B_{h+1/2}   \\
& \qquad \qquad + \omega_{h+1/2}^f(U_{h+1/2}, V_{h+1/2}, S_{h+1/2}) A_h  B_{h+1/2} + \iota_{h+1/2}^f (U_{h+1/2}, V_{h+1/2}, S_{h+1/2})   
\#
for any function $f\colon \cS \times \cU \times \cV \to \RR$.

The challenge of this two-player turn-based game is that as Alice, she cannot access the private information of Bob, which will prevent us from using the offline data to identify the value function. In this case, standard policy learning methods will in general induce
the so-called confounding bias. To address this challenge, we  treat Bob's action $B_{h-1/2}$ at $(h-1/2)$-th step as an IV for Alice $h$-th step's decision making in order to remove the effect of unobserved confounding caused by $V_h$. When considering the decision-making for Bob from Alice's side, we will use Alice $h$-step's action as an IV for Bob's $(h+1/2)$-th step's decision.  In particular, we impose the following assumptions for the collected offline data. 

\begin{assumption}[Identification]\label{ass:model}
For any $h\in [H]$, the following statements hold for the offline data generated by the behavior policy $b = (b^\A, b^\B)$. 
\begin{enumerate}
    \item[(a)] $A_h \notindp B_{h-1/2} \given (U_h, S_h)$.
    \item[(b)] $B_{h-1/2} \indp V_h \given (U_h, S_h)$.
    \item[(c)] $B_{h+1/2} \notindp A_h  \given (U_{h+1/2}, S_{h+1/2})$.
    \item[(d)] $A_h  \indp V_{h+1/2} \given (U_{h+1/2}, S_{h+1/2})$.
    \item[(e)] $\EE_{V_h}[\beta_{u,s}^\A(U_h, V_h, S_h)] = \EE_{V_{h+1/2}}[\beta_{u,s}^\A(U_{h+1/2}, V_{h+1/2}, S_{h+1/2})] = 0$. 
    \item[(f)] $\EE_{V_h}[\iota_h^f(U_h, V_h, S_h)] = \EE_{V_{h+1/2}}[\iota_h^f(U_{h+1/2}, V_{h+1/2}, S_{h+1/2})] = 0$,
\end{enumerate}
where $\EE_{V_{h+1/2}}$ and $\EE_{V_{h}}$ denote the expectations with respect to $V_{h}$ and $V_{h+1/2}$ given other variables respectively.
\end{assumption}
Assumption \ref{ass:model}~(a)-(b) are two well-known core IV conditions in the literature of causal inference \citep{angrist1995identification,angrist1996identification}. Specifically, Assumption \ref{ass:model}~(a) ensures that Bob's action $B_{h-1/2}$, which is regarded as an IV for Alice's action $A_h$, is correlated with the action $A_h$ of Alice given the observed state variables $U_h$ and $S_h$, which is mild. Assumption \ref{ass:model}~(b) states that the causal effect of $B_{h-1/2}$ on the future transition cannot be confounded by $V_h$, which is reasonable as $B_{h-1/2}$ happens before $V_h$. Assumption \ref{ass:model}~(c)-(d) are similar conditions, where we consider $A_h$ as an instrumental variable for adjusting the causal effect of $B_{h+1/2}$ on future transitions without accessing $V_{h+1/2}$. The last two assumptions \ref{ass:model}~(e)-(f) ensure that random effects in the reward and transition models considered in \eqref{eq:alice}--\eqref{eq:hfeur} will not carry over to the early stages. These two conditions are used to guarantee that the backward induction is valid during our learning procedure. See Figure \ref{fig:erfr} for a graphical illustration of the offline data collected from the two-player turn-based game satisfying Assumption \ref{ass:model}. 

\begin{figure}[htbp]
    \centering
\begin{tikzpicture}[scale=1.2]
        \node (bo) at (-1.5,0) [label=below:{$B_{h-1/2}$}, circle, fill=black]{};
        \node (a) at (1,0) [label=below:{$A_h$}, circle, fill=black]{};
        \node (s) at (1,1.2) [label=above:{$S_h$}, circle, fill=black]{};
        \node (u) at (1,2.4) [label=above:{$U_h$}, circle, fill=black]{};
        \node (v) at (-0.25,2.4) [label=above:{$V_h$}, circle, draw]{};
        \node (bn) at (3.5,0) [label=below:{$B_{h+1/2}$}, circle, fill=black]{};
        \node (sn) at (3.5,1.2) [label=above:{$S_{h+1/2}$}, circle, fill=black]{};
        \node (un) at (4.75,2.4) [label=above:{$U_{h+1/2}$}, circle, fill=black]{};
        \node (vn) at (3.5,2.4) [label=above:{$V_{h+1/2}$}, circle, draw]{};
        \node (snn) at (6,1.2) [label=above:{$S_{h+1}$}, circle, fill=black]{};
    \draw[-stealth] (bo) edge (s) edge (a) edge (sn);
    \draw[-stealth] (s) edge (a) edge (sn);
    \draw[-stealth] (a) edge (sn) edge (snn) edge (bn);
    \draw[-stealth] (sn) edge (bn) edge (snn);
    \draw[-stealth] (bn) edge (snn);
    \draw[-stealth] (u) [out=-45,in=40] edge (a);
    \draw[-stealth] (vn) [out=-45,in=40] edge (bn);
    \draw[dashed] (-0.6,1) rectangle ++(2.2,2.2);
    \draw[dashed] (3.0,1) rectangle ++(2.2,2.2);
    \draw [-stealth, dashed](1.6,2.1) -- (3,2.1);
\end{tikzpicture}
\caption{A graphical illustration of the offline data generated by a two-player turn-based game described in \S\ref{sec:model} satisfying Assumption \ref{ass:model}. 
Here $\{U_h, V_h\}$ are private information, which satisfy the independence assumption, i.e., $B_{h-1/2} \indp V_h \given (U_h, S_h)$ and $A_h  \indp V_{h+1/2} \given (U_{h+1/2}, S_{h+1/2})$. Also the trajectory $\{S_h, U_h, V_h, A_h, S_{h+1/2}, U_{h+1/2}, V_{h+1/2}, B_{h+1/2}\}_{h\in [H]}$ satisfies the Markov property. In the meanwhile, the dependence of $(S_{h+1/2}, U_{h+1/2}, V_{h+1/2})$ on $(S_h, U_h, V_h)$ is characterized by the dashed arrow in the figure for the simplicity of the presentation.}
\label{fig:erfr}
\end{figure}
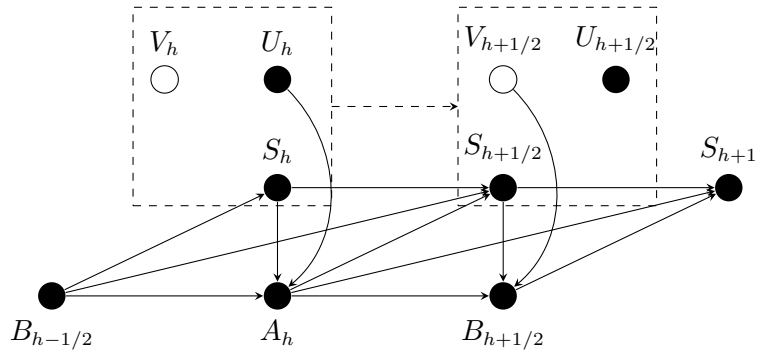

By the model structure specified in \eqref{eq:alice}--\eqref{eq:hfeur}, immediately we have the following property for the action-value functions of Alice. Similar results also hold for Bob. 

\begin{proposition}\label{prop:q-form-alice}
For any policy $\pi = (\pi^\A, \pi^\B)\in \Pi^\A\times \Pi^\B$ and $h\in [H]$, there exist functions $\tilde \theta_h^{\A}$, $\tilde \theta_{h+1/2}^{\A}$, $\tilde \gamma_h^{\A}$, $\tilde \gamma_{h+1/2}^{\A}$, $\tilde \omega_h^{\A}$, and $\tilde \omega_{h+1/2}^{\A}$ such that
{\small
\$
& Q^{\A,\pi}_h(S_h, A_h , B_{h-1/2}, U_h, V_h) \\
& \qquad = \tilde \theta_h^\A(S_h, U_h, V_h; \pi)\cdot A_h  + \tilde \gamma_h^\A(S_h, U_h, V_h; \pi) \cdot B_{h-1/2}  \\
& \qquad \qquad + \tilde \omega_h^\A(S_h, U_h, V_h; \pi) \cdot A_h  B_{h-1/2} + \tilde \zeta_h^\A(S_h, U_h, V_h; \pi),\\
& Q^{\A,\pi}_{h+1/2}(S_{h+1/2}, A_h , B_{h+1/2}, U_{h+1/2}, V_{h+1/2}) \\
& \qquad = \tilde \theta_{h+1/2}^\A(S_{h+1/2}, U_{h+1/2}, V_{h+1/2}; \pi)\cdot A_h  + \tilde \gamma_{h+1/2}^\A(S_{h+1/2}, U_{h+1/2}, V_{h+1/2}; \pi) \cdot B_{h+1/2}   \\
& \qquad \qquad + \tilde \omega_{h+1/2}^\A(S_{h+1/2}, U_{h+1/2}, V_{h+1/2}; \pi) \cdot A_h  B_{h+1/2} +  \tilde \zeta_{h+1/2}^\A(S_{h+1/2}, U_{h+1/2}, V_{h+1/2}; \pi). 
\$
}
\end{proposition}

Proposition \ref{prop:q-form-alice} implies that the action-value functions can always be decomposed into three components as the effects from Alice's action, Bob's action, and their interaction. It can be seen that our data generating process can be easily generalized to the multiple-action setting. Since $V_h$ is not observed, one cannot use the offline data to identify $Q^{\A,\pi}_h$ and $Q^{\A,\pi}_{h+1/2}$, and ultimately $J^\A(\pi^\A, \pi^\B)$. Under Assumption \ref{ass:model}~(e), we can define 
\$
& Q^{\A,\pi}_h(S_h, A_h , B_{h-1/2}, U_h) = \tilde \theta_h^\A(S_h, U_h; \pi)\cdot A_h  + \tilde \gamma_h^\A(S_h, U_h; \pi) \cdot B_{h-1/2} + \tilde \omega_h^\A(S_h, U_h; \pi) \cdot A_h  B_{h-1/2},  \\
& Q^{\A,\pi}_{h+1/2}(S_{h+1/2}, A_h , B_{h+1/2}, U_{h+1/2}) = \tilde \theta_{h+1/2}^\A(S_{h+1/2}, U_{h+1/2}; \pi)\cdot A_h  + \tilde \gamma_{h+1/2}^\A(S_{h+1/2}, U_{h+1/2}; \pi) \cdot B_{h+1/2}   \\
& \qquad \qquad \qquad \qquad \qquad \qquad \qquad \qquad \qquad + \tilde \omega_{h+1/2}^\A(S_{h+1/2}, U_{h+1/2}; \pi) \cdot A_h  B_{h+1/2},
\$
where,
\$
(\tilde \theta_h^\A(S_h, U_h; \pi), \tilde \gamma_h^\A(S_h, U_h; \pi), \tilde \omega_h^\A(S_h, U_h; \pi))^\top = \EE_{V_h}\left[(\tilde \theta_h^\A(S_h, U_h, V_h; \pi), \tilde \gamma_h^\A(S_h, U_h, V_h; \pi), \tilde \omega_h^\A(S_h, U_h, V_h; \pi))^\top\right]
\$
and 
{\small
\$
&(\tilde \theta_{h+1/2}^\A(S_{h+1/2}, U_{h+1/2}; \pi), \tilde \gamma_{h+1/2}^\A(S_{h+1/2}, U_{h+1/2}; \pi), \tilde \omega_{h+1/2}^\A(S_{h+1/2}, U_{h+1/2}; \pi))^\top
\\
& \quad =\EE_{V_{h+1/2}}\left[ (\tilde \theta_{h+1/2}^\A(S_{h+1/2}, U_{h+1/2}, V_{h+1/2}; \pi), \tilde \gamma_{h+1/2}^\A(S_{h+1/2}, U_{h+1/2}, V_{h+1/2}; \pi), \tilde \omega_{h+1/2}^\A(S_{h+1/2}, U_{h+1/2}, V_{h+1/2}; \pi))^\top \right].
\$
}

As long as we are able to use the offline data to identify $Q^{\A,\pi}_h(S_h, A_h , B_{h-1/2}, U_h)$ and $Q^{\A,\pi}_h(S_h, A_h , B_{h-1/2}, U_h)$, we can obtain that
$$
J^\A(\pi^\A, \pi^\B) = \EE[Q^{\A,\pi}_1(S_h, A_h , B_{h-1/2}, U_h)].
$$
Similar notations and results hold for Bob's state-action value function. In the following two sections, we leverage instrumental variables and establish a novel identification result for identifying the policy value of both players.

\section{Identification Results under Instrumental Variables: A Detour}\label{sec:id-all}
Since our identification results in the two-player turned-based game rely on the instrumental variable approach, in this subsection, we take a detour and introduce a new instrumental variable approach for identifying the effect of both the action and IV in the standard causal inference setting under unmeasured confounding. Such an approach serves as the foundation of our proposal for learning the optimal in-class policy pair in the two-player turned-based game without accessing Bob's private information, which will be introduced later in \S\ref{sec:uiehuier}. 

Consider a setting where the outcome $Y$ is generated by some observed state $(S, U)$, unobserved state $V$, a binary action $A \in \cA$ and a binary instrumental variable $B \in \cB$. We impose the following structure equation models for the outcome $Y$ and the action $A$. Specifically, for unknown functions $\bvartheta^*(S, U, V) = (\vartheta_a^*(S, U, V), \vartheta_z^*(S, U, V), \vartheta_{az}^*(S, U, V))$, the following two conditional expectation models hold,
\$
& \EE\left[Y\given S, U, V, A, B\right] = \vartheta_a^*(S, U, V) A + \vartheta_z^*(S, U, V) B + \vartheta_{az}^*(S, U, V) A B + \vartheta_{us}^*(S, U, V),\\
&\EE[ A  \given S, B, U, V ] = \alpha_z^*(S, U, V) B + \alpha_{u,s}^*(S, U, V),
\$
where one can treat $Y$ as an immediate reward, $A$ as the action of Alice, $B$ as the action taken by Bob in the previous step, $(S, U)$ as the information accessed by the Alice and $V$ as the private information of Bob in the current step. In the standard causal inference setting, existing works such as \cite{tchetgen2021genius} did not model the interaction effect $\vartheta_{az}^*(S, U, V)$. As discussed in the Alice's model, motivated by our applications, it is essential to model such an effect. This implies that our model is more general than the existing literature.
In this section, we aim to establish identification results for $\bvartheta^*(S, U) = (\vartheta_a^*(S, U), \vartheta_z^*(S, U), \vartheta_{az}^*(S, U))$ so that one can leverage the offline data, which are i.i.d copies of $\{Y, S, U, A, B\}$ (i.e., $\{Y^i, S^i, U^i, A^i, B^i\}_{i \in [n]}$), to learn the effect of the action $A$ and IV $B$ on the outcome $Y$. Here $\bvartheta^*(S, U) = \EE_{V}[\bvartheta^*(S, U, V)]$. Once $\bvartheta^*(S, U)$ is identified, one can figure out the best actions for $A$ and $B$ given observed information so as the expected outcome of $Y$ is maximized. Since this problem can be regarded as a special case of the two-player turn-based game, we impose the following assumptions similar as Assumption \ref{ass:model}, together with some orthogonal conditions.
\begin{assumption}[Identification for A Basic Setting]\label{ass:single-stage model}
The following statements hold for the random sample $\{Y, S, U, A, B\}$.
\begin{enumerate}
    \item[(a)] $B \notindp A \given (S, U)$.
    \item[(b)] $B \indp V \given (S, U)$.
    \item[(c)] The following orthogonality conditions hold.
	\begin{align}
		\Cov(\vartheta^\ast_a(S, U, V), \alpha^\ast_{z}(S, U, V) \given S, U) &= 	\Cov(\vartheta^\ast_a(S, U, V), \alpha^\ast_{u, s}(S, U, V) \given S, U)  = 0 \\
		\Cov(\vartheta^\ast_z(S, U, V), \alpha^\ast_{z}(S, U, V) \given S, U) &= 	\Cov(\vartheta^\ast_z(S, U, V), \alpha^\ast_{u, s}(S, U, V) \given S, U)  = 0 \\
		\Cov(\alpha^\ast_z(S, U, V), \vartheta^\ast_{az}(S, U, V) \given S, U) &= \Cov(\alpha^\ast_z(S, U, V), \vartheta_{us}(S, U,  V) \given  S, U) = 0\\
		\Cov(\vartheta^\ast_{az}(S, U,  V), \alpha^\ast_{u,s}(S, U,  V) \given  S, U) & = 0.
	\end{align}
\end{enumerate}
\end{assumption}
Again, Assumption \ref{ass:single-stage model}~(a)-(b) impose IV-relevant and independence assumptions on $\{Y, S, U, A, B\}$, which are standard in the existing literature \citep[e.g.,][]{wang2018bounded}. Since the IV $B$ (i.e., Bob's action) may have an effect on the outcome $Y$, we did not impose the exclusion restriction on $B$ by only allowing it to  affect $Y$ directly. Instead, we impose the orthogonal conditions in Assumption \ref{ass:single-stage model}~(c) to relax the IV exclusion restriction. One \textit{sufficient} condition for Assumption \ref{ass:single-stage model}~(c) is that there is no common effect modifiers resulted from $V$ in the aforementioned conditional expectation models related to $Y$ and $A$ (see Figure \ref{fig:6746} for a graphical illustration of such a case). 
Our identification condition is motivated by \cite{tchetgen2021genius}, where in their outcome model of $Y$, only the marginal effects of $A$ and $B$ are considered but not their interaction. Therefore, our models, identification conditions and results are more general than theirs. In the following, we present our identification results so that the offline data can identify $\bvartheta^*(S, U)$.

\begin{figure}[htbp]
    \centering
\begin{tikzpicture}[scale=1.2]
        \node (bo) at (-1.5,0) [label=above:{$(S,U)$}, circle, fill=black]{};
        \node (a) at (1,0) [label=below:{$A$}, circle, fill=black]{};
        \node (s) at (0,1.5) [label=above:{$B$}, circle, fill=black]{};
        \node (bn) at (3.5,0) [label=below:{$Y$}, circle, fill=black]{};
        \node (v2) at (2.5,1.5) [label=above:{$V_2$}, circle, draw]{};
        \node (v1) at (2.0,1.5) [label=above:{$V_1$}, circle, draw]{};
    \draw[-stealth] (bo) edge (a);
    \draw[-stealth] (s) edge (a);
    \draw[-stealth] (a) edge (bn);
    \draw[-stealth] (v1) edge (a);
    \draw[-stealth] (v2) edge (bn);
    \draw[-stealth] (s) edge (bn);
    \draw[dashed] (1.65,1.2) rectangle ++(1.2,0.95);
    \draw[-stealth] (bo) [out=-45,in=225] edge (bn);
\end{tikzpicture}
\caption{A graphical illustration of one sufficient condition for Assumption \ref{ass:single-stage model}~(c), where $V = (V_1, V_2)$ and $V_1$ and $V_2$ are independent. }
\label{fig:6746}
\end{figure}
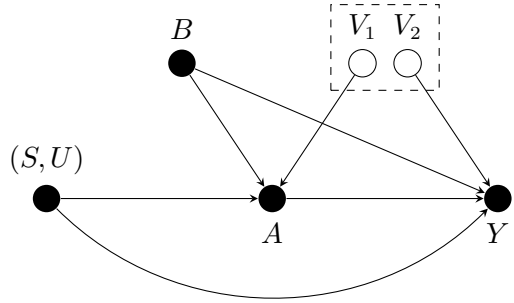


\begin{theorem}\label{theorem: basic id}
Under Assumption \ref{ass:single-stage model}, if the following three linear equations with respect to $\bvartheta^*(S, U)$ have a unique solution, then $\bvartheta^*(S, U)$ can be non-parametrically identified by our offline data.
\$
& \EE\left[ \left( B - \EE[B\given S, U] \right) \cdot \left( A - \EE[A\given B,  S, U] \right)\cdot Y \given  S, U \right] \\
& \quad = \EE\left[ \left( B - \EE[B \given  S, U ] \right) \cdot \left( A - \EE[A\given B,  S, U] \right)\cdot A \given  S, U \right] \cdot \vartheta_a^*(S, U) \\
& \quad \quad + \EE\left[ B \cdot \left( B - \EE[B \given  S, U ] \right) \cdot \left( A - \EE[A\given B,  S, U] \right)\cdot A \given  S, U \right] \cdot \vartheta_{az}^*( S, U), \\
& \EE\left[ \left( B - \EE[B\given  S, U] \right) \cdot Y \given  S, U \right] \\
& \quad = \EE\left[ A \cdot \left( B - \EE[B \given  S, U ] \right) \given  S, U \right] \cdot \vartheta_a^*( S, U) \\
& \quad \quad + \EE\left[ B \cdot \left( B - \EE[B\given  S, U ] \right) \given  S, U \right] \cdot \vartheta_z^*( S, U) \\
& \quad \quad + \EE\left[ A B \cdot \left( B - \EE[B\given  S, U ] \right) \given  S, U \right] \cdot \vartheta_{az}^*( S, U), \\
& \EE\left[ \left( B - \EE[B\given  S, U] \right) \cdot Y \given  S, U \right] \\
& \quad = \EE\left[ A \cdot \left( B - \EE[B \given  S, U ] \right) \given  S, U \right] \cdot \vartheta_a^*( S, U) \\
& \quad \quad + \EE\left[ B \cdot \left( B - \EE[B\given  S, U ] \right) \given  S, U \right] \cdot \vartheta_z^*( S, U) \\
& \quad \quad + \EE\left[ A B \cdot \left( B - \EE[B\given  S, U ] \right) \given  S, U \right] \cdot \vartheta_{az}^*( S, U). 
\$
\end{theorem}
\begin{proof}
    The proof can be found in \S \ref{appendix: lm 1-3}.
\end{proof}

By Theorem \ref{theorem: basic id}, we can identify $\vartheta_a^*$, $\vartheta_z^*$, and $\vartheta_{az}^*$ non-parametrically by solving a system of linear equations as long as the solution is unique. In the following, we introduce our estimator for $\bvartheta^*$ based on these identification results. 
To simplify the presentation, we assume that there exists an oracle that gives the following nuisance parameters, 
\$
& f_1( S, U) = \EE[B\given  S, U], \quad f_2(S, U, B) = \EE[A\given S, U, B], \quad f_3( S, U) = \EE[(A-f_2(S, U, B)) Y \given  S, U ], \\
& f_4( S, U) = \EE[A(A - f_2(S, U, B))\given  S, U], \qquad f_5( S, U) = \EE[AB (A - f_2(S, U, B))\given  S, U]. 
\$
We remark that such nuisance parameters can be estimated via supervised learning without oracle. See \S\ref{sec:id-all-wo-oracle} for such a setting. 
Further, we define
\#\label{eq:quanw}
& w_1(Y, A, B,  S, U) = \rho_1(Y, S, U, B) - \rho_2(S, U, B ) \vartheta_a^*( S, U) - \rho_3(S, U, B ) \vartheta_{az}^*( S, U), \\
& w_2(Y, A, B,  S, U) = \rho_4(Y,  S, U, B) - \rho_2(S, U, B, f_1) \vartheta_a^*(S, U) - \rho_6(S, U, B) \vartheta_z^*( S, U) - \rho_7(S, U, B) \vartheta_{az}^*( S, U), \\
& w_3(Y, A, B,  S, U) = \rho_8(Y, S, U, B ) - (1 - f_1( S, U))f_1( S, U) f_3( S, U) \\
& \qquad\qquad\qquad\qquad - \left( \rho_9(S, U, B ) - (1 - f_1( S, U)) f_1( S, U) f_4( S, U) \right)\vartheta_a^*( S, U) \\
& \qquad\qquad\qquad\qquad - \left( \rho_{10}(S, U, B ) - (1 - f_1( S, U)) f_1( S, U) f_5( S, U) \right)\vartheta_{az}^*( S, U),
\#
where $\{\rho_i\}_{i\in [10]}$ are defined in \eqref{eq:rhos} of the appendix. 
We denote by $W(Y, S, U, A, B, \bvartheta^*) = (w_1, w_2, w_3)$, where we recall that $\bvartheta^*(S, U) \in \mathbb{R}^3$. 
Observing the linear system in \eqref{eq:quanw}, we can write
\#\label{eq:ufwheiu}
W = \Phi(\bff) \bm{\vartheta}^* + \alpha(\bff), 
\#
where $\Phi(\bff) \in \RR^{3\times 3}$, $\alpha(\bff)\in \RR^3$, and $\bff = (f_1, f_2, \ldots, f_5)$. 
Then by the identification results in Theorem \ref{theorem: basic id}, we have the following conditional moment restriction, 
\#\label{eq:cmr}
\EE[W(Y, S, U, A, B, \bvartheta^*)\given S, U] = 0.
\#
We aim to recover $\bm{\vartheta}^*$ via such a conditional moment restriction. In general, we need to estimate both $\bvartheta^*$ and $\bff$, where \eqref{eq:cmr} is a nonlinear conditional moment restriction model with respect to $\bvartheta^*$ and $\bff$. Although we assume that we have accessed the oracle $\bff$, we still consider this problem as solving a nonlinear conditional moment restriction model. We
follow \cite{chen2012estimation} and propose to use a sieve minimum distance (SMD) estimator for $\bvartheta^*$. We remark that such a nonlinear conditional moment restriction has been studied in the econometrics literature. See \cite{chen2012estimation} and the reference therein. 
We first introduce the definition of the sieve space $\cH_n$,  
\#\label{eq:wuefiuyr}
\cH_n = \left\{ \bvartheta\in \cH \colon \bvartheta(\cdot) = \sum_{k=1}^{k(n)} a_k q_k(\cdot) \right\},
\#
where $k(n)<\infty$, $k(n)\to \infty$ as $n\to\infty$, and $\{q_k\}_{k=1}^\infty$ is a sequence of known basis functions of a Banach space (i.e., $\cH$) such as wavelet basis functions. 
Then the SMD estimator for $\bvartheta^*$ takes the following form, 
\#\label{eq:psmd-single}
& \hat{\bm{\vartheta}}_n \in \argmin_{{\bm{\vartheta}}\in \cH_n} \hat L({\bm{\vartheta}}), \\
& \text{where } \hat L({\bm{\vartheta}}) = \frac{1}{n} \sum_{i = 1}^n \hat m(Y^i, S^i, U^i, A^i, B^i; {\bm{\vartheta}})^\top  \hat m(Y^i, S^i, U^i, A^i, B^i; {\bm{\vartheta}}). 
\#
Here $\hat m(Y^i, S^i, U^i, A^i, B^i; \bvartheta)$ is a consistent estimator of $\EE[W(Y, S, U, A, B, \bvartheta)\given S, U]$ under some mild conditions. For example, one can use a non-parametric regression with basis function approximation to obtain $\hat m(Y^i, S^i, U^i, A^i, B^i; \bvartheta)$. In the meanwhile, we denote by $m(S, U,{\bm{\vartheta}}) = \EE[W(Y, S, U, A, B, {\bm{\vartheta}})\given S, U]$ for any ${\bm{\vartheta}}$. 

In the following, based on the result of \cite{chen2012estimation}, we derive the convergence rate of $\hat{\bm{\vartheta}}_n$. We introduce the following assumptions. 

\begin{assumption}[Consistency]\label{ass:consistency}
    The SMD estimator $\hat{\bm{\vartheta}}_n$ is a consistent estimator of $\bm{\vartheta}^*$. 
\end{assumption}

As we will show in Theorem \ref{thm:consistency} in \S\ref{sec:id-all-wo-oracle} of the appendix, such an assumption of consistency can be achieved under mild conditions. With Assumption \ref{ass:consistency}, we can restrict our space $\cH$ to a shrinking neighborhood around the ground truth $\bvartheta^*$ and define
\$
\cH_s = \left\{ \bvartheta\in \cH \colon  \|\bvartheta - \bvartheta^*\| \leq \delta, \|\bvartheta\| \leq M  \right\}, \qquad \cH_{sn} = \cH_s \cap \cH_n
\$
for a positive constant $M$ and a sufficiently small positive $\delta$ such that $\PP(\hat{\bm{\vartheta}}_n \notin \cH_s) < \delta$. $\| \bullet \|$ refers to the $L_2$ norm with respect to some underlying probability measure. Thus, we restrict the sieve space $\cH_n$ in \eqref{eq:psmd-single} to be $\cH_{sn}$ hereafter. 
Meanwhile, we define the path-wise derivative in the direction $[\bvartheta - \bvartheta^*]$ evaluated at $\bvartheta^*$ as follows, 
\$
\frac{\ud m(S,U,\bvartheta^*)}{\ud \bvartheta}[\bvartheta - \bvartheta^*] = \frac{\ud \EE[W(Y,S,A,B,U,(1-\tau)\bvartheta + \tau\bvartheta^*)\given S,U]}{\ud \tau} \bigggiven_{\tau = 0}. 
\$
We also define the following notation of pseudometric for any $\bvartheta_1,\bvartheta_2\in \cH_s$, 
\$
\|\bvartheta_1 - \bvartheta_2\|_\ps = \sqrt{\EE\left[ \left\| \frac{\ud m(S,U,\bvartheta^*)}{\ud \bvartheta} [\bvartheta_1 - \bvartheta_2] 
\right\|^2\right]}. 
\$

We introduce the following assumptions for the shrinking space $\cH_{s}$. 

\begin{assumption}[Local Curvature]\label{ass:curv-paper}
    The following statements hold. 
    \begin{itemize}
        \item $\cH_s$ and $\cH_{sn}$ are convex, and $m(S, U, \bvartheta)$ is continuously path-wise differentiable with respect to $\bvartheta\in \cH_{s}$. 
        \item There exists a constant $c>0$ such that $\|\bvartheta - \bvartheta^*\|_\ps \leq c \cdot \|\bvartheta - \bvartheta^*\|$ for any $\bvartheta\in \cH_{s}$. 
        \item There exist finite constants $c_1,c_2>0$ such that $\|\bvartheta - \bvartheta^*\|_\ps^2 \leq c_1 \EE[\|m(S,U,\bvartheta)\|^2]$ for any $\bvartheta\in \cH_{sn}$; and $c_2\EE[\|m(S,U,\Gamma_n \bvartheta^*)\|^2] \leq \|\Gamma_n \bvartheta^* - \bvartheta^*\|_\ps^2$ for some $\Gamma_n \bvartheta^*\in \cH_{k(n)}$ such that $\|\Gamma_n \bvartheta^* - \bvartheta^*\| = o(1)$. 
    \end{itemize}
\end{assumption}

The first and the second statements in Assumption \ref{ass:curv-paper} imply that the pseudometric $\|\cdot \|_\ps$ is well-defined in $\cH_{s}$ and is weaker than $\|\cdot\|$, which is the $L_2$ norm. The last statement in Assumption \ref{ass:curv-paper} implies that the pseudometric $\|\cdot \|_\ps$ is Lipschitz continuous with respect to the function $\EE[\|m(S,U,\bvartheta)\|^2]$ for any $\bvartheta\in \cH_{sn}$. Assumption \ref{ass:curv-paper} is imposed so that a fast convergence rate in terms of the pseudometric can be obtained.

\begin{assumption}[Sieve Approximation Error]\label{ass:fwjeib}
    We have $\|\bvartheta^* - \sum_{j = 1}^{k(n)} \la \bvartheta^*, q_j\ra q_j \| = O( \nu^{-\alpha}_{k(n)} )$ for a finite $\alpha > 0$ and a positive sequence $\{\nu_j\}_{j=1}^\infty$ that strictly increases and $\nu_j = \Theta(j^{1/d})$, where $d$ is the dimension of $S$ and $U$.
\end{assumption}

\begin{assumption}[Sieve Link Condition]\label{ass:uiwerrweff}
    There are finite constants $c, C > 0 $ such that
    \begin{itemize}
        \item $\|\bvartheta\|_\ps^2 \geq c \sum_{j=1}^\infty \varphi(\nu_j^{-2}) \cdot |\la \bvartheta, q_j\ra |^2$ for all $\bvartheta \in \cH_{sn}$; 
        \item $\|\Gamma_n \bvartheta^* - \bvartheta^*\|_\ps^2 \leq C\sum_{j = 1}^\infty \varphi(\nu_j^{-2})\cdot | \la \Gamma_n \bvartheta^* - \bvartheta^*, q_j \ra|^2$, where $\Gamma_n \bvartheta^*$ is defined in Assumption \ref{ass:curv-paper}; 
        \item $\varphi(\tau) = \tau^\varsigma$ for some $\varsigma \geq 0$. 
    \end{itemize}
\end{assumption}


Assumption \ref{ass:fwjeib} implies that $\cH_n$ is a natural sieve for the approximation of $\bvartheta^*$, which is a mild condition on the smoothness of $\vartheta^*\in \cH_s$. The constant $\alpha$ usually refers to the smoothness constant of $\cH$ and $d$ is the dimension of the input features $S$ and $U$. In the meanwhile, the first statement in Assumption \ref{ass:uiwerrweff} implies that the  norm in the shrinking sieve neighborhood $\cH_{sn}$ is upper bounded by the pseudometric; the second statement is a stability condition that is required only for the sieve approximation error; and the last statement characterizes the ill-posedness of solving the conditional moment restriction \eqref{eq:cmr} via a link function $\varphi$. Here we consider the mildly ill-posed case \citep{chen2011rate} and $\varsigma$ is the corresponding exponent. When $\varsigma$ is large, the hardness of solving \eqref{eq:cmr} increases.
Assumptions \ref{ass:fwjeib} and \ref{ass:uiwerrweff} are used to establish the link between the error quantified by the pseudometric and that by the $L_2$-norm.
See \cite{chen2012estimation} for more details on these assumptions.

With the above assumptions, we have the following results characterizing the convergence of $\hat \bvartheta_n$. 

\begin{theorem}[\cite{chen2012estimation}]\label{thm:chen2012}
Suppose that Assumptions \ref{ass:single-stage model}--\ref{ass:uiwerrweff} hold. We have
\$
\left\|\hat \bvartheta_n - \bvartheta^*\right\| = O_p\left( n^{-\frac{\alpha}{2\alpha+2\varsigma+d}} \right), 
\$
where $\hat \bvartheta_n$ is the SMD estimator defined in \eqref{eq:psmd-single}. 
\end{theorem}

Theorem \ref{thm:chen2012} indicates that the convergence rate depends on the smoothness constant $\alpha$, the dimension of $S$ and $U$ (i.e., $d$), and the ill-posed constant $\varsigma$. Specifically, when the smoothness of the true-function increases, our SMD estimator converges faster. In contrast, when the dimension and the ill-posed constant increase, the converge rate becomes slower.  Moreover, as shown in \cite{chen2012estimation}, the convergence rate in Theorem \ref{thm:chen2012} is minimax-optimal when the function $W$ defined in \eqref{eq:cmr} is linear. We can thus construct plug-in estimators of $\vartheta_a^*$, $\vartheta_z^*$, and $\vartheta_{az}^*$ as $\hat \vartheta_a^*$, $\hat \vartheta_z^*$, and $\hat \vartheta_{az}^*$ from $\hat \bvartheta_n$ and obtain the following convergence results. 

\begin{corollary}\label{cor:estimation}
Suppose that Assumptions \ref{ass:consistency}--\ref{ass:uiwerrweff} hold. We have
\$
\left\|\hat \vartheta_a^* - \vartheta_a^* \right\| = O_p\left( n^{-\frac{\alpha}{2\alpha+2\varsigma+d}} \right), \quad \left\|\hat \vartheta_z^* - \vartheta_z^* \right\| = O_p\left( n^{-\frac{\alpha}{2\alpha+2\varsigma+d}} \right), \quad \left\|\hat \vartheta_{az}^* - \vartheta_{az}^* \right\| = O_p\left( n^{-\frac{\alpha}{2\alpha+2\varsigma+d}} \right).
\$
\end{corollary}
Corollary \ref{cor:estimation} demonstrates that our estimators are valid for the effects of the action $A$ and the IV $B$ on the outcome $Y$. These results serve as the foundation of our proposal below.

\section{Warm-up: Single-Stage Game}\label{sec:uiehuier}

In this section, as a warm-up, we consider a single-stage version of the aforementioned two-player turn-based game, where the length of horizon $H=1$. Under such a setting, the game terminates once the state $S_2$ is revealed to both players. In order to perform policy learning under the unmeasured confounding, we first need to conduct off-policy evaluation. Specifically, we need to establish identification results so that we are able to estimate
the action-value functions $Q_1^{\A,\pi}$ and $Q_1^{\B,\pi}$ for any policy $\pi = (\pi^\A, \pi^\B)$ using the offline data without accessing the private information of Bob. In \S\ref{sec:id-all}, we study a standard causal setting, where we introduce a general identification condition and an estimating procedure using the instrumental variable approach. The established results motivate our proposal in the single-stage game below.  In \S\ref{sec:ygerf}, we tailor the results of \S\ref{sec:id-all} to the single-stage two-player turn-based game and propose our estimators for $Q_1^{\A,\pi}$ and $Q_1^{\B,\pi}$. Based on these results, in \S\ref{sec:mfniuwer}, we propose our method for offline policy learning and finally in \S\ref{sec:ukyefgyuer}, we show that our proposed method provably converges to an optimal policy. 

\subsection{Evaluation of Expected Total Rewards}\label{sec:ygerf}
In this subsection, we extend results from \S\ref{sec:id-all} to our single-stage turn-based game. For any policy pair $\pi = (\pi^\A, \pi^\B)$, we introduce the estimation of $Q_1^{\A,\pi}$, $Q_{3/2}^{\B,\pi}$, and $Q_1^{\B,\pi}$ as follows. 

\vskip5pt
\noindent\textbf{Estimation of $Q_1^{\A,\pi}$.}
Recall that $Q^{\A,\pi}_{3/2} = 0$. Then by Alice's model \eqref{eq:alice} and the definition of $Q_1^{\A,\pi}(S_1, A_1, B_{1/2}, U_1)$, we have 
\#\label{eq:bellman-h-single}
Q_1^{\A,\pi}(S_1, A_1, B_{1/2}, U_1) & = 
\beta_a^\A(U_1, S_1) A_1  + \beta_z^\A(U_1, S_1)  B_{1/2} + \beta_{az}^\A(U_1, S_1) A_1 B_{1/2}, 
\#
where we write $\beta_a^\A(U_1, S_1) = \EE_{V_1}[\beta_a^\A(U_1, V_1, S_1)]$, $\beta_z^\A(U_1, S_1) = \EE_{V_1}[\beta_z^\A(U_1, V_1, S_1)]$, and $\beta_{az}^\A(U_1, S_1) = \EE_{V_1}[\beta_{az}^\A(U_1, V_1, S_1)]$ for notational convenience. 
Thus, to construct an estimator of $Q_1^{\A,\pi}$, we only need to estimate the right-hand side of \eqref{eq:bellman-h-single}. Following from the estimating procedure in \S\ref{sec:id-all} by treating $A_1$ as the action and $B_{1/2}$ as an IV, together with similar orthogonal conditions, we can obtain an estimation of the right-hand side of \eqref{eq:bellman-h-single} as follows, 
\#\label{eq:est-1-single}
\hat Q_1^{\A,\pi}(S_1, A_1, B_{1/2}, U_1) = \hat \beta_a^\A(U_1, S_1) A_1  + \hat \beta_z^\A(U_1, S_1)  B_{1/2} + \hat \beta_{az}^\A(U_1, S_1) A_1 B_{1/2}. 
\#

\vskip5pt
\noindent\textbf{Estimation of $Q_{3/2}^{\B,\pi}$.}
Similar to the derivation of \eqref{eq:bellman-h-single}, we have the following equation, 
\$
Q_{3/2}^{\B,\pi}(S_{3/2}, A_1, B_{3/2}, U_{3/2}) 
= \beta_a^\B(U_{3/2}, S_{3/2}) B_{3/2} + \beta_z^\B(U_{3/2}, S_{3/2}) A_1 + \beta_{az}^\B(U_{3/2}, S_{3/2}) A_1  B_{3/2},
\$
where the estimation of the right-hand side can be obtained via the procedure in \S\ref{sec:id-all} by treating $B_{3/2}$ as the action and $A_1$ as the IV, i.e.,
\#\label{eq:est-2-single}
\hat Q_{3/2}^{\B,\pi}(S_{3/2}, A_1, B_{3/2}, U_{3/2}) = \hat \beta_a^\B(U_{3/2}, S_{3/2}) B_{3/2}  + \hat \beta_z^\B(U_{3/2}, S_{3/2})  A_1 + \hat \beta_{az}^\B(U_{3/2}, S_{3/2}) A_1 B_{3/2}.
\#

\vskip5pt
\noindent\textbf{Estimation of $Q_1^{\B,\pi}$.}
Note that by the memoryless assumption of $V_{3/2}$ and Assumption \ref{ass:model}~(e), we have the following Bellman equation,
\$\label{eq:bellman-h-single-bob-1}
Q_1^{\B,\pi}(S_1, A_1, B_{1/2}, U_1, V_1) & = \EE\left[ Q^{\B,\pi}_{3/2}(S_{3/2}, A_1, B_{3/2}, U_{3/2}) \given S_1, A_1, B_{1/2}, U_1, V_1 \right].
\$
Then 
we can obtain that
\#\label{eq:bellman-h-single-bob-1}
Q_1^{\B,\pi}(S_1, A_1, B_{1/2}, U_1, V_1) & = \EE\left[ Q^{\B,\pi}_{3/2}(S_{3/2}, A_1, B_{3/2}, U_{3/2}) \given S_1, A_1, B_{1/2}, U_1, V_1 \right] \\
& = \EE\left[\beta_a^\B(U_{3/2}, S_{3/2}) B_{3/2}\given S_1, A_1, B_{1/2}, U_1, V_1 \right] \\
& \qquad + \EE \left[ \beta_z^\B(U_{3/2}, S_{3/2}) \given S_1, A_1, B_{1/2}, U_1, V_1 \right] \cdot A_1\\
& \qquad + \EE \left[ \beta_{az}^\B(U_{3/2}, S_{3/2}) B_{3/2} \given S_1, A_1, B_{1/2}, U_1, V_1 \right] \cdot A_1\\
& \triangleq (I) + (II) \cdot A_1 + (III) \cdot A_1
\#
By Assumption \ref{ass:model}~(e),  the model specified in \eqref{eq:model-trans-h} and Bob's reward model, we know that there exist functions $\{ \theta_1^{\B,j}, \gamma_1^{\B,j}, \omega_1^{\B,j}\}_{j\in [3]}$ such that
\$
& (I) = \theta_1^{\B,1}(S_1,U_1, V_1;\pi) \cdot A_1 + \gamma_1^{\B,1}(S_1,U_1, V_1;\pi) \cdot B_{1/2} + \omega_1^{\B,1}(S_1,U_1, V_1;\pi) \cdot A_1 B_{1/2}, \\
& (II) = \theta_1^{\B,2}(S_1,U_1, V_1;\pi) \cdot A_1 + \gamma_1^{\B,2}(S_1,U_1, V_1;\pi) \cdot B_{1/2} + \omega_1^{\B,2}(S_1,U_1, V_1;\pi) \cdot A_1 B_{1/2}, \\
& (III) = \theta_1^{\B,3}(S_1,U_1, V_1;\pi) \cdot A_1 + \gamma_1^{\B,3}(S_1,U_1, V_1;\pi) \cdot B_{1/2} + \omega_1^{\B,3}(S_1,U_1, V_1;\pi) \cdot A_1 B_{1/2}. 
\$

Since we have estimators $\hat \beta_a^\B$, $\hat \beta_z^\B$, and $\hat \beta_{az}^\B$ for $\beta_a^\B$, $\beta_z^\B$, and $\beta_{az}^\B$, respectively from the estimator $\hat Q^{\B,\pi}_{3/2}$ in \eqref{eq:est-2-single}, following the discussion after Proposition \ref{prop:q-form-alice}, we only need to estimate the following three terms, 
\#\label{eq:bob-1-model}
& \text{term 1} = \theta_1^{\B,1}(S_1,U_1;\pi) \cdot A_1 + \gamma_1^{\B,1}(S_1,U_1;\pi) \cdot B_{1/2} + \omega_1^{\B,1}(S_1,U_1;\pi) \cdot A_1 B_{1/2}, \\
& \text{term 2} = \theta_1^{\B,2}(S_1,U_1;\pi) \cdot A_1 + \gamma_1^{\B,2}(S_1,U_1;\pi) \cdot B_{1/2} + \omega_1^{\B,2}(S_1,U_1;\pi) \cdot A_1 B_{1/2}, \\
& \text{term 3} = \theta_1^{\B,3}(S_1,U_1;\pi) \cdot A_1 + \gamma_1^{\B,3}(S_1,U_1;\pi) \cdot B_{1/2} + \omega_1^{\B,3}(S_1,U_1;\pi) \cdot A_1 B_{1/2},
\#
by integrating out $V_1$ given $(S_1, U_1)$.
Again, following from the estimation procedure in \S\ref{sec:id-all} where we treat $A_1$ as the action and $B_{1/2}$ as the IV, we can obtain the estimators of the right-hand sides of the three terms in \eqref{eq:bob-1-model} as follows, 
\#\label{eq:est-ferufnrue}
& \hat \theta_1^{\B,1}(S_1,U_1;\pi) \cdot A_1 + \hat \gamma_1^{\B,1}(S_1,U_1;\pi) \cdot B_{1/2} + \hat \omega_1^{\B,1}(S_1,U_1;\pi) \cdot A_1 B_{1/2}, \\
& \hat \theta_1^{\B,2}(S_1,U_1;\pi) \cdot A_1 + \hat \gamma_1^{\B,2}(S_1,U_1;\pi) \cdot B_{1/2} + \hat \omega_1^{\B,2}(S_1,U_1;\pi) \cdot A_1 B_{1/2}, \\
& \hat \theta_1^{\B,3}(S_1,U_1;\pi) \cdot A_1 + \hat \gamma_1^{\B,3}(S_1,U_1;\pi) \cdot B_{1/2} + \hat \omega_1^{\B,3}(S_1,U_1;\pi) \cdot A_1 B_{1/2},
\#
respectively. 
Combining \eqref{eq:bellman-h-single-bob-1}, \eqref{eq:bob-1-model}, and \eqref{eq:est-ferufnrue}, we obtain the following estimator of $Q_1^{\B,\pi}$, 
\#\label{eq:est-3-single}
& \hat Q_1^{\B,\pi}(S_1, A_1, B_{1/2}, U_1) = \hat \theta_1^\B(S_1, U_1; \pi)\cdot A_1  + \hat \gamma_1^\B(S_1, U_1; \pi) \cdot B_{1/2} + \hat \omega_1^\B(S_1, U_1; \pi) \cdot A_1  B_{1/2},
\#
where 
\#\label{eq:jtoall}
& \hat \theta_1^\B(S_1, U_1; \pi) = \hat \theta_1^{\B,1}(S_1,U_1;\pi) + \hat \theta_1^{\B,2}(S_1,U_1;\pi) + \hat \theta_1^{\B,3}(S_1,U_1;\pi), \\
& \hat \gamma_1^\B(S_1, U_1; \pi) = \hat \gamma_1^{\B,1}(S_1,U_1;\pi), \\
& \hat \omega_1^\B(S_1, U_1; \pi) = \hat \omega_1^{\B,1}(S_1,U_1;\pi) + \hat \omega_1^{\B,2}(S_1,U_1;\pi) + \hat \omega_1^{\B,3}(S_1,U_1;\pi) \\
& \qquad \qquad \qquad \qquad + \hat \gamma_1^{\B,2}(S_1,U_1;\pi) + \hat \gamma_1^{\B,3}(S_1,U_1;\pi). 
\#
See Figure \ref{fig:11} for a graphical illustration of the above construction.

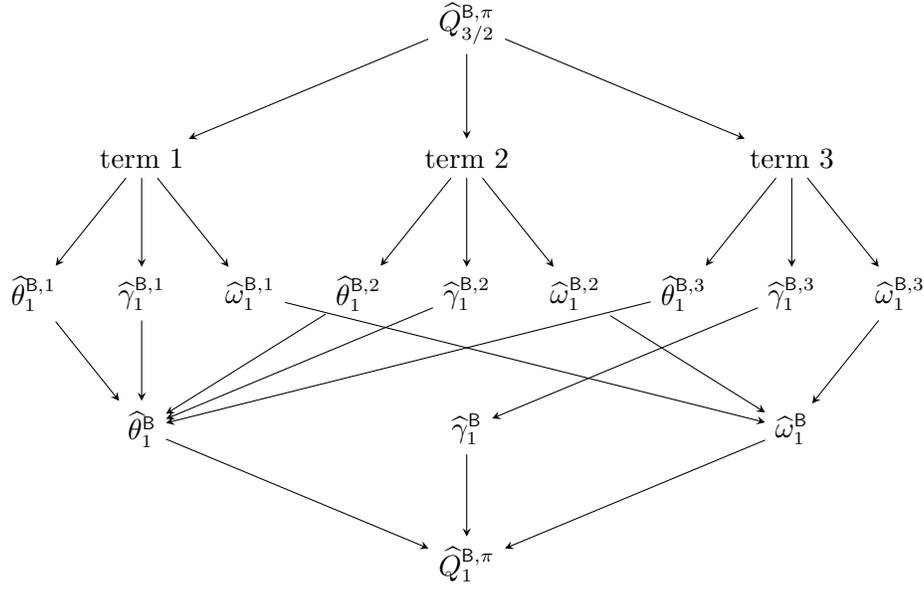
\begin{figure}[htbp]
    \centering
\begin{tikzpicture}[scale=1.2]
        \node (tq) at (0,0) {$\hat Q^{\B,\pi}_{3/2}$}{};
        \node (t1) at (-3.6,-1.5) {term 1}{};
        \node (t2) at (0,-1.5) {term 2}{};
        \node (t3) at (3.6,-1.5) {term 3}{};
        \node (theta11) at (4.8,-3) {$\hat \omega_1^{\B,3}$}{};
        \node (gamma11) at (3.6,-3) {$\hat \gamma_1^{\B,3}$}{};
        \node (omega11) at (2.4,-3) {$\hat \theta_1^{\B,3}$}{};
        \node (theta12) at (1.2,-3) {$\hat \omega_1^{\B,2}$}{};
        \node (gamma12) at (0,-3) {$\hat \gamma_1^{\B,2}$}{};
        \node (omega12) at (-1.2,-3) {$\hat \theta_1^{\B,2}$}{};
        \node (theta13) at (-2.4,-3) {$\hat \omega_1^{\B,1}$}{};
        \node (gamma13) at (-3.6,-3) {$\hat \gamma_1^{\B,1}$}{};
        \node (omega13) at (-4.8,-3) {$\hat \theta_1^{\B,1}$}{};
        \node (theta1) at (3.6,-4.5) {$\hat \omega_1^{\B}$}{};
        \node (gamma1) at (0,-4.5) {$\hat \gamma_1^{\B}$}{};
        \node (omega1) at (-3.6,-4.5) {$\hat \theta_1^{\B}$}{};
        \node (eq) at (0,-6) {$\hat Q_1^{\B,\pi}$}{};

    \draw[-stealth] (tq) edge (t3);
    \draw[-stealth] (tq) edge (t2);
    \draw[-stealth] (tq) edge (t1);
    \draw[-stealth] (t3) edge (theta11);
    \draw[-stealth] (t3) edge (gamma11);
    \draw[-stealth] (t3) edge (omega11);
    \draw[-stealth] (t2) edge (theta12);
    \draw[-stealth] (t2) edge (gamma12);
    \draw[-stealth] (t2) edge (omega12);
    \draw[-stealth] (t1) edge (theta13);
    \draw[-stealth] (t1) edge (gamma13);
    \draw[-stealth] (t1) edge (omega13);
    
    \draw[-stealth] (theta11) edge (theta1);
    \draw[-stealth] (gamma11) edge (gamma1);
    \draw[-stealth] (omega11) edge (omega1);
    \draw[-stealth] (theta12) edge (theta1);
    \draw[-stealth] (gamma12) edge (omega1);
    \draw[-stealth] (omega12) edge (omega1);
    \draw[-stealth] (theta13) edge (theta1);
    \draw[-stealth] (gamma13) edge (omega1);
    \draw[-stealth] (omega13) edge (omega1);

    \draw[-stealth] (theta1) edge (eq);
    \draw[-stealth] (gamma1) edge (eq);
    \draw[-stealth] (omega1) edge (eq);
    
\end{tikzpicture}
\caption{A graphical illustration of the estimation of action-value function $Q_1^{\B,\pi}$.}
\label{fig:11}
\end{figure}

\vskip5pt
\noindent\textbf{Estimation of $J^\A$ and $J^\B$.}
In summary, to construct $\hat Q^\A_1$ and $\hat Q^\B_1$ in \eqref{eq:est-1-single} and \eqref{eq:est-3-single}, respectively, we need to obtain 
\$
& \hat \bbeta^\A = \left(\hat \beta_a^\A, \hat \beta_z^\A, \hat \beta_{az}^\A\right), \quad \hat \bbeta^\B = \left(\hat \beta_a^\B, \hat \beta_z^\B, \hat \beta_{az}^\B\right), \quad \hat \btheta_1^{\B,j} = \left(\hat \theta_1^{\B,j}, \hat \gamma_1^{\B,j}, \hat \omega_1^{\B,j}\right) \quad \text{for any $j\in [3]$}
\$
as estimators of the following functions, 
\$
&  \bbeta^\A = \left( \beta_a^\A,  \beta_z^\A,  \beta_{az}^\A\right), \quad  \bbeta^\B = \left( \beta_a^\B, \beta_z^\B, \beta_{az}^\B\right), \quad \btheta_1^{\B,j} = \left( \theta_1^{\B,j}, \gamma_1^{\B,j},  \omega_1^{\B,j}\right) \quad \text{for any $j\in [3]$}. 
\$
Once we obtain these estimators, we can estimate $J^\A(\pi^\A, \pi^\B)$ and $J^\B(\pi^\A, \pi^\B)$ as follows,
\$
\hat J^\A(\pi^\A, \pi^\B) = \EE_\pi \left[ \hat Q_1^{\A,\pi}(S_1, A_1, B_{1/2}, U_1) \right], \quad \hat J^\B(\pi^\A, \pi^\B) = \EE_\pi \left[ \hat Q_1^{\B,\pi}(S_1, A_1, B_{1/2}, U_1) \right], 
\$
where the expectation $\EE_\pi[\cdot]$ is taken with following the policy $\pi = (\pi^\A, \pi^\B)$.

\subsection{Pessimistic Policy Learning Method}\label{sec:mfniuwer}
With the evaluation procedure proposed in \S\ref{sec:ygerf} and \S\ref{sec:id-all}, we present our method to find an optimal policy in this section. Instead of directly maximizing $\hat J^\A(\pi^\A, \pi^\B)$ and $\hat J^\B(\pi^\A, \pi^\B)$, we use the pessimistic idea in policy learning. See \cite{jin2021pessimism,fu2022offline} for more details.  To begin with, in light of \eqref{eq:ufwheiu}, we have the following equations 
\#\label{eq:nviuwer}
W^{\Diamond, \r} = \Phi^{\Diamond, \r} \bbeta^\Diamond + \alpha^{\Diamond, \r}, \qquad
W^{\B,j,\pi}(\tilde \bbeta^\B) = \Phi_1^{\B,j,\pi}(\tilde \bbeta^\B) \btheta_1^{\B,j,\pi}(\tilde \bbeta^\B) + \alpha_1^{\B,j,\pi} (\tilde \bbeta^\B), 
\#
where $\Phi^{\Diamond, \r}, \Phi_1^{\B,j,\pi}(\tilde \bbeta^\B) \in \RR^{3\times 3}$ and $\alpha^{\Diamond, \r}, \alpha_1^{\B,j,\pi} (\tilde \bbeta^\B)\in \RR^3$ for any $\tilde \bbeta^\B\in \cH^3$, $\Diamond\in \{\A,\B\}$, and $j\in [3]$. Here $W^{\Diamond, \r} = (w_1^{\Diamond, \r}, w_2^{\Diamond, \r}, w_3^{\Diamond, \r})$ and $W^{\B,j,\pi}(\tilde \bbeta^\B) = (w_1^{\B,j,\pi}, w_2^{\B,j,\pi}, w_3^{\B,j,\pi})$ are similarly constructed as \eqref{eq:quanw} by treating $A_1$ as the action and  $B_{1/2}$ as the IV or by treating $B_{3/2}$ as the action and $A_1$ as the IV. With \eqref{eq:nviuwer}, we further define the loss functions for the construction of estimators in \eqref{eq:est-1-single}, \eqref{eq:est-2-single}, and \eqref{eq:est-3-single} as follows, 
\#\label{eq:wiuehfrwe}
& \hat L^{\A,\r}(\tilde \bbeta^\A) = \frac{1}{n} \sum_{i = 1}^n \hat m^{i,\A,\r}(\tilde \bbeta^\A)^\top  \hat m^{i,\A,\r}(\tilde \bbeta^\A), \\
& \hat L^{\B,\r}(\tilde \bbeta^\B) = \frac{1}{n} \sum_{i = 1}^n \hat m^{i,\B,\r}(\tilde \bbeta^\B)^\top  \hat m^{i,\B,\r}(\tilde \bbeta^\B), \\
& \hat L_1^{\B,j,\pi}(\tilde \btheta_1^{\B,j}, \tilde \bbeta^\B) = \frac{1}{n} \sum_{i = 1}^n \hat m_1^{i,\B,j,\pi}(\tilde \btheta_1^{\B,j}, \tilde \bbeta^\B)^\top  \hat m_1^{i,\B,j,\pi}(\tilde \btheta_1^{\B,j}, \tilde \bbeta^\B),
\#
where $\hat m^{i,\A,\r}(\tilde \bbeta^\A)$ is a consistent estimator of $\EE[\Phi^{\A, \r} \tilde \bbeta^\A + \alpha^{\A, \r}\given S_1, U_1]$, $\hat m^{i,\B,\r}(\tilde \bbeta^\B)$ is a consistent estimator of $\EE[\Phi^{\B, \r} \tilde \bbeta^\B + \alpha^{\B, \r}\given S_{3/2}, U_{3/2}]$, while $\hat m_1^{i,\B,j,\pi}(\tilde \btheta_1^{\B,j}, \tilde \bbeta^\B)$ is a consistent estimator of $\EE[\Phi_1^{\B,j,\pi}(\tilde \bbeta^\B) \btheta_1^{\B,j,\pi}(\tilde \bbeta^\B) + \alpha_1^{\B,j,\pi} (\tilde \bbeta^\B) \given S_1, U_1]$. By minimizing the loss functions in \eqref{eq:wiuehfrwe}, we can obtain the following estimators needed to construct estimators of action-value functions defined in \eqref{eq:est-1-single}, \eqref{eq:est-2-single}, and \eqref{eq:est-3-single}, 
\#\label{eq:estimatorssss}
& \hat \bbeta^\A = \left(\hat \beta_a^\A, \hat \beta_z^\A, \hat \beta_{az}^\A\right), \quad \hat \bbeta^\B = \left(\hat \beta_a^\B, \hat \beta_z^\B, \hat \beta_{az}^\B\right), \quad \hat \btheta_1^{\B,j} = \left(\hat \theta_1^{\B,j}, \hat \gamma_1^{\B,j}, \hat \omega_1^{\B,j}\right) \quad \text{for any $j\in [3]$}. 
\#

In this work, instead of directly plugging in the estimators defined in \eqref{eq:estimatorssss}, we propose to use conservative estimators as proxies of action-value functions, and maximize such conservative action-value functions to learn an optimal policy. Such a method is also termed pessimism in related literature. 
Specifically, given the loss functions defined in \eqref{eq:wiuehfrwe}, we define the following confidence regions, 
\#\label{eq:conf-genius-single}
& \ci^{\Diamond,\r} = \left\{\tilde \bbeta^{\Diamond} = \left(\tilde \beta_a^\Diamond, \tilde \beta_z^\Diamond, \tilde \beta_{az}^\Diamond\right) \colon \hat L^{\Diamond,\r}(\tilde \bbeta^{\Diamond}) - \hat L^{\Diamond,\r}(\hat \bbeta^{\Diamond}) \leq \eta^{\Diamond,\r} \right\}, \\
& \ci_1^{\B,j,\pi}(\tilde \bbeta^\B) = \left\{\tilde \btheta_1^{\B,j} = \left(\tilde \theta_1^{\B,j}, \tilde \gamma_1^{\B,j}, \tilde \omega_1^{\B,j} \right) \colon \hat L_1^{\B,j,\pi}(\tilde \btheta_1^{\B,j}, \tilde \bbeta^\B) - \hat L_1^{\B,j,\pi}(\hat \btheta_1^{\B,j}, \tilde \bbeta^\B) \leq \eta_1^{\B,j} \right\}, \\
& \ci_1^{\B,\pi}(\tilde \bbeta^\B) = \left\{\tilde \btheta_1^{\B} = \left(\tilde \theta_1^\B, \tilde \gamma_1^\B, \tilde \omega_1^\B \right) \text{ satisfying \eqref{eq:jtoall}} \colon \tilde \btheta_1^{\B,j} \in \ci_1^{\B,j,\pi}(\tilde \bbeta^\B) \text{ for any $j\in [3]$} \right\},
\#
for any $j\in [3]$ and $\Diamond\in \{\A, \B\}$. Here  $\eta^{\Diamond,\r}$ and $\eta_1^{\B,j}$ measure the size of the confidence regions, which will be specified later in \eqref{eq:wiuehfre}. 
Then, we construct confidence regions for the action-value functions as follows,
\#\label{eq:conf-q-single}
& \ci_1^{\A,\Q,\pi} = \left\{ Q_1^\A = \tilde \beta_a^\A A_1  + \tilde \beta_z^\A B_{1/2} + \tilde \beta_{az}^\A A_1 B_{1/2} \colon \left(\tilde \beta_a^\A, \tilde \beta_z^\A, \tilde \beta_{az}^\A \right) \in \ci^{\A,\r}
 \right\}, \\
& \ci_1^{\B,\Q,\pi} = \bigcup_{\tilde \bbeta^\B\in \ci^{\B,\r}} \left\{ Q^\B_1 = \tilde \theta_1^\B A_1  + \tilde \gamma_1^\B B_{1/2} + \tilde \omega_1^\B A_1 B_{1/2} \colon \left(\tilde \theta_1^\B, \tilde \gamma_1^\B, \tilde \omega_1^\B \right) \in \ci_1^{\B,\pi}(\tilde \bbeta^\B) \right\}. 
\#
See Figure \ref{fig:54} for a graphical illustration.

\begin{figure}[htbp]
    \centering
\begin{tikzpicture}[scale=1.2]
        \node (tq) at (0,0) {$\ci^{\B,\r}$}{};
        \node (t1) at (3,1) {$\ci^{\B,1,\pi}(\tilde \bbeta^\B)$}{};
        \node (t2) at (3,0) {$\ci^{\B,2,\pi}(\tilde \bbeta^\B)$}{};
        \node (t3) at (3,-1) {$\ci^{\B,3,\pi}(\tilde \bbeta^\B)$}{};
        \node (kk) at (6,0) {$\ci_1^{\B,\pi}(\tilde \bbeta^\B)$}{};
        \node (qq) at (9,0) {$\ci_1^{\B,\Q,\pi}$}{};

    \draw[-stealth] (tq) edge (t1);
    \draw[-stealth] (tq) edge (t2);
    \draw[-stealth] (tq) edge (t3);
    \draw[-stealth] (t1) edge (kk);
    \draw[-stealth] (t2) edge (kk);
    \draw[-stealth] (t3) edge (kk);
    \draw[-stealth] (kk) edge (qq);

\end{tikzpicture}
\caption{A graphical illustration of the construction of the confidence region $\ci_1^{\B,\Q,\pi}$.}
\label{fig:54}
\end{figure}
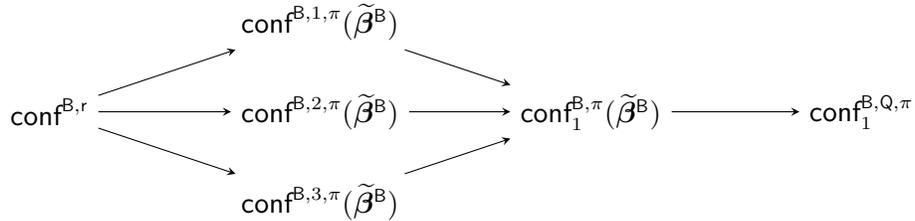

We solve the following optimization problem, which maximizes the sum of the estimated expected total rewards of Alice and Bob, 
\#\label{prob:opt-genius}
(\hat \pi^\A, \hat \pi^\B) \in \argmax_{\pi^\A\in \Pi^\A, \pi^\B\in \Pi^\B} \min_{Q_1^\A \in \ci_1^{\A,\Q,\pi}, Q_1^\B \in \ci_1^{\B,\Q,\pi}} \EE_\pi \left[ \sum_{\Diamond\in \{\A,\B\}} Q_1^\Diamond(S_1, A_1, B_{1/2}, U_1) \right], 
\#
where the confidence regions $\ci_1^{\A,\Q,\pi}$ and $\ci_1^{\B,\Q,\pi}$ are defined in \eqref{eq:conf-q-single}. $(\hat \pi^\A, \hat \pi^\B)$ is denoted as our estimated policy pair.

\subsection{Theoretical Results}\label{sec:ukyefgyuer}

In this section, we characterize the performance of the learned policy pair $(\hat \pi^\A, \hat \pi^\B)$ in \eqref{prob:opt-genius}. Specifically, we aim to upper bound the performance gap as follows, 
\#\label{eq:nwiuebfre}
\gap(\hat \pi^\A, \hat \pi^\B) = \sup_{\pi^\A\in \Pi^\A, \pi^\B\in \Pi^\B} J(\pi^\A, \pi^\B) - J(\hat \pi^\A, \hat \pi^\B), 
\#
where $J(\cdot,\cdot)$ is the expected total reward defined in \eqref{eq:opt-prob}. 
For the sizes of the confidence regions in \eqref{eq:conf-genius-single}, we let
\#\label{eq:wiuehfre}
\eta^{\A,\r} = O\left( n^{-\frac{2\alpha}{2\alpha+2\varsigma+d}} \right), \quad \eta^{\B,\r} = O\left( n^{-\frac{2\alpha}{2\alpha+2\varsigma+d}} \right), \quad  \eta^{\B,j}_1 = O\left( n^{-\frac{2\alpha}{2\alpha+2\varsigma+d}} \right)
\#
for any $j\in [3]$.

\begin{assumption}
\label{ass:lower-bound-phi-paper}
    We have $\sigma_{\min}(\EE[ \Phi_1^{\B, j, \pi^*}(\bbeta^\B)^\top \Phi_1^{\B, j, \pi^*}(\bbeta^\B)^\top]) \geq c$ and $\sigma_{\min}(\EE[ \Phi^{\Diamond, \r \top} \Phi^{\Diamond, \r}]) \geq c$ for any $\Diamond \in \{\A,\B\}$ and $j\in [3]$, where $c$ is a positive absolute constant. 
\end{assumption}
Assumption \ref{ass:lower-bound-phi-paper} only requires that the loss functions in \eqref{eq:wiuehfrwe} admit a unique minimizer under the distribution induced by the optimal policy pair $\pi^*$, which is indeed a weak assumption. For example, as long as statements in Assumption \ref{ass:lower-bound-phi-paper} hold under the data generating distribution and $\pi^\ast$ is covered by the behavior policy, Assumption \ref{ass:lower-bound-phi-paper} holds as well. This is often called partial coverage in the literature of RL. Now we are ready to state our theorem, which upper bounds the performance gap in \eqref{eq:nwiuebfre}. 

\begin{theorem}\label{thm: regret for single}
Under Assumptions \ref{ass:model},\ref{ass:lower-bound-phi-paper}, and similar sets of conditions as Assumption \ref{ass:single-stage model}~(c), Assumptions 
\ref{ass:consistency}--\ref{ass:uiwerrweff}, it holds that
\$
\gap(\hat \pi^\A, \hat \pi^\B) = O_p\left( n^{-\frac{\alpha}{2\alpha+2\varsigma+d}} \right). 
\$
\end{theorem}
\begin{proof}

We first introduce two supporting lemmas. The following lemma shows that with the sizes of the confidence regions in \eqref{eq:wiuehfre}, the true action-value functions $Q_1^{\A,\pi}$ and $Q_1^{\B,\pi}$ lie in the confidence regions. 

\begin{lemma}\label{lemma:pess}
Under the assumptions stated in Theorem \ref{thm:chen2012}, for any $\pi\in \Pi$, it holds that  $\lim_{n\to\infty} \PP(Q_1^{\A,\pi}\in \ci_1^{\A,\Q,\pi}, Q_1^{\B,\pi}\in \ci_1^{\B,\Q,\pi}) = 1$. 
\end{lemma}
\begin{proof}
See \S\ref{prf:lemma:pess} for a detailed proof. 
\end{proof}

The following lemma shows that with the sizes of the confidence regions in \eqref{eq:wiuehfre}, for any functions in the confidence regions, they are close to the true action-value functions $Q_1^{\A,\pi}$ and $Q_1^{\B,\pi}$. 

\begin{lemma}\label{lemma:bound}
Under Assumption \ref{ass:lower-bound-phi-paper} and the assumptions stated in Theorem \ref{thm:chen2012}, for any $Q_1^\A\in \ci_1^{\A,\Q,\pi^*}$ and $Q_1^\B\in \ci_1^{\B,\Q,\pi^*}$, we have 
\$
\left |\EE_{\pi^*}\left[ Q_1^{\A,\pi^*} - Q_1^\A \right] \right | = O_p\left( n^{-\frac{\alpha}{2\alpha+2\varsigma+d}} \right), \qquad \left |\EE_{\pi^*}\left[ Q_1^{\B,\pi^*} - Q_1^\B \right] \right | = O_p\left( n^{-\frac{\alpha}{2\alpha+2\varsigma+d}} \right). 
\$
\end{lemma}
\begin{proof}
See \S\ref{prf:lemma:bound} for a detailed proof. 
\end{proof}

Finally, it holds that 
\$
& J(\pi^{*,\A}, \pi^{*,\B}) - J(\hat \pi^\A, \hat \pi^\B) \\
& \qquad = \EE_{\pi^*}\left[ (Q_1^{\A,\pi^*} + Q_1^{\B,\pi^*})(S_1, A_1, B_{1/2}, U_1) \right] - \EE_{\hat \pi}\left[ (Q_1^{\A,\hat \pi} + Q_1^{\B,\hat \pi})(S_1, A_1, B_{1/2}, U_1) \right] \\
& \qquad \leq \EE_{\pi^*}\left[ (Q_1^{\A,\pi^*} + Q_1^{\B,\pi^*})(S_1, A_1, B_{1/2}, U_1) \right] - \min_{Q_1^\A \in \ci_1^{\A,\Q,\hat \pi}, Q_1^\B \in \ci_1^{\B,\Q,\hat \pi}}\EE_{\hat \pi}\left[ (Q_1^\A + Q_1^\B)(S_1, A_1, B_{1/2}, U_1) \right] \\
& \qquad \leq \EE_{\pi^*}\left[ (Q_1^{\A,\pi^*} + Q_1^{\B,\pi^*})(S_1, A_1, B_{1/2}, U_1) \right] - \min_{Q_1^\A \in \ci_1^{\A,\Q,\pi^*}, Q_1^\B \in \ci_1^{\B,\Q,\pi^*}} \EE_{\pi^*}\left[ (Q_1^\A + Q_1^\B)(S_1, A_1, B_{1/2}, U_1) \right] \\
& \qquad \leq \max_{Q_1^\A \in \ci_1^{\A,\Q,\pi^*}, Q_1^\B \in \ci_1^{\B,\Q,\pi^*}} \left | \EE_{\pi^*}\left[ (Q_1^{\A,\pi^*} - Q_1^\A + Q_1^{\B,\pi^*} - Q_1^\B)(S_1, A_1, B_{1/2}, U_1) \right]  \right |, 
\$
where in the first and second inequalities, we use Lemma \ref{lemma:pess} and the optimality of $(\hat \pi^\A, \hat \pi^\B)$, respectively.
Further, by Lemma \ref{lemma:bound}, we have 
\$
J(\pi^{*,\A}, \pi^{*,\B}) - J(\hat \pi^\A, \hat \pi^\B) = O_p\left( n^{-\frac{\alpha}{2\alpha+2\varsigma+d}} \right),
\$
which concludes the proof of the theorem. 
\end{proof}

In Theorem \ref{thm: regret for single}, with only partial coverage assumption and the help of convergence of the SMD estimators, the upper bound on the performance gap implies that the regret of finding an optimal in-class policy pair converges to 0 as long as the number of trajectories $n$ goes to infinity.

\section{Policy Learning Method for Multi-Stage Game}\label{sec: multi-game}

In a multi-stage game with $H \geq 2$, the idea of policy evaluation is similar to the single-stage game in \S\ref{sec:ygerf}. Therefore details are deferred to \S\ref{sec:full-eval} of the appendix. 
In the meanwhile, similar to the construction of confidence regions as in \eqref{eq:conf-q-single}, we define the following confidence regions for the action-value function at the first step, 
\#\label{eq:wiuevhrei}
\ci_1^{\Diamond,\Q,\pi} = \bigcup_{\substack{ \tilde \bbeta^{\Diamond} \in \ci^{\Diamond, \r}, \\ \tilde \btheta^\Diamond_{h}\in \ci_{h}^{\Diamond,\pi}(\tilde \btheta^\Diamond_{h+1/2}), \\ \text{for any $h\in \cC^{\Diamond}$} }  } \left\{ Q^\Diamond_1 \colon \left(\tilde \theta_1^\Diamond, \tilde \gamma_1^\Diamond, \tilde \omega_1^\Diamond \right) \in \ci_1^{\Diamond,\pi}(\tilde \btheta^\Diamond_{3/2}) \right\}, 
\#
where $Q^\Diamond_1 = \tilde \theta_1^\Diamond A_1  + \tilde \gamma_1^\Diamond B_{1/2} + \tilde \omega_1^\Diamond A_1 B_{1/2}$ for any $\Diamond\in \{\A,\B\}$, $\cC^\A = \{3/2, 2, \ldots, H-1/2\}$, $\cC^\B = \{3/2, 2, \ldots, H\}$, and the confidence regions $\ci^{\Diamond, \r}$ and $\ci_{h}^{\Diamond,\pi}$ are defined in \eqref{eq:conf-genius} in \S\ref{sec:iureiuw} of the appendix. 
Figure \ref{fig:ci-full} illustrates how $\ci_1^{\A,\Q,\pi}$ is constructed for an example. We remark that the construct of $\ci_1^{\B,\Q,\pi}$ can also be visualized in a similar way.

\begin{figure}[htbp]
    \centering
\begin{tikzpicture}[scale=1.2]
        \node (tq) at (0,0) {$\ci^{\A,\r}$}{};
        \node (t1) at (-3,-1) {$\ci_{H-1/2}^{\A,1,\pi}(\tilde \btheta_H^\A)$}{};
        \node (t2) at (0,-1) {$\ci_{H-1/2}^{\A,2,\pi}(\tilde \btheta_H^\A)$}{};
        \node (t3) at (3,-1) {$\ci_{H-1/2}^{\A,3,\pi}(\tilde \btheta_H^\A)$}{};
        \node (kk) at (0,-2) {$\ci_{H-1/2}^{\A,\pi}(\tilde \btheta_H^\A)$}{};
        
        \node (pp2) at (0,-3) {$\vdots$}{};
        
        \node (pp21) at (0,-4) {$\ci_{3/2}^{\A,\pi}(\tilde \btheta_{2}^\A)$}{};

        \node (t11) at (-3,-5) {$\ci_{1}^{\A,1,\pi}(\tilde \btheta_{3/2}^\A)$}{};
        \node (t21) at (0,-5) {$\ci_{1}^{\A,2,\pi}(\tilde \btheta_{3/2}^\A)$}{};
        \node (t31) at (3,-5) {$\ci_{1}^{\A,3,\pi}(\tilde \btheta_{3/2}^\A)$}{};
        \node (kk1) at (0,-6) {$\ci_{1}^{\A,\pi}(\tilde \btheta_{3/2}^\A)$}{};

        \node (qq) at (0,-7) {$\ci_1^{\A,\Q,\pi}$}{};

    \draw[-stealth] (tq) edge (t1);
    \draw[-stealth] (tq) edge (t2);
    \draw[-stealth] (tq) edge (t3);
    \draw[-stealth] (t1) edge (kk);
    \draw[-stealth] (t2) edge (kk);
    \draw[-stealth] (t3) edge (kk);
    
    \draw[-stealth] (kk) edge (pp2);

    \draw[-stealth] (pp21) edge (t11);
    \draw[-stealth] (pp21) edge (t21);
    \draw[-stealth] (pp21) edge (t31);
    
    \draw[-stealth] (pp2) edge (pp21);

    \draw[-stealth] (t11) edge (kk1);
    \draw[-stealth] (t21) edge (kk1);
    \draw[-stealth] (t31) edge (kk1);
    
    \draw[-stealth] (kk1) edge (qq);

\end{tikzpicture}
\caption{A graphical illustration of the construction of the confidence region $\ci_1^{\A,\Q,\pi}$ in \eqref{eq:wiuevhrei}.}
\label{fig:ci-full}
\end{figure}
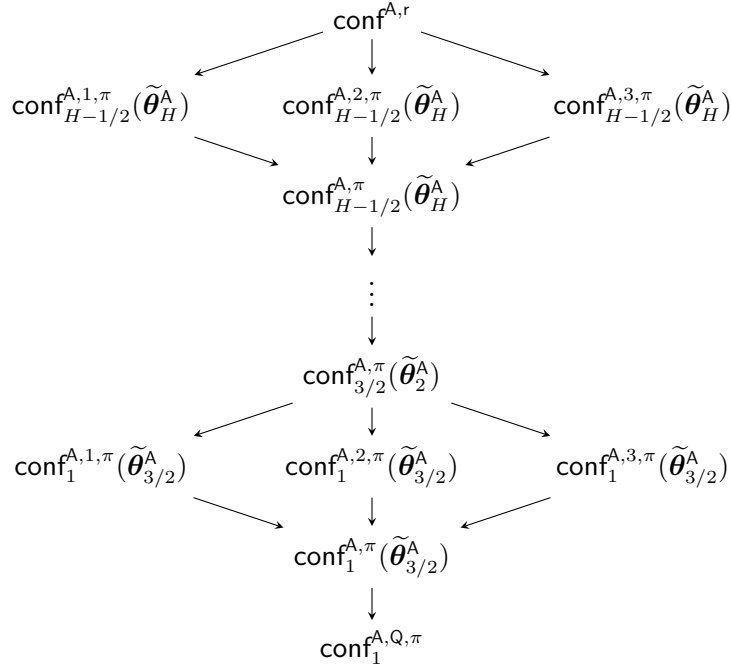

We solve the following optimization problem, which maximizes the sum of the estimated expected total rewards of Alice and Bob, 
\#\label{prob:opt-genius2}
(\hat \pi^\A, \hat \pi^\B)\in \argmax_{\pi^\A\in \Pi^\A, \pi^\B\in \Pi^\B} \min_{Q_1^\Diamond \in \ci_1^{\Diamond,\Q,\pi} \forall\Diamond \in \{\A, \B\} } \EE_\pi \left[ \sum_{\Diamond\in \{\A,\B\}} Q_1^\Diamond(S_1, A_1, B_{1/2}, U_1) \right], 
\#
where the confidence regions $\ci_1^{\Diamond,\Q,\pi}$ are defined in \eqref{eq:wiuevhrei} for any $\Diamond\in \{\A,\B\}$.

\subsection{Theoretical Results}

In this section, we characterize the performance of the learned policy pair $(\hat \pi^\A, \hat \pi^\B)$ in \eqref{prob:opt-genius2}. Specifically, we aim to upper bound the performance gap defined in \eqref{eq:nwiuebfre}. For the sizes of the confidence regions in \eqref{eq:conf-genius-single}, we let 
\#\label{eq:iowuehfurew}
\eta^{\Diamond,\r} = O\left( n^{-\frac{2\alpha}{2\alpha+2\varsigma+d}} \right), \quad  \eta^{\Diamond,j}_h = O\left( (H-h)^4 \cdot  n^{-\frac{2\alpha}{2\alpha+2\varsigma+d}} \right)
\#
for any $j\in [3]$, $h\in [H]\cup [H]_{1/2}$, and $\Diamond\in \{\A, \B\}$.

The following assumption ensures the loss functions in \eqref{eq:conf-genius} admit unique minimizers only under the optimal policy pair $\pi^*$, which is indeed a weak assumption only requiring partial coverage. We remark that such an assumption is in parallel to Assumption \ref{ass:lower-bound-phi-paper} for single-stage games. 

\begin{assumption}\label{ass:lower-bound-phi-all}
    We have 
    \$
    & \sigma_{\min}(\EE[ \Phi_{h}^{\Diamond,j,\pi^*}( \btheta_{h+1/2}^{\Diamond}, \bff_{h}^{\Diamond})^\top \Phi_{h}^{\Diamond,j,\pi^*}( \btheta_{h+1/2}^{\Diamond}, \bff_{h}^{\Diamond})^\top]) \geq c, \\
    & \sigma_{\min}(\EE[ \Phi^{\Diamond, \r}(\bff^\Diamond)^\top \Phi^{\Diamond, \r}(\bff^\Diamond)]) \geq c
    \$ 
    for any $\Diamond \in \{\A,\B\}$ and $j\in [3]$, where $c$ is a positive absolute constant, and the matrices $\Phi_{h}^{\Diamond,j,\pi^*}( \btheta_{h+1/2}^{\Diamond}, \bff_{h}^{\Diamond})$ and $\Phi^{\Diamond, \r}(\bff^\Diamond)$ are defined in \eqref{eq:nuwiehv} in \S\ref{sec:iureiuw} of the appendix. 
\end{assumption}

Now we are ready to state our theorem, which upper bounds the performance gap in \eqref{eq:nwiuebfre}. 

\begin{theorem}\label{thm:main-H}
Under Assumptions \ref{ass:model},\ref{ass:lower-bound-phi-all}, and similar sets of conditions as Assumption \ref{ass:single-stage model}~(c), Assumptions 
\ref{ass:consistency}--\ref{ass:uiwerrweff}, it holds that 
\$
\gap(\hat \pi^\A, \hat \pi^\B) = O_p\left( H^{2} \cdot n^{-\frac{\alpha}{2\alpha+2\varsigma+d}} \right). 
\$
\end{theorem}
\begin{proof}
See \S\ref{prf:thm:main-H} for a detailed proof. 
\end{proof}

Similar to Theorem \ref{thm: regret for single}, in Theorem \ref{thm:main-H}, with only partial coverage assumption and the help of convergence of the SMD estimators, the upper bound on the performance gap implies that the regret of finding an optimal in-class policy pair converges to 0 as long as the number of trajectories $n$ goes to infinity in a multi-stage game.

\section{Conclusion}\label{sec: conclusion}

In this paper, we study human-guided human-machine interaction via a form of two-player turn-based games, where one player (a human) tries to guide another player (a machine) to maximize the sum of their expected total rewards together under the presence of private information of the machine. Specifically, we consider an offline setting where the data are pre-collected. Two key challenges are presented in solving this problem: (i) confounding biases induced by the private information, which cannot be accessed by the human; (ii) distributional shift in offline setting. To solve (i), we innovatively use machine's action as a natural instrumental variable, and establish a general non-parametric identification result for the causal effect of both action and IV, based on which we propose an estimators for off-policy evaluation. To tackle (ii), we leverage the idea of pessimism and propose an optimal policy pair estimator for both human and machine. Theoretically, a rate of convergence to an optimal policy pair is rigorously established.

\newpage
\begin{APPENDICES}
\section{Detailed Method for Multi-Stage Game}

\subsection{Evaluation of Expected Total Rewards}\label{sec:full-eval}

For any policy pair $\pi = (\pi^\A, \pi^\B)$, we estimate $J^\A(\pi^\A, \pi^\B)$ and $J^\B(\pi^\A, \pi^\B)$ in a backward fashion. We introduce the estimation procedure of $Q_h^{\A,\pi}$, $Q_{h+1/2}^{\A,\pi}$, $Q_h^{\B,\pi}$, and $Q_{h+1/2}^{\B,\pi}$ for any $h \in [H]$ as follows. 

\vskip5pt
\noindent\textbf{Estimation of $Q_h^{\A,\pi}$.}
For any $h \in [H]$, with an estimator $\hat Q_{h+1/2}^\A(\cdot, \cdot, \cdot, \cdot; \pi)$, we estimate $Q_h^{\A,\pi}(\cdot, \cdot, \cdot, \cdot)$ as follows. Note that by the memoryless assumption of $V_{h+1/2}$, the following Bellman equation holds, 
\#\label{eq:bellman-h}
Q_h^{\A,\pi}(S_h, A_h , B_{h-1/2}, U_h, V_h) & = \EE_\pi \left[ R_h^\A + Q^{\A,\pi}_{h+1/2}(S_{h+1/2}, A_h , B_{h+1/2}, U_{h+1/2}) \right] \\
& = \EE\left[R_h^\A\right] + \EE\left[\tilde \theta_{h+1/2}^\A(S_{h+1/2}, U_{h+1/2}; \pi) \right]\cdot A_h \\
& \qquad + \EE\left[\tilde \gamma_{h+1/2}^\A(S_{h+1/2}, U_{h+1/2}; \pi)\cdot \pi^\B_{h+1/2}(S_{h+1/2}) \right]  \\
& \qquad + \EE\left[\tilde \omega_{h+1/2}^\A(S_{h+1/2}, U_{h+1/2}; \pi) \cdot \pi^\B_{h+1/2}(S_{h+1/2}) \right] \cdot A_h. 
\#
Note that the above expectations are taken conditioning on $S_h$, $A_h$, $B_{h-1/2}$, $U_h$, and $V_h$. 
Thus, to construct an estimator of $Q_h^{\A,\pi}$, we only need to estimate the right-hand side of \eqref{eq:bellman-h}. Following from \eqref{eq:bellman-h-single}, we can estimate the first term $\EE[R_h^\A]$ on the right-hand side of \eqref{eq:bellman-h} as 
\#\label{eq:est-1}
\hat \beta_a^\A(S_h, U_h) A_h  + \hat \beta_z^\A(S_h, U_h)  B_{h-1/2} + \hat \beta_{az}^\A(S_h, U_h) A_h  B_{h-1/2}.
\#
Note that the conditional expectation of $\beta_{u,s}^\A(U_h, V_h, S_h)$ is zero by Assumption \ref{ass:model}, thus we have the above formulation.
For the remaining three terms on the right-hand side of \eqref{eq:bellman-h}, since we have estimators $\hat \theta_{h+1/2}^\A$, $\hat \gamma_{h+1/2}^\A$, and $\hat \omega_{h+1/2}^\A$ for $\tilde \theta_{h+1/2}^\A$, $\tilde \gamma_{h+1/2}^\A$, and $\tilde \omega_{h+1/2}^\A$, respectively from the estimator $\hat Q^\A_{h+1/2}$, we only need to estimate the following three terms, 
\$
& (I) = \EE\left[\hat \theta_{h+1/2}^\A(S_{h+1/2}, U_{h+1/2}; \pi) \given S_h, A_h , B_{h-1/2}, U_h, V_h \right], \\
& (II) = \EE\left[\hat \gamma_{h+1/2}^\A(S_{h+1/2}, U_{h+1/2}; \pi)\cdot \pi^\B_{h+1/2}(S_{h+1/2}) \given S_h, A_h , B_{h-1/2}, U_h, V_h \right], \\
& (III) = \EE\left[\hat \omega_{h+1/2}^\A(S_{h+1/2}, U_{h+1/2}; \pi) \cdot \pi^\B_{h+1/2}(S_{h+1/2}) \given S_h, A_h , B_{h-1/2}, U_h, V_h \right].
\$
By Assumption \ref{ass:model}~(e) and the model specified in \eqref{eq:model-trans-h}, we know that there exist functions $\{ \theta_h^{\A,j}, \gamma_h^{\A,j}, \omega_h^{\A,j}\}_{j\in [3]}$ such that
\#\label{eq:est-2-ppp}
& \theta_h^{\A,1}(S_h,U_h,V_h;\pi) \cdot A_h + \gamma_h^{\A,1}(S_h,U_h,V_h;\pi) \cdot B_{h-1/2} + \omega_h^{\A,1}(S_h,U_h,V_h;\pi) \cdot A_h B_{h-1/2}, \\
& \theta_h^{\A,2}(S_h,U_h,V_h;\pi) \cdot A_h + \gamma_h^{\A,2}(S_h,U_h,V_h;\pi) \cdot B_{h-1/2} + \omega_h^{\A,2}(S_h,U_h,V_h;\pi) \cdot A_h B_{h-1/2}, \\
& \theta_h^{\A,3}(S_h,U_h,V_h;\pi) \cdot A_h + \gamma_h^{\A,3}(S_h,U_h,V_h;\pi) \cdot B_{h-1/2} + \omega_h^{\A,3}(S_h,U_h,V_h;\pi) \cdot A_h B_{h-1/2}, 
\#
Following from a similar idea as in \eqref{eq:bob-1-model} by integrating out $V_h$ in \eqref{eq:est-2-ppp}, we construct the following estimators for the above terms, 
\#\label{eq:est-2}
& \hat \theta_h^{\A,1}(S_h,U_h;\pi) \cdot A_h + \hat \gamma_h^{\A,1}(S_h,U_h;\pi) \cdot B_{h-1/2} + \hat \omega_h^{\A,1}(S_h,U_h;\pi) \cdot A_h B_{h-1/2}, \\
& \hat \theta_h^{\A,2}(S_h,U_h;\pi) \cdot A_h + \hat \gamma_h^{\A,2}(S_h,U_h;\pi) \cdot B_{h-1/2} + \hat \omega_h^{\A,2}(S_h,U_h;\pi) \cdot A_h B_{h-1/2}, \\
& \hat \theta_h^{\A,3}(S_h,U_h;\pi) \cdot A_h + \hat \gamma_h^{\A,3}(S_h,U_h;\pi) \cdot B_{h-1/2} + \hat \omega_h^{\A,3}(S_h,U_h;\pi) \cdot A_h B_{h-1/2}, 
\#
respectively. 
Combining \eqref{eq:bellman-h}, \eqref{eq:est-1}, and \eqref{eq:est-2}, we obtain the following estimator of $Q_h^{\A,\pi}$, 
\#\label{eq:est-qha}
& \hat Q^\A_h(S_h, A_h , B_{h-1/2}, U_h; \pi) \\
& \qquad = \hat \theta_h^\A(S_h, U_h; \pi)\cdot A_h  + \hat \gamma_h^\A(S_h, U_h; \pi) \cdot B_{h-1/2} + \hat \omega_h^\A(S_h, U_h; \pi) \cdot A_h  B_{h-1/2},
\#
where 
\#
& \hat \theta_h^\A(S_h, U_h; \pi) = \hat \beta_a^\A(U_h, S_h) + \hat \theta_h^{\A,1}(S_h,U_h;\pi) + \hat \theta_h^{\A,2}(S_h,U_h;\pi) + \hat \theta_h^{\A,3}(S_h,U_h;\pi), \\
& \hat \gamma_h^\A(S_h, U_h; \pi) = \hat \beta_z^\A(U_h, S_h) + \hat \gamma_h^{\A,2}(S_h,U_h;\pi), \\
& \hat \omega_h^\A(S_h, U_h; \pi) = \hat \beta_{az}^\A(U_h, S_h) + \hat \omega_h^{\A,1}(S_h,U_h;\pi) + \hat \omega_h^{\A,2}(S_h,U_h;\pi) + \hat \omega_h^{\A,3}(S_h,U_h;\pi) \\
& \qquad \qquad \qquad \qquad + \hat \gamma_h^{\A,1}(S_h,U_h;\pi) + \hat \gamma_h^{\A,3}(S_h,U_h;\pi). 
\#

\vskip5pt
\noindent\textbf{Estimation of $Q_{h+1/2}^{\A,\pi}$.}
For any $h \in [H]$, we estimate $Q_{h+1/2}^{\A,\pi}(\cdot, \cdot, \cdot, \cdot; \pi)$ as follows. 
Since we have estimators $\hat \theta_{h+1}^\A$, $\hat \gamma_{h+1}^\A$, and $\hat \omega_{h+1}^\A$ of $\tilde \theta_{h+1}^\A$, $\tilde \gamma_{h+1}^\A$, and $\tilde \omega_{h+1}^\A$, respectively from the estimator $\hat Q^\A_{h+1}$, we first estimate the following three terms,
\$
& \EE\left[\hat \theta_{h+1}^\A(S_{h+1}, U_{h+1}^\A; \pi) \cdot \pi^\A_{h+1}(S_{h+1}, U_{h+1}^\A) \given S_{h+1/2}, A_h , B_{h+1/2}, U_{h+1/2}, V_{h+1/2} \right], \\
& \EE\left[\hat \gamma_{h+1}^\A(S_{h+1}, U_{h+1}^\A; \pi) \given S_{h+1/2}, A_h , B_{h+1/2}, U_{h+1/2}, V_{h+1/2} \right], \\
& \EE\left[\hat \omega_{h+1}^\A(S_{h+1}, U_{h+1}^\A; \pi) \cdot \pi^\A_{h+1}(S_{h+1}, U_{h+1}^\A) \given S_{h+1/2}, A_h , B_{h+1/2}, U_{h+1/2}, V_{h+1/2} \right]
\$
by
\#\label{eq:est-3}
& \hat \theta_{h+1/2}^{\A,1}(S_{h+1/2},U_{h+1/2};\pi) \cdot A_h + \hat \gamma_{h+1/2}^{\A,1}(S_{h+1/2},U_{h+1/2};\pi) \cdot B_{h+1/2} \\
& \qquad + \hat \omega_{h+1/2}^{\A,1}(S_{h+1/2},U_{h+1/2};\pi) \cdot A_h B_{h+1/2}, \\
& \hat \theta_{h+1/2}^{\A,2}(S_{h+1/2},U_{h+1/2};\pi) \cdot A_h + \hat \gamma_{h+1/2}^{\A,2}(S_{h+1/2},U_{h+1/2};\pi) \cdot B_{h+1/2} \\
& \qquad + \hat \omega_{h+1/2}^{\A,2}(S_{h+1/2},U_{h+1/2};\pi) \cdot A_h B_{h+1/2}, \\
& \hat \theta_{h+1/2}^{\A,3}(S_{h+1/2},U_{h+1/2};\pi) \cdot A_h + \hat \gamma_{h+1/2}^{\A,3}(S_{h+1/2},U_{h+1/2};\pi) \cdot B_{h+1/2} \\
& \qquad + \hat \omega_{h+1/2}^{\A,3}(S_{h+1/2},U_{h+1/2};\pi) \cdot A_h B_{h+1/2}, 
\#
respectively. We remark that the above estimators in \eqref{eq:est-3} are constructed following from a similar idea as in \eqref{eq:est-2}. Similar as in \eqref{eq:est-qha}, we obtain the following estimator of $Q_{h+1/2}^{\A,\pi}$, 
\#\label{eq:est-qha2}
& \hat Q^\A_{h+1/2}(S_{h+1/2}, A_h , B_{h+1/2}, U_{h+1/2}; \pi)   \\
& \qquad = \hat \theta_{h+1/2}^\A(S_{h+1/2}, U_{h+1/2}; \pi) \cdot A_h  + \hat \gamma_{h+1/2}^\A(S_{h+1/2}, U_{h+1/2}; \pi) \cdot B_{h+1/2}   \\
& \qquad \qquad+ \hat \omega_{h+1/2}^\A(S_{h+1/2}, U_{h+1/2}; \pi) \cdot A_h  B_{h+1/2}. 
\#
where 
\#
& \hat \theta_{h+1/2}^\A(S_{h+1/2}, U_{h+1/2}; \pi) = \hat \theta_{h+1/2}^{\A,1}(S_{h+1/2}, U_{h+1/2}; \pi), \\
& \hat \gamma_{h+1/2}^\A(S_{h+1/2}, U_{h+1/2}; \pi) = \hat \gamma_{h+1/2}^{\A,1}(S_{h+1/2}, U_{h+1/2}; \pi) + \hat \gamma_{h+1/2}^{\A,2}(S_{h+1/2}, U_{h+1/2}; \pi) \\
& \qquad \qquad \qquad \qquad \qquad \qquad + \hat \gamma_{h+1/2}^{\A,3}(S_{h+1/2}, U_{h+1/2}; \pi), \\
& \hat \omega_{h+1/2}^\A(S_{h+1/2}, U_{h+1/2}; \pi) = \hat \omega_{h+1/2}^{\A,1}(S_{h+1/2}, U_{h+1/2}; \pi) + \hat \omega_{h+1/2}^{\A,2}(S_{h+1/2}, U_{h+1/2}; \pi)  \\ 
& \qquad \qquad \qquad \qquad \qquad \qquad + \hat \omega_{h+1/2}^{\A,3}(S_{h+1/2}, U_{h+1/2}; \pi) + \hat \theta_{h+1/2}^{\A,2}(S_{h+1/2}, U_{h+1/2}; \pi) \\
& \qquad \qquad \qquad \qquad \qquad \qquad + \hat \theta_{h+1/2}^{\A,3}(S_{h+1/2}, U_{h+1/2}; \pi). 
\#

\vskip5pt
\noindent\textbf{Estimation of $Q_h^{\B,\pi}$.}
For any $h \in [H]$, with an estimator $\hat Q_{h+1/2}^\B(\cdot, \cdot, \cdot, \cdot; \pi)$, we estimate $Q_h^{\B,\pi}(\cdot, \cdot, \cdot, \cdot)$ as follows. 
Since we have estimators $\hat \theta_{h+1/2}^\B$, $\hat \gamma_{h+1/2}^\B$, and $\hat \omega_{h+1/2}^\B$ for $\tilde \theta_{h+1/2}^\B$, $\tilde \gamma_{h+1/2}^\B$, and $\tilde \omega_{h+1/2}^\B$, respectively from the estimator $\hat Q^\B_{h+1/2}$, we first estimate the following three terms, 
\$
& \EE\left[\hat \theta_{h+1/2}^\B(S_{h+1/2}, U_{h+1/2}; \pi) \given S_h, A_h , B_{h-1/2}, U_h, V_h \right], \\
& \EE\left[\hat \gamma_{h+1/2}^\B(S_{h+1/2}, U_{h+1/2}; \pi)\cdot \pi^\B_{h+1/2}(S_{h+1/2}) \given S_h, A_h , B_{h-1/2}, U_h, V_h \right], \\
& \EE\left[\hat \omega_{h+1/2}^\B(S_{h+1/2}, U_{h+1/2}; \pi) \cdot \pi^\B_{h+1/2}(S_{h+1/2}) \given S_h, A_h , B_{h-1/2}, U_h, V_h \right]
\$
by
\#\label{eq:est-4}
& \hat \theta_h^{\B,1}(S_h,U_h;\pi) \cdot A_h + \hat \gamma_h^{\B,1}(S_h,U_h;\pi) \cdot B_{h-1/2} + \hat \omega_h^{\B,1}(S_h,U_h;\pi) \cdot A_h B_{h-1/2}, \\
& \hat \theta_h^{\B,2}(S_h,U_h;\pi) \cdot A_h + \hat \gamma_h^{\B,2}(S_h,U_h;\pi) \cdot B_{h-1/2} + \hat \omega_h^{\B,2}(S_h,U_h;\pi) \cdot A_h B_{h-1/2}, \\
& \hat \theta_h^{\B,3}(S_h,U_h;\pi) \cdot A_h + \hat \gamma_h^{\B,3}(S_h,U_h;\pi) \cdot B_{h-1/2} + \hat \omega_h^{\B,3}(S_h,U_h;\pi) \cdot A_h B_{h-1/2},
\#
respectively. We remark that the above estimators in \eqref{eq:est-4} are constructed following from a similar idea as in \eqref{eq:est-2}.
Similar as in \eqref{eq:est-qha}, we obtain the following estimator of $Q_h^{\B,\pi}$, 
\#\label{eq:est-qha3}
& \hat Q^\B_h(S_h, A_h , B_{h-1/2}, U_h; \pi) \\
& \qquad = \hat \theta_h^\B(S_h, U_h; \pi)\cdot A_h  + \hat \gamma_h^\B(S_h, U_h; \pi) \cdot B_{h-1/2} + \hat \omega_h^\B(S_h, U_h; \pi) \cdot A_h  B_{h-1/2},
\#
where 
\#
& \hat \theta_h^\B(S_h, U_h; \pi) = \hat \theta_h^{\B,1}(S_h,U_h;\pi) + \hat \theta_h^{\B,2}(S_h,U_h;\pi) + \hat \theta_h^{\B,3}(S_h,U_h;\pi), \\
& \hat \gamma_h^\B(S_h, U_h; \pi) = \hat \gamma_h^{\B,2}(S_h,U_h;\pi), \\
& \hat \omega_h^\B(S_h, U_h; \pi) = \hat \omega_h^{\B,1}(S_h,U_h;\pi) + \hat \omega_h^{\B,2}(S_h,U_h;\pi) + \hat \omega_h^{\B,3}(S_h,U_h;\pi) \\
& \qquad \qquad \qquad \qquad + \hat \gamma_h^{\B,1}(S_h,U_h;\pi) + \hat \gamma_h^{\B,3}(S_h,U_h;\pi). 
\#

\vskip5pt
\noindent\textbf{Estimation of $Q_{h+1/2}^{\B,\pi}$.}
For any $h \in [H]$, we estimate $Q_{h+1/2}^{\B,\pi}(\cdot, \cdot, \cdot, \cdot)$ as follows. 
Since we have estimators $\hat \theta_{h+1}^\B$, $\hat \gamma_{h+1}^\B$, and $\hat \omega_{h+1}^\B$ of $\tilde \theta_{h+1}^\B$, $\tilde \gamma_{h+1}^\B$, and $\tilde \omega_{h+1}^\B$, respectively from the estimator $\hat Q^\B_{h+1}$, we first estimate the following terms,
\$
& \EE\left[R_{h+1/2}^\B \given S_{h+1/2}, A_h , B_{h+1/2}, U_{h+1/2}, V_{h+1/2} \right], \\
& \EE\left[\hat \theta_{h+1}^\B(S_{h+1}, U_{h+1}^\A; \pi) \cdot \pi^\B_{h+1}(S_{h+1}, U_{h+1}^\A) \given S_{h+1/2}, A_h , B_{h+1/2}, U_{h+1/2}, V_{h+1/2} \right], \\
& \EE\left[\hat \gamma_{h+1}^\B(S_{h+1}, U_{h+1}^\A; \pi) \given S_{h+1/2}, A_h , B_{h+1/2}, U_{h+1/2}, V_{h+1/2} \right], \\
& \EE\left[\hat \omega_{h+1}^\B(S_{h+1}, U_{h+1}^\A; \pi) \cdot \pi^\B_{h+1}(S_{h+1}, U_{h+1}^\A) \given S_{h+1/2}, A_h , B_{h+1/2}, U_{h+1/2}, V_{h+1/2} \right]
\$
by
\#\label{eq:est-5}
& \hat \beta_a^\B(S_{h+1/2},U_{h+1/2}) A_h  + \hat \beta_z^\B(S_{h+1/2},U_{h+1/2})  B_{h+1/2} + \hat \beta_{az}^\B(S_{h+1/2},U_{h+1/2}) A_h   B_{h+1/2}, \\
& \hat \theta_{h+1/2}^{\B,1}(S_{h+1/2},U_{h+1/2};\pi) \cdot A_h + \hat \gamma_{h+1/2}^{\B,1}(S_{h+1/2},U_{h+1/2};\pi) \cdot B_{h+1/2} \\
& \qquad + \hat \omega_{h+1/2}^{\B,1}(S_{h+1/2},U_{h+1/2};\pi) \cdot A_h B_{h+1/2}, \\
& \hat \theta_{h+1/2}^{\B,2}(S_{h+1/2},U_{h+1/2};\pi) \cdot A_h + \hat \gamma_{h+1/2}^{\B,2}(S_{h+1/2},U_{h+1/2};\pi) \cdot B_{h+1/2} \\
& \qquad + \hat \omega_{h+1/2}^{\B,2}(S_{h+1/2},U_{h+1/2};\pi) \cdot A_h B_{h+1/2}, \\
& \hat \theta_{h+1/2}^{\B,3}(S_{h+1/2},U_{h+1/2};\pi) \cdot A_h + \hat \gamma_{h+1/2}^{\B,3}(S_{h+1/2},U_{h+1/2};\pi) \cdot B_{h+1/2} \\
& \qquad + \hat \omega_{h+1/2}^{\B,3}(S_{h+1/2},U_{h+1/2};\pi) \cdot A_h B_{h+1/2}, 
\#
respectively. We remark that the above estimators in \eqref{eq:est-5} are constructed following from a similar idea as in \eqref{eq:est-2}. Similar as in \eqref{eq:est-qha}, we obtain the following estimator of $Q_{h+1/2}^{\B,\pi}$, 
\#\label{eq:est-qha4}
& \hat Q^\B_{h+1/2}(S_{h+1/2}, A_h , B_{h+1/2}, U_{h+1/2}; \pi)   \\
& \qquad = \hat \theta_{h+1/2}^\B(S_{h+1/2}, U_{h+1/2}; \pi) \cdot A_h  + \hat \gamma_{h+1/2}^\B(S_{h+1/2}, U_{h+1/2}; \pi) \cdot B_{h+1/2}   \\
& \qquad \qquad+ \hat \omega_{h+1/2}^\B(S_{h+1/2}, U_{h+1/2}; \pi) \cdot A_h  B_{h+1/2}. 
\#
where 
\#
& \hat \theta_{h+1/2}^\B(S_{h+1/2}, U_{h+1/2}; \pi) = \hat \beta_a^\B(S_{h+1/2},U_{h+1/2}) + \hat \theta_{h+1/2}^{\B,1}(S_{h+1/2}, U_{h+1/2}; \pi), \\
& \hat \gamma_{h+1/2}^\B(S_{h+1/2}, U_{h+1/2}; \pi) = \hat \beta_z^\B(S_{h+1/2},U_{h+1/2}) + \hat \gamma_{h+1/2}^{\B,1}(S_{h+1/2}, U_{h+1/2}; \pi) \\
& \qquad \qquad \qquad \qquad \qquad \qquad + \hat \gamma_{h+1/2}^{\B,2}(S_{h+1/2}, U_{h+1/2}; \pi) + \hat \gamma_{h+1/2}^{\B,3}(S_{h+1/2}, U_{h+1/2}; \pi), \\
& \hat \omega_{h+1/2}^\B(S_{h+1/2}, U_{h+1/2}; \pi) = \hat \beta_{az}^\B(S_{h+1/2},U_{h+1/2}) + \hat \omega_{h+1/2}^{\B,1}(S_{h+1/2}, U_{h+1/2}; \pi)  \\ 
& \qquad \qquad \qquad \qquad \qquad \qquad + \hat \omega_{h+1/2}^{\B,2}(S_{h+1/2}, U_{h+1/2}; \pi) + \hat \omega_{h+1/2}^{\B,3}(S_{h+1/2}, U_{h+1/2}; \pi) \\
& \qquad \qquad \qquad \qquad \qquad \qquad + \hat \theta_{h+1/2}^{\B,2}(S_{h+1/2}, U_{h+1/2}; \pi) + \hat \theta_{h+1/2}^{\B,3}(S_{h+1/2}, U_{h+1/2}; \pi). 
\#

\vskip5pt
\noindent\textbf{Estimation of $J^\A$ and $J^\B$.}
We estimate $J^\A(\pi^\A, \pi^\B)$ as
\$
\hat J^\A(\pi^\A, \pi^\B) = \EE_\pi \left[ \hat Q_1^\A(S_1, A_1^\A, A_{1/2}^\B, U_1; \pi) \right], 
\$
where the expectation $\EE_\pi[\cdot]$ is taken with respect to 
$A^\B_{1/2} \sim \pi_{1/2}^\B(\cdot)$, $S_1\sim \nu(\cdot)$, $U_1$, and $A_1^\A \sim \pi_1^\A(\cdot \given S_1, U_1)$. Similarly, we construct an estimator $\hat J^\B(\pi^\A, \pi^\B)$ for $J^\B(\pi^\A, \pi^\B)$.

\subsection{Construction of Confidence Regions}\label{sec:iureiuw}
In a multi-stage game with $H \geq 1$, with a similar evaluation procedure proposed in \S\ref{sec:full-eval}, we present our method to find an optimal policy in this section. To begin with, in light of \eqref{eq:iuwervbuire} in \S\ref{sec:id-all-wo-oracle} of the appendix, we have the following equations 
\#\label{eq:nuwiehv}
& W^{\Diamond, \r}(\tilde \bff^{\Diamond}) = \begin{pmatrix}
    \Phi^{\Diamond, \r}(\tilde \bff^{\Diamond}) \bbeta^\Diamond \\
    \varphi^{\Diamond, \r}(\tilde \bff^{\Diamond})
\end{pmatrix} + \alpha^{\Diamond, \r}(\tilde \bff^{\Diamond}), \\
& W_h^{\Diamond,j,\pi}(\tilde \btheta_{h+1/2}^{\Diamond}, \tilde \bff_h^{\Diamond}) = \begin{pmatrix}
    \Phi_h^{\Diamond,j,\pi}(\tilde \btheta_{h+1/2}^{\Diamond}, \tilde \bff_h^{\Diamond}) \btheta_h^{\Diamond,j,\pi}(\tilde \btheta_{h+1/2}^{\Diamond})\\
    \varphi_h^{\Diamond,j,\pi}(\tilde \btheta_{h+1/2}^{\Diamond}, \tilde \bff_h^{\Diamond})
\end{pmatrix} + \alpha_h^{\Diamond,j,\pi} (\tilde \btheta_{h+1/2}^{\Diamond}, \tilde \bff_h^{\Diamond}), \\
& W_{h+1/2}^{\Diamond,j,\pi}(\tilde \btheta_{h+1}^{\Diamond}, \tilde \bff_{h+1/2}^{\Diamond}) = \begin{pmatrix}
    \Phi_{h+1/2}^{\Diamond,j,\pi}(\tilde \btheta_{h+1}^{\Diamond}, \tilde \bff_{h+1/2}^{\Diamond}) \btheta_{h+1/2}^{\Diamond,j,\pi}(\tilde \btheta_{h+1}^{\Diamond}) \\
    \varphi_{h+1/2}^{\Diamond,j,\pi}(\tilde \btheta_{h+1}^{\Diamond}, \tilde \bff_{h+1/2}^{\Diamond})
\end{pmatrix} + \alpha_{h+1/2}^{\Diamond,j,\pi} (\tilde \btheta_{h+1}^{\Diamond}, \tilde \bff_{h+1/2}^{\Diamond}), 
\#
where 
\$
& \Phi^{\Diamond, \r}(\tilde \bff^{\Diamond}), ~\Phi_h^{\Diamond,j,\pi}(\tilde \btheta_{h+1/2}^{\Diamond}, \tilde \bff_h^{\Diamond}), ~\Phi_{h+1/2}^{\Diamond,j,\pi}(\tilde \btheta_{h+1}^{\Diamond}, \tilde \bff_{h+1/2}^{\Diamond}) \in \RR^{3\times 3}, \\ 
& \varphi^{\Diamond, \r}(\tilde \bff^{\Diamond}), ~\varphi_h^{\Diamond,j,\pi}(\tilde \btheta_{h+1/2}^{\Diamond}, \tilde \bff_h^{\Diamond}), ~\varphi_{h+1/2}^{\Diamond,j,\pi}(\tilde \btheta_{h+1}^{\Diamond}, \tilde \bff_{h+1/2}^{\Diamond}) \in \RR^5, \\
& \alpha^{\Diamond, \r}(\tilde \bff^{\Diamond}), ~\alpha_h^{\Diamond,j,\pi} (\tilde \btheta_{h+1/2}^{\Diamond}, \tilde \bff_h^{\Diamond}), ~\alpha_{h+1/2}^{\Diamond,j,\pi} (\tilde \btheta_{h+1}^{\Diamond}, \tilde \bff_{h+1/2}^{\Diamond}) \in \RR^3
\$ 
for any $\tilde \bff^{\Diamond}, \tilde \bff_h^{\Diamond}, \tilde \bff_{h+1/2}^{\Diamond} \in \cH^5$, $\tilde \btheta_{h+1/2}^{\Diamond}, \tilde \btheta_{h+1}^{\Diamond} \in \cH^3$,  $\Diamond\in \{\A,\B\}$, $j\in [3]$, and $h\in [H]$. 
Here $W^{\Diamond, \r}(\tilde \bff^{\Diamond})$, $W_h^{\Diamond,j,\pi}(\tilde \btheta_{h+1/2}^{\Diamond}, \tilde \bff_h^{\Diamond})$, and $W_{h+1/2}^{\Diamond,j,\pi}(\tilde \btheta_{h+1}^{\Diamond}, \tilde \bff_{h+1/2}^{\Diamond})$ are constructed via \eqref{eq:weiuhfurwvrtg}. With \eqref{eq:nuwiehv}, we further define the loss functions for the construction of estimators in \eqref{eq:est-qha}, \eqref{eq:est-qha2}, \eqref{eq:est-qha3}, and \eqref{eq:est-qha4} as follows, 
\#\label{eq:wiuehfrwe444}
& \hat L^{\Diamond,\r}(\tilde \bbeta^{\Diamond}; \tilde \bff^{\Diamond}) = \frac{1}{n} \sum_{i = 1}^n \hat m^{i,\A,\r}(\tilde \bbeta^{\Diamond}; \tilde \bff^{\Diamond})^\top  \hat m^{i,\A,\r}(\tilde \bbeta^{\Diamond}; \tilde \bff^{\Diamond}), \\
& \hat L_h^{\Diamond,j,\pi}(\tilde \btheta_h^{\Diamond,j}, \tilde \btheta_{h+1/2}^{\Diamond}; \tilde \bff_h^{\Diamond}) = \frac{1}{n} \sum_{i = 1}^n \hat m_h^{i,\B,j,\pi}(\tilde \btheta_h^{\Diamond,j}, \tilde \btheta_{h+1/2}^{\Diamond}; \tilde \bff_h^{\Diamond})^\top  \hat m_h^{i,\B,j,\pi}(\tilde \btheta_h^{\Diamond,j}, \tilde \btheta_{h+1/2}^{\Diamond}; \tilde \bff_h^{\Diamond}), \\
& \hat L_{h+1/2}^{\Diamond,j,\pi}(\tilde \btheta_{h+1/2}^{\Diamond,j}, \tilde \btheta_{h+1}^{\Diamond}; \tilde \bff_{h+1/2}^{\Diamond}) = \frac{1}{n} \sum_{i = 1}^n \hat m_{h+1/2}^{i,\B,j,\pi}(\tilde \btheta_{h+1/2}^{\Diamond,j}, \tilde \btheta_{h+1}^{\Diamond}; \tilde \bff_{h+1/2}^{\Diamond})^\top  \hat m_{h+1/2}^{i,\B,j,\pi}(\tilde \btheta_{h+1/2}^{\Diamond,j}, \tilde \btheta_{h+1}^{\Diamond}; \tilde \bff_{h+1/2}^{\Diamond}), 
\#
where $\hat m^{i,\A,\r}(\tilde \bbeta^{\Diamond}; \tilde \bff^{\Diamond})$, $\hat m_h^{i,\B,j,\pi}(\tilde \btheta_h^{\Diamond,j}, \tilde \btheta_{h+1/2}^{\Diamond}; \tilde \bff_h^{\Diamond})$, and $\hat m_{h+1/2}^{i,\B,j,\pi}(\tilde \btheta_{h+1/2}^{\Diamond,j}, \tilde \btheta_{h+1}^{\Diamond}; \tilde \bff_{h+1/2}^{\Diamond})$ are consistent estimators of \eqref{eq:nuwiehv}. 
By minimizing the loss functions in \eqref{eq:wiuehfrwe444}, for any policy pair $\pi$, we can obtain the following estimators needed to construct estimators of action-value functions defined in \eqref{eq:est-qha}, \eqref{eq:est-qha2}, \eqref{eq:est-qha3}, and \eqref{eq:est-qha4}, 
\#\label{eq:uihrgiur}
& \hat \bbeta^\Diamond = \left(\hat \beta_a^\Diamond, \hat \beta_z^\Diamond, \hat \beta_{az}^\Diamond\right), \quad \hat \btheta_h^{\Diamond,j,\pi} = \left(\hat \theta_h^{\Diamond,j,\pi}, \hat \gamma_h^{\Diamond,j,\pi}, \hat \omega_h^{\Diamond,j,\pi}\right),
\#
for any $j\in [3]$, $h\in [H]\cup [H]_{1/2}$, and $\Diamond\in \{\A, \B\}$.  
With \eqref{eq:wiuehfrwe444} and \eqref{eq:uihrgiur}, we define the following confidence regions, 
\#\label{eq:conf-genius}
& \ci^{\Diamond,\r} = \left\{\left(\tilde \beta_a^\Diamond, \tilde \beta_z^\Diamond, \tilde \beta_{az}^\Diamond\right) \colon \hat L^{\Diamond,\r}(\tilde \bbeta^{\Diamond}; \hat \bff^{\Diamond}) - \hat L^{\Diamond,\r}(\hat \bbeta^{\Diamond}; \hat \bff^{\Diamond}) \leq \eta^{\Diamond,\r} \right\}, \\
& \ci_h^{\Diamond,j,\pi}(\tilde \btheta_{h+1/2}^{\Diamond}) = \left\{\left(\tilde \theta_h^{\Diamond,j}, \tilde \gamma_h^{\Diamond,j}, \tilde \omega_h^{\Diamond,j} \right) \colon \hat L_h^{\Diamond,j,\pi}(\tilde \btheta_h^{\Diamond,j}, \tilde \btheta_{h+1/2}^{\Diamond}; \hat \bff_h^{\Diamond,j}) - \hat L_h^{\Diamond,j,\pi}(\hat \btheta_h^{\Diamond,j,\pi}, \tilde \btheta_{h+1/2}^{\Diamond}; \hat \bff_h^{\Diamond,j}) \leq \eta_h^{\Diamond,j} \right\}, \\
& \ci_h^{\Diamond,\pi}(\tilde \btheta_{h+1/2}^{\Diamond}) = \left\{\left(\tilde \theta_h^{\Diamond}, \tilde \gamma_h^{\Diamond}, \tilde \omega_h^{\Diamond} \right) \colon \left(\tilde \theta_h^{\Diamond,j}, \tilde \gamma_h^{\Diamond,j}, \tilde \omega_h^{\Diamond,j} \right) \in \ci_h^{\Diamond,j,\pi}(\tilde \btheta^{\Diamond}_{h+1/2}) \text{ for any $j\in [3]$} \right\},
\#
for any $j\in [3]$, $h\in [H]\cup [H]_{1/2}$, and $\Diamond\in \{\A, \B\}$. 
Here we denote by $\hat \btheta_{H}^\A = \hat \bbeta^\A$ and $\hat \btheta_{H+1/2}^\B = \hat \bbeta^\B$ for notational convenience.

\section{Estimation Procedure without Oracle}\label{sec:id-all-wo-oracle}
The estimators constructed in \eqref{eq:est-qha}, \eqref{eq:est-qha2}, \eqref{eq:est-qha3}, and \eqref{eq:est-qha4} follow a similar procedure. 
Specifically, for a given policy pair $\pi$ and a function $\vartheta$, we aim to estimate the conditional expectation of the form, 
\$
& \EE\left[Y\given S, A, B, U\right] = \vartheta_a^*(S,U) A + \vartheta_z^*(S,U) B + \vartheta_{az}^*(S,U) A B,
\$
where $Y = g(R, S', U')$ and the expectation $\EE[\cdot]$ is taken following the policy pair $\pi$ and the turn-based model introduced in \S\ref{sec:model}. Note here we drop the subscripts for notational convenience. 
Specifically, we want to construct estimators for $\vartheta_a^*$, $\vartheta_z^*$, and $\vartheta_{az}^*$.

\subsection{Identification Results}\label{appendix: lm 1-3}

We rely on the following three lemmas for our identification results in Theorem \ref{theorem: basic id}.
\begin{lemma}\label{lemma:id1}
We have
\$
& \EE\left[ \left( B - \EE[B\given S, U ] \right) \cdot \left( A - \EE[A\given B, S, U] \right)\cdot Y \given S, U \right] \\
& \quad = \EE\left[ \left( B - \EE[B \given S, U ] \right) \cdot \left( A - \EE[A\given B, S, U] \right)\cdot A \given S, U \right] \cdot \vartheta_a^*(S,U) \\
& \quad \quad + \EE\left[ B \cdot \left( B - \EE[B \given S, U ] \right) \cdot \left( A - \EE[A\given B, S, U] \right)\cdot A \given S, U \right] \cdot \vartheta_{az}^*(S, U). 
\$
\end{lemma}

\begin{lemma}\label{lemma:id2}
We have
\$
& \EE\left[ \left( B - \EE[B\given S, U] \right) \cdot Y \given S, U \right] \\
& \quad = \EE\left[ A \cdot \left( B - \EE[B \given S, U ] \right) \given S, U \right] \cdot \vartheta_a^*(S, U) \\
& \quad \quad + \EE\left[ B \cdot \left( B - \EE[B\given S, U ] \right) \given S, U \right] \cdot \vartheta_{z}(S, U) \\
& \quad \quad + \EE\left[ A B \cdot \left( B - \EE[B\given S, U ] \right) \given S, U \right] \cdot \vartheta_{az}^*(S, U). 
\$
\end{lemma}

\begin{lemma}\label{lemma:id3}
We have
\$
& \Cov\left( B\cdot \left( B - \EE[B\given S, U ] \right), \left( A - \EE[A\given B, S, U] \right)\cdot Y \given S, U \right) \\
& \quad = \Cov\left( B \cdot \left( B - \EE[B\given S, U ] \right), A \cdot \left( A - \EE[A\given B, S, U] \right) \given S, U \right) \cdot \vartheta_a^*(S, U) \\
& \quad \quad + \Cov\left( B \cdot \left( B - \EE[B\given S, U ] \right), AB\cdot \left( A - \EE[A\given B, S, U] \right) \given S, U \right) \cdot \vartheta_{az}^*(S, U). 
\$
\end{lemma}

By Lemmas \ref{lemma:id1}, \ref{lemma:id2}, and \ref{lemma:id3}, we can identify $\vartheta_a^*$, $\vartheta_z^*$, and $\vartheta_{az}^*$ non-parametrically. 

\subsection{Estimation} 
We define the following nuisance parameters, 
\$
& f_1(S, U) = \EE[B\given S, U], \quad f_2(X, B) = \EE[A\given X, B], \quad f_3(S, U) = \EE[(A-f_2(S, B)) Y \given S, U ], \\
& f_4(S, U) = \EE[A(A - f_2(S, B))\given S, U], \qquad f_5(S, U) = \EE[AB (A - f_2(S, B))\given S, U]. 
\$
Further, we have
\#\label{eq:weiuhfurwvrtg}
& w_1(Y, A, B, S, U, f_1, f_2) = \rho_1(Y, S, B, U, f_1, f_2) - \rho_2(S, B, U, f_1, f_2) \vartheta_a^*(S, U) \\
& \qquad\qquad\qquad\qquad\qquad\qquad\qquad - \rho_3(S, B, U, f_1, f_2) \vartheta_{az}^*(S, U), \\
& w_2(Y, A, B, S, U, f_1) = \rho_4(Y, S, U, B, f_1) - \rho_2(S, B, U, f_1) \vartheta_a^*(S, U) \\
& \qquad\qquad\qquad\qquad\qquad\qquad - \rho_6(S, B, U, f_1) \vartheta_z^*(S, U) - \rho_7(S, B, U, f_1) \vartheta_{az}^*(S, U), \\
& w_3(Y, A, B, S, U, f_1, f_2, f_3) = \rho_8(Y, S, B, U, f_1, f_2) - (1 - f_1(S, U))f_1(S, U) f_3(S, U) \\
& \qquad\qquad\qquad\qquad\qquad\qquad\qquad\quad - \left( \rho_9(S, B, U, f_1, f_2) - (1 - f_1(S, U)) f_1(S, U) f_4(S, U) \right)\vartheta_a^*(S, U) \\
& \qquad\qquad\qquad\qquad\qquad\qquad\qquad\quad  - \left( \rho_{10}(S, B, U, f_1, f_2) - (1 - f_1(S, U)) f_1(S, U) f_5(S, U) \right)\vartheta_{az}^*(S, U), \\
& w_4(S, B, U, f_1) = B - f_1(S, U), \\
& w_5(A, B, S, U, f_2) = A - f_2(S, B, U), \\
& w_6(Y, A, B, S, U, f_2, f_3) = (A - f_2(S, B, U))\cdot Y - f_3(S, U), \\
& w_7(A, B, S, U, f_2, f_4) = (A - f_2(S, B, U))A - f_4(S, U), \\
& w_8(A, B, S, U, f_2, f_5) = (A - f_2(S, B, U))AB - f_5(S), 
\#
where 
\#\label{eq:rhos}
& \rho_1(Y, S, A, B, U, f_1, f_2) = (B - f_1(S, U)) \cdot (A - f_2(S, B, U))\cdot Y, \\
& \rho_2(S, A, B, U, f_1, f_2) = (B - f_1(S, U)) \cdot (A - f_2(S, B, U))\cdot A, \\
& \rho_3(S, A, B, U, f_1, f_2) = B (B - f_1(S, U)) \cdot (A - f_2(S, B, U))\cdot A, \\
& \rho_4(Y, S, B, U, f_1) = (B - f_1(S, U)) \cdot Y, \\
& \rho_5(S, A, B, U, f_1) = (B - f_1(S, U)) \cdot A, \\
& \rho_6(S, B, U, f_1) = B(B - f_1(S, U)), \\
& \rho_7(S, A, B, U, f_1) = AB (B - f_1(S, U)), \\
& \rho_8(Y, S, A, B, U, f_1, f_2) = B(B - f_1(S, U))\cdot (A - f_2(S, B, U))\cdot Y, \\
& \rho_9(S, A, B, U, f_1, f_2) = A(B - f_1(S, U))\cdot (A - f_2(S, B, U)), \\
& \rho_{10}(S, A, B, U, f_1, f_2) = AB(B - f_1(S, U))\cdot (A - f_2(S, B, U)). 
\#

We denote by $W(Y, S, A, B, U, \varpi^*) = (w_1, w_2, \ldots, w_8)$, where $\varpi^* = (\vartheta_a^*, \vartheta_z^*, \vartheta_{az}^*, f_1, f_2, \ldots, f_5)$. 
Note that we can write
\#\label{eq:iuwervbuire}
W = \begin{pmatrix}
\Phi(\bff) \bm{\vartheta} \\
\varphi(\bff)
\end{pmatrix} + \alpha(f), 
\#
where $\Phi(\bff)\in \RR^{3\times 3}$, $\varphi(\bff)\in \RR^5$, $\bm{\vartheta} = (\vartheta_a^*, \vartheta_z^*, \vartheta_{az}^*)^\top$, and $\bff = (f_1, f_2, \ldots, f_5)^\top$. 
Then by Lemmas \ref{lemma:id1}, \ref{lemma:id2}, and \ref{lemma:id3}, we have the following conditional moment restriction, 
\#
\EE[W(Y, S, A, B, U, \varpi^*)\given S, U] = 0.
\#
We aim to recover $\varpi$ via such a conditional moment restriction. We consider the following sieve minimum distance (SMD) estimator, 
\#\label{eq:psmd}
& \hat \varpi_n \in \argmin_{\varpi\in \cH_n} \hat L(\varpi), \\
& \text{where } \hat L(\varpi) = \frac{1}{n} \sum_{i = 1}^n \hat m(Y^i, S^i, U^i, A^i, B^i; \varpi)^\top  \hat m(Y^i, S^i, U^i, A^i, B^i; \varpi). 
\#
Here $\hat m(Y^i, S^i, U^i, A^i, B^i)$ is a consistent estimator of $\EE[W(Y, S, A, B, U, \varpi)\given S, U]$ and $\cH_n$ is a sieve parameter space. The following result characterizes the convergence of $\hat \varpi_n$ to $\varpi^*$. In the meanwhile, we denote by 
\$
L(\varpi) = \EE \left[m(S,U,\varpi)^\top m(S,U,\varpi) \right],
\$
where $m(S,U,\varpi) = \EE[W(Y, S, A, B, U, \varpi)\given S, U]$. 
We know that $L(\varpi^*) = 0$ by the conditional moment restriction.

We first introduce the definition of sieve space $\cH_n$ as  follows,  
\$
\cH_n = \left\{ \varpi\in \cH \colon \varpi(\cdot) = \sum_{k=1}^{k(n)} a_k q_k(\cdot) \right\},
\$
where $k(n)<\infty$, $k(n)\to \infty$ as $n\to\infty$, and $\{q_k\}_{k=1}^\infty$ is a sequence of known basis functions of a Banach space. 

We introduce the following assumptions. 

\begin{assumption}[Sieves]\label{ass:sieves}
    The following statements hold. 
    \begin{itemize}
        \item We have $\|\varpi^* - \varpi\| = 0$ for any $\varpi\in \cH$ such that $\EE[W(Y,S,A,B,U,\varpi)\given S,U]=0$. 
        \item The sequence of sieve spaces $\{\cH_k\}_{k\geq 1}$ is a sequence of nonempty closed subsets satisfying $\cH_k\subseteq \cH_{k+1}\subseteq \cH$, and there exists $\Gamma_n \varpi^*\in \cH_{k(n)}$ such that $\|\Gamma_n \varpi^* - \varpi^*\| = o(1)$. 
        \item We have $\EE[\|m(S,U,\Gamma_n \varpi^*)\|^2] = o(1)$. 
    \end{itemize}
\end{assumption}

\begin{assumption}[Sample Criterion]\label{ass:sample}
    The following statements hold. 
    \begin{itemize}
        \item We have $\frac{1}{n}\sum_{i=1}^n \| \hat m(Y^i, S^i, U^i, A^i, B^i; \Gamma_n \varpi^*) \|^2 \leq c_0\cdot \EE[\| m(S, U, \Gamma_n \varpi^*) \|^2] + O_p(\eta_n) $, where $c_0>0$ is a constant and $\eta_n = o(1)$. 
        \item We have $\frac{1}{n}\sum_{i=1}^n \| \hat m(Y^i, S^i, U^i, A^i, B^i; \varpi) \|^2 \geq c\cdot \EE[\| m(S, U, \varpi) \|^2] - O_p(\delta_n^2) $ for any $\varpi\in \cH_{k(n)}$, where $c>0$ is a constant and $\delta_n^2 = c\cdot k(n)/n = o(1)$ satisfying $\eta_n = O(\delta_n^2)$. 
    \end{itemize}
\end{assumption}

With Assumptions \ref{ass:sieves} and \ref{ass:sample}, we have the following result showing that $\hat \varpi_n$ is a consistent estimator. 

\begin{theorem}[Consistency, \cite{chen2012estimation}]\label{thm:consistency}
    Suppose Assumptions \ref{ass:sieves} and \ref{ass:sample} hold. We have $\left\|\hat \varpi_n - \varpi^*\right\| = o_p(1)$, where $\hat \varpi_n$ is the SMD estimator defined in \eqref{eq:psmd}. 
\end{theorem}

Given the consistency result in Theorem \ref{thm:consistency}, we can restrict our space $\cH$ to a shrinking neighborhood around the ground truth $\varpi^*$. We define
\$
\cH_s = \left\{ \varpi\in \cH \colon  \|\varpi - \varpi^*\| \leq \delta, \|\varpi\| \leq M  \right\}, \qquad \cH_{sn} = \cH_s \cap \cH_n
\$
for a positive constant $M$ and a sufficiently small positive $\delta$ such that $\PP(\hat \varpi_n \notin \cH_s) < \delta$. Meanwhile, we define the path-wise derivative in the direction $[\varpi - \varpi^*]$ evaluated at $\varpi^*$ as follows, 
\$
\frac{\ud m(S,U,\varpi^*)}{\ud \varpi}[\varpi - \varpi^*] = \frac{\ud \EE[W(Y,S,A,B,U,(1-\tau)\varpi + \tau\varpi^*)\given S,U]}{\ud \tau} \bigggiven_{\tau = 0}. 
\$
We define the following notation of pseudometric for any $\varpi_1,\varpi_2\in \cH_s$, 
\$
\|\varpi_1 - \varpi_2\|_\ps = \sqrt{\EE\left[ \left\| \frac{\ud m(S,U,\varpi^*)}{\ud \varpi} [\varpi_1 - \varpi_2] 
\right\|^2\right]}. 
\$

We introduce the following assumption. 

\begin{assumption}[Local Curvature]\label{ass:curv}
    The following statements hold. 
    \begin{itemize}
        \item $\cH_s$ and $\cH_{sn}$ are convex, and $m(S, U, \varpi)$ is continuously path-wise differentiable with respect to $\varpi\in \cH_{s}$. 
        \item There exists a constant $c>0$ such that $\|\varpi - \varpi^*\|_\ps \leq c \cdot \|\varpi - \varpi^*\|$ for any $\varpi\in \cH_{s}$. 
        \item There are finite constants $c_1,c_2>0$ such that $\|\varpi - \varpi^*\|_\ps^2 \leq c_1 \EE[\|m(S,U,\varpi)\|^2]$ for any $\varpi\in \cH_{sn}$; and $c_2\EE[\|m(S,U,\Gamma_n \varpi^*)\|^2] \leq \|\Gamma_n \varpi^* - \varpi^*\|_\ps^2$. 
    \end{itemize}
\end{assumption}

\begin{assumption}[Sieve Approximation Error]
    We have $\|\varpi^* - \sum_{j = 1}^{k(n)} \la \varpi^*, q_j\ra q_j \| = O(\{\nu_{k(n)}^{-\alpha}\})$ for a finite $\alpha > 0$ and a positive sequence $\{\nu_j\}_{j=1}^\infty$ that strictly increases and $\nu_j = \Theta(j^{1/d})$. 
\end{assumption}

\begin{assumption}[Sieve Link Condition]\label{ass:link}
    There are finite constants $c, C > 0 $ and a continuous increasing function $\varphi\colon \RR^+ \to \RR^+$ such that 
    \begin{itemize}
        \item $\varphi(\tau) = \tau^\varsigma$ for some $\varsigma \geq 0$;
        \item $\|\varpi\|_\ps^2 \geq c \sum_{j=1}^\infty \varphi(\nu_j^{-2}) \cdot |\la \varpi, q_j\ra |^2$ for all $\varpi \in \cH_{sn}$;
        \item $\|\Gamma_n \varpi^* - \varpi^*\|_\ps^2 \leq C\sum_{j = 1}^\infty \varphi(\nu_j^{-2})\cdot | \la \Gamma_n \varpi^* - \varpi^*, q_j \ra|^2$. 
    \end{itemize}
\end{assumption}


\begin{theorem}[\cite{chen2012estimation}]\label{thm:chen2012-app}
Suppose that Assumptions \ref{ass:sieves}--\ref{ass:link} hold. We have
\$
\left\|\hat \varpi_n - \varpi^*\right\| = O_p\left( n^{-\frac{\alpha}{2\alpha+2\varsigma+d}} \right), 
\$
where $\hat \varpi_n$ is the SMD estimator defined in \eqref{eq:psmd}. 
\end{theorem}

We can thus construct plug-in estimators of $\vartheta_a^*$, $\vartheta_z^*$, and $\vartheta_{az}^*$ as $\hat \vartheta_a^*$, $\hat \vartheta_z^*$, and $\hat \vartheta_{az}^*$ and obtain the following convergence results. 

\begin{corollary}
Suppose that Assumptions \ref{ass:sieves}--\ref{ass:link} hold. We have
\$
\left\|\hat \vartheta_a^* - \vartheta_a^* \right\| = O_p\left( n^{-\frac{\alpha}{2\alpha+2\varsigma+d}} \right), \quad \left\|\hat \vartheta_z^* - \vartheta_z^* \right\| = O_p\left( n^{-\frac{\alpha}{2\alpha+2\varsigma+d}} \right), \quad \left\|\hat \vartheta_{az}^* - \vartheta_{az}^* \right\| = O_p\left( n^{-\frac{\alpha}{2\alpha+2\varsigma+d}} \right).
\$
\end{corollary}

\section{Proof of Identification Results}

\subsection{Proof of Lemma \ref{lemma:id1}}\label{prf:lemma:id1}
\begin{proof}
  Without loss of generality, we assume $(S,U) = \emptyset$. By Lemma 3.1 in \cite{tchetgen2021genius}, we can show that
  \begin{align*}
    &\EE\left[(B - \EE\left[B \given S,U\right])(A - \EE\left[A \given B, S,U \right])\left\{A\vartheta_a^*(V) + \vartheta_z^*(V)B + \vartheta_u^*(V) \right\} \right]\\
    = &  \EE\left[(B - \EE\left[B \given S,U\right])(A - \EE\left[A \given B, S,U \right])A\right] \times \EE\left[\vartheta_a^*(V)\right].
  \end{align*}
By direct calculation, we can show that
\begin{align*}
  &\EE\left[(B - \EE\left[B \given S,U\right])(A - \EE\left[A \given B\right])Y\right]\\
  =&\EE\left[(B - \EE\left[B \given S,U\right])(A - \EE\left[A \given B \right])\left\{A\vartheta_a^*(V) + \vartheta_z^*(V)B + \vartheta_z^*(V) \right\} \right]\\
  +& \EE\left[B(B - \EE\left[B \given S,U\right])(A - \EE\left[A \given B \right])A \vartheta_{a,z}^*(V)\right]\\
  = &  \EE\left[(B - \EE\left[B \given S,U\right])(A - \EE\left[A \given B \right])A\right] \times \EE\left[\vartheta_a^*(V)\right]\\
  +& \EE\left[B(B - \EE\left[B \given S,U\right])(1- \EE\left[A \given B \right])\vartheta_{a,z}^*(V)\left\{ \alpha^*_z(V)B + \alpha^*_{u}(V) \right\}\right]\\
  = &  \EE\left[(B - \EE\left[B \given S,U\right])(A - \EE\left[A \given B \right])A\right] \times \EE\left[\vartheta_a^*(V)\right]\\
  +& \EE\left[B(B - \EE\left[B \given S,U\right])(1- \EE\left[A \given B\right]A)\right]\times \EE\left[\vartheta_{a,z}^*(V)\right] \\
  +& \EE\left[B(B - \EE\left[B \given S,U\right])(1- \EE\left[A \given B\right])\left(\vartheta_{a,z}^*(V) - \EE\left[\vartheta_{a,z}^*(V)\right] \right)\left\{ \alpha^*_z(V)B + \alpha^*_{u}(V) \right\}\right]\\
= &  \EE\left[(B - \EE\left[B \given S,U\right])(A - \EE\left[A \given B \right])A\right] \times \EE\left[\vartheta_a^*(V)\right]\\
+& \EE\left[B(B - \EE\left[B \given S,U\right])(1- \EE\left[A \given B\right]A)\right]\times \EE\left[\vartheta_{a,z}^*(V)\right],
\end{align*}
where we use $Cov(\alpha^*_z(S,U,V), \vartheta_{a,z}^*(S,U,V) \given  S,U) =
Cov(\vartheta_{a,z}^*(S,U,V), \alpha^*_{u,s}(S,U,V) \given  S,U)= 0$ in the last equation.
\end{proof}

\subsection{Proof of Lemma \ref{lemma:id2}}\label{prf:lemma:id2}
\begin{proof}
  Again, without loss of generality, assume $(S,U) = \emptyset$. By direct calculation, we can show that
    \begin{align*}
    &\EE\left[(B - \EE\left[B\right])Y\right]\\
    = & \EE\left[(B - \EE\left[B \right])\left\{\vartheta^*_a(V)A + \vartheta^*_z(V)B + \vartheta^*_{a, z}(V)BA + \vartheta^*_u(V)  \right\}\right]\\
    =& \EE\left[(B - \EE\left[B \right])A \right] \times \EE\left[\vartheta^*_a(V)\right] + \EE\left[(B - \EE\left[B \right])\left(\alpha^*_z(V)B + \alpha^*_u(V)\right)\left(\vartheta^*_a(V) - \EE\left[\vartheta^*_a(V)\right]\right) \right] \\
    + & \EE\left[(B - \EE\left[B \right])B\right] \times \EE\left[\vartheta^*_z(V)\right]  + \EE\left[(B - \EE\left[B \right])B\vartheta^*_{a, z}(V)A\right] \\
    =& \EE\left[(B - \EE\left[B \right])A \right] \times \EE\left[\vartheta^*_a(V)\right]+\EE\left[(B - \EE\left[B \right])B\right] \times \EE\left[\vartheta^*_z(V)\right]  \\
    + &  \EE\left[(B - \EE\left[B \right])BA\right]\times \EE\left[\vartheta^*_{a, z}(V)\right] + \EE\left[(B - \EE\left[B \right])B\left(\vartheta^*_{a, z}(V) - \EE\left[\vartheta^*_{a, z}(V)\right] \right)\left(\alpha^*_z(V)B + \alpha^*_u(V)\right)\right]   \\
    =& \EE\left[(B - \EE\left[B \right])A \right] \times \EE\left[\vartheta^*_a(V)\right]+\EE\left[(B - \EE\left[B \right])B\right] \times \EE\left[\vartheta^*_z(V)\right]  \\
    + &  \EE\left[(B - \EE\left[B \right])BA\right]\times \EE\left[\vartheta^*_{a, z}(V)\right]    \\
  \end{align*}
\end{proof}

\subsection{Proof of Lemma \ref{lemma:id3}}\label{prf:lemma:id3}
\begin{proof}
  Again, we omit $(S,U)$. First, by similar arguments in the proof of Lemma \eqref{lemma:id1}, we can show that
  \begin{align*}
    &\EE\left[B(B - \EE\left[B\right])(A - \EE\left[A \given B\right])Y\right]\\
    =& \EE\left[B(B - \EE\left[B \right])(A - \EE\left[A \given B\right])A \right] \times \EE\left[\vartheta_a^*(V)\right]+\EE\left[B(B - \EE\left[B \right])(A - \EE\left[A \given B\right])A\right] \times \EE\left[\vartheta_{a,z}^*(V)\right]  \\
    + &   \EE\left[B(B - \EE\left[B \right])(A - \EE\left[A \given B\right])\vartheta_u^*(V)\right]   \\
    =& \EE\left[B(B - \EE\left[B \right])(A - \EE\left[A \given B\right])A \right] \times \EE\left[\vartheta_a^*(V)\right]+\EE\left[B(B - \EE\left[B \right])(A - \EE\left[A \given B\right])A\right] \times \EE\left[\vartheta_{a,z}^*(V)\right]  \\
    + &   \EE\left[B(B - \EE\left[B \right])\right]\Cov(\alpha^*_u(V), \vartheta_u^*(V)).
  \end{align*}
  Next, we derive that
  \begin{align*}
    &\EE\left[(A - \EE\left[A \given B\right])Y\right]\\
    =& \EE\left[(A - \EE\left[A \given B\right])\left\{\vartheta_a^*(V)A + \theta_z^*(V)B + \vartheta_{a,z}^*(V)BA + \vartheta_u^*(V)  \right\} \right] \\
    =& \EE\left[A(A - \EE\left[A \given B \right])\right] \times \EE\left[\vartheta_a^*(V)\right] + \EE\left[(1 - \EE\left[A \given B \right])\left(\alpha^*_z(V)B + \alpha^*_u(V)\right)\left(\vartheta_a^*(V) - \EE\left[\vartheta_a^*(V)\right]\right) \right] \\
    + &   \EE\left[BA(A - \EE\left[A \given B \right])\right] \times \EE\left[\vartheta_{a,z}^*(V)\right] + \EE\left[B(1 - \EE\left[A \given B \right])\left(\alpha^*_z(V)B + \alpha^*_u(V)\right)\left(\vartheta_{a,z}^*(V) - \EE\left[\vartheta_{a,z}^*(V)\right]\right) \right] \\
    + &  \Cov(\alpha^*_u(V), \vartheta_u^*(V))\\
    =& \EE\left[A(A - \EE\left[A \given B \right])\right] \times \EE\left[\vartheta_a^*(V)\right] + \EE\left[BA(A - \EE\left[A \given B \right])\right] \times \EE\left[\vartheta_{a,z}^*(V)\right] \\
    + &  \Cov(\alpha^*_u(V), \vartheta_u^*(V)).
  \end{align*}

Summarizing together, we obtain that
\begin{align*}
    &\EE\left[B(B - \EE\left[B\right])(A - \EE\left[A \given B\right])Y\right]\\
  =& \EE\left[B(B - \EE\left[B \right])(A - \EE\left[A \given B\right])A \right] \times \EE\left[\vartheta_a^*(V)\right]+\EE\left[B(B - \EE\left[B \right])(A - \EE\left[A \given B\right])A\right] \times \EE\left[\vartheta_{a,z}^*(V)\right]  \\
  + &   \EE\left[B(B - \EE\left[B \right])\right]\\
  \times & \left\{\EE\left[(A - \EE\left[A \given B\right])Y\right] -   \EE\left[A(A - \EE\left[A \given B \right])\right] \times \EE\left[\vartheta_a^*(V)\right] - \EE\left[BA(A - \EE\left[A \given B \right])\right] \times \EE\left[\vartheta_{a,z}^*(V)\right] \right\}\\
  =& \Cov(B(B - \EE\left[B \right]), A(A - \EE\left[A \given B \right]))\EE\left[\vartheta_a^*(V)\right] + \Cov(B(B - \EE\left[B \right]), BA(A - \EE\left[A \given B \right]))\EE\left[\vartheta_{a,z}^*(V)\right]\\
  + & \EE\left[B(B - \EE\left[B \right])\right] \times \EE\left[(A - \EE\left[A \given B\right])Y\right], 
\end{align*}
which concludes our proof.
\end{proof}

\section{Proofs for Single-Stage Game}
\subsection{Proof of Lemma \ref{lemma:pess}} \label{prf:lemma:pess}
\begin{proof}
To show that $Q_1^{\A,\pi}\in \ci_1^{\A,\Q,\pi}$, we only need to show that $(\beta_a^\A, \beta_z^\A, \beta_{az}^\A)\in \ci^{\A,\r}$, i.e., 
\#\label{eq:pess-wtp-1}
\hat L^{\A,\r}(\bbeta^{\A}) - \hat L^{\A,\r}(\hat \bbeta^{\A}) \leq \eta^{\A, \r}. 
\#
Note that 
\#\label{eq:815-1}
& \hat L^{\A,\r}(\bbeta^{\A}) - \hat L^{\A,\r}(\hat \bbeta^{\A}) \\
& \qquad = \hat L^{\A,\r}(\bbeta^{\A}) - L^{\A,\r}(\bbeta^{\A}) + L^{\A,\r}(\bbeta^{\A}) - L^{\A,\r}(\hat \bbeta^{\A}) + L^{\A,\r}(\hat \bbeta^{\A}) - \hat L^{\A,\r}(\hat \bbeta^{\A}) \\
& \qquad \leq \hat L^{\A,\r}(\bbeta^{\A}) - L^{\A,\r}(\bbeta^{\A}) + L^{\A,\r}(\hat \bbeta^{\A}) - \hat L^{\A,\r}(\hat \bbeta^{\A}) \\
& \qquad = \left(\hat L^{\A,\r}(\bbeta^{\A}) - \hat L^{\A,\r}(\hat \bbeta^{\A}) \right) - \left(L^{\A,\r}(\bbeta^{\A}) - L^{\A,\r}(\hat \bbeta^{\A}) \right)
\#
where in the first inequality, we use the fact that $L^{\A,\r}(\bbeta^{\A}) = 0$ while $L^{\A,\r}(\hat \bbeta^{\A}) \geq 0$. Now, by using Bernstein's inequality on the right-hand side of \eqref{eq:815-1}, we have
\#\label{eq:815-2}
\hat L^{\A,\r}(\bbeta^{\A}) - \hat L^{\A,\r}(\hat \bbeta^{\A}) & = O_p \left( \frac{c}{n} + \sqrt{\frac{\Var\left[\|\hat m(\bbeta^{\A})\|^2 - \|\hat m(\hat \bbeta^{\A})\|^2 \right]}{n} } \right) \\
& \leq O_p \left(\frac{c}{n} + \sqrt{\frac{\EE[\left(\|\hat m(\bbeta^{\A})\|^2 - \|\hat m(\hat \bbeta^{\A})\|^2 \right)^2]}{n}} \right),
\#
where we write $\hat m(\bbeta^{\A}) = \hat m(Y, S, U, A, B; \bbeta^{\A})$ for notational convenience. Note that
\#\label{eq:815-3}
\left | \|\hat m(\bbeta^{\A})\|^2 - \|\hat m(\hat \bbeta^{\A})\|^2 \right | & \leq \left\|\hat m(\bbeta^{\A}) + \hat m(\hat \bbeta^{\A}) \right\| \cdot \left\|\hat m(\bbeta^{\A}) - \hat m(\hat \bbeta^{\A}) \right\| \\
& \leq c\cdot \left\|\hat m(\bbeta^{\A}) - \hat m(\hat \bbeta^{\A}) \right\| = O_p\left( n^{-\frac{\alpha}{2(\alpha+\varsigma)+d}} \right),
\#
where we use Corollary \ref{cor:estimation} in the last inequality. 
By combining \eqref{eq:815-2} and \eqref{eq:815-3}, it holds that
\$
\hat L^{\A,\r}(\bbeta^{\A}) - \hat L^{\A,\r}(\hat \bbeta^{\A}) = O_p\left( n^{-\frac{2\alpha}{2(\alpha+\varsigma)+d}} \right), 
\$
which verifies \eqref{eq:pess-wtp-1}. Thus, it holds that $Q_1^{\A,\pi}\in \ci_1^{\A,\Q,\pi}$ as $n\to \infty$. 

To show that $Q_1^{\B,\pi}\in \ci_1^{\B,\Q,\pi}$, we need to show that there exists $\bbeta^{\B,\dagger}\in \ci^{\B,\r}$ such that $\btheta_1^\B = (\theta_1^\B, \gamma_1^\B, \omega_1^\B) \in \ci_1^{\B,\pi}(\bbeta^{\B, \dagger})$, where $(\theta_1^\B, \gamma_1^\B, \omega_1^\B)$ is associated with $Q_1^{\B,\pi}$. In other words, we need to show that there exists $\bbeta^{\B,\dagger}\in \ci^{\B,\r}$ such that $\btheta_1^{\B,j} = (\theta_1^{\B,j}, \gamma_1^{\B,j}, \omega_1^{\B,j}) \in \ci_1^{\B,j,\pi}(\bbeta^{\B, \dagger})$ for any $j\in [3]$. Following from a similar argument as in \eqref{eq:pess-wtp-1}, we can show that 
as $n\to\infty$, we have $\bbeta^\B\in \ci^{\B,\r}$, where $\bbeta^\B$ is associated with $Q^{\B,\pi}_{3/2}$. Thus, we only need to show that $(\theta_1^{\B,j}, \gamma_1^{\B,j}, \omega_1^{\B,j}) \in \ci_1^{\B,j,\pi}(\bbeta^{\B})$, i.e., 
\#\label{eq:pess-wtp-2}
\hat L_1^{\B,j,\pi}(\btheta_1^{\B,j}; \bbeta^\B) - \hat L_1^{\B,j,\pi}(\hat \btheta_1^{\B,j}; \bbeta^\B) \leq \eta_1^{\B,j}
\#
for any $j\in [3]$. Following from a similar argument to prove \eqref{eq:pess-wtp-1}, we have
\$
\hat L_1^{\B,j,\pi}(\btheta_1^{\B,j}; \bbeta^\B) - \hat L_1^{\B,j,\pi}(\hat \btheta_1^{\B,j}; \bbeta^\B) = O_p\left( n^{-\frac{2\alpha}{2(\alpha+\varsigma)+d}} \right), 
\$
which verifies \eqref{eq:pess-wtp-2} for any $j\in [3]$. Thus, we know that $Q_1^{\B,\pi}\in \ci_1^{\B,\Q,\pi}$ as $n\to\infty$. This concludes the proof of the lemma. 
\end{proof}

\subsection{Proof of Lemma \ref{lemma:bound}}\label{prf:lemma:bound}
\begin{proof}
By \eqref{eq:bellman-h-single} and the definition of the confidence region in \eqref{eq:conf-q-single}, we obtain that
\#\label{eq:uiwehf}
\left|\EE_{\pi^*}\left[ Q_1^{\A,\pi^*} - Q_1^\A \right]\right| & = \EE_{\pi^*}\left[ (\beta_a^\A-\tilde \beta_a^\A) A_1  + (\beta_z^\A-\tilde \beta_z^\A) B_{1/2} + (\beta_{az}^\A-\tilde\beta_{az}^\A) A_1 B_{1/2} \right] \\
& \leq \left\| \beta^\A_a - \tilde \beta^\A_a \right\|_{\pi^*} + \left\| \beta^\A_z - \tilde \beta^\A_z \right\|_{\pi^*} + \left\| \beta^\A_{az} - \tilde \beta^\A_{az} \right\|_{\pi^*},
\#
where $\tilde\bbeta^\A = (\tilde \beta^\A_a, \tilde \beta^\A_z, \tilde \beta^\A_{az})^\top$ is associated with $Q_1^\A$. Note that
\$
\hat L^{\A,\r}(\tilde \bbeta^{\A}) = \frac{1}{n} \sum_{i=1}^n \left\| \Phi_i \tilde \bbeta^\A_i + \alpha_i \right\|^2, 
\$
where $\Phi_i \in \RR^{3\times 3}$ and $\tilde \bbeta^\A_i = \tilde \bbeta^\A(S^i_1, U_1^{\A,i})$.
In the meanwhile, by Bernstein's inequality, we have
\#\label{eq:feiuwbr}
& \left(L^{\A,\r}(\tilde \bbeta^\A) - L^{\A,\r}(\bbeta^\A) \right) - \left(\hat L^{\A,\r}(\tilde \bbeta^\A) - \hat L^{\A,\r}(\bbeta^\A) \right) \\
& \qquad = O_p \left( \frac{1}{n} + \sqrt{\frac{\EE\left[(\|\Phi \tilde \bbeta^\A + \alpha \|^2 - \|\Phi \bbeta^\A + \alpha \|^2 )^2\right]}{n}} \right). 
\#
Note that
\#\label{eq:uehfuier}
& \EE\left[\left(\left\|\Phi \tilde \bbeta^\A+ \alpha \right\|^2 - \left\|\Phi \bbeta^\A+ \alpha \right\|^2 \right)^2\right] \leq c \cdot \EE\left[\left\|\Phi (\tilde \bbeta^\A - \bbeta^\A) \right\|^2  \right]. 
\#
By combining \eqref{eq:feiuwbr} and \eqref{eq:uehfuier}, we have
\#\label{eq:uwenfrervf}
& \left(L^{\A,\r}(\tilde \bbeta^\A) - L^{\A,\r}(\bbeta^\A) \right) - \left(\hat L^{\A,\r}(\tilde \bbeta^\A) - \hat L^{\A,\r}(\bbeta^\A) \right) \\
& \qquad = O_p \left( n^{-1} + \sqrt{n^{-1}\cdot\EE[\|\Phi (\tilde \bbeta^\A-\bbeta^\A) \|^2 ]} \right). 
\#
By \eqref{eq:uwenfrervf}, it holds for any $\tilde \bbeta^\A\in \ci^{\A,\r}$ that
\#\label{eq:iehwifer}
& L^{\A,\r}(\tilde \bbeta^\A) - L^{\A,\r}(\bbeta^\A) = O_p \left(  n^{-1} +  \sqrt{n^{-1}\cdot \EE[\|\Phi (\tilde \bbeta^\A-\bbeta^\A)\|^2 ]} + \eta^{\A,\r} \right),
\#
In the meanwhile, since $L^{\A,\r}(\cdot)$ is a quadratic functional, it holds for any $\tilde \bbeta^\A$ that
\#\label{eq:ehurfe}
\EE\left[\left\|\Phi (\tilde \bbeta^\A - \bbeta^{\A}) \right\|^2 \right] \leq L^{\A,\r}(\tilde \bbeta^\A) - L^{\A,\r}(\bbeta^\A). 
\#
Now, by combining \eqref{eq:iehwifer} and \eqref{eq:ehurfe}, we have
\$
\EE\left[\left\|\Phi (\tilde \bbeta^\A - \bbeta^{\A}) \right\|^2 \right] = O_p\left( n^{-\frac{2\alpha}{2\alpha+2\varsigma+d}} \right). 
\$
By Assumption \ref{ass:lower-bound-phi-paper}, we have
\$
\EE\left[\left\|\Phi (\tilde \bbeta^\A - \bbeta^{\A}) \right\|^2 \right] = \EE\left[\left(\tilde \bbeta^\A - \bbeta^{\A} \right )^\top  \EE\left[ \Phi^\top \Phi\right] \left(\tilde \bbeta^\A - \bbeta^{\A} \right)\right] \geq c\cdot \EE\left[\|\tilde \bbeta^\A - \bbeta^{\A} \|^2 \right]. 
\$
Thus, we know that 
\#\label{eq:uiwefhr}
\EE\left[\|\tilde \bbeta^\A - \bbeta^{\A} \|^2 \right] = O_p\left( n^{-\frac{2\alpha}{2\alpha+2\varsigma+d}} \right). 
\#
Combining \eqref{eq:uiwehf} and \eqref{eq:uiwefhr}, we have 
\$
\left|\EE_{\pi^*}\left[ Q_1^{\A,\pi^*} - Q_1^\A \right]\right| = O_p\left( n^{-\frac{\alpha}{2\alpha+2\varsigma+d}} \right). 
\$

Similarly, we can show for any $j\in [3]$ that 
\$
\EE\left[\|\tilde \bbeta_1^{\B,j} - \bbeta_1^{\B,j} \|^2 \right] = O_p\left( n^{-\frac{2\alpha}{2\alpha+2\varsigma+d}} \right), 
\$
which implies that 
\$
\EE\left[\|\tilde \bbeta_1^{\B} - \bbeta_1^{\B} \|^2 \right] = O_p\left( n^{-\frac{2\alpha}{2\alpha+2\varsigma+d}} \right). 
\$
Further, by the definition of confidence regions in \eqref{eq:conf-genius-single} and \eqref{eq:conf-q-single}, we know that 
\$
\left|\EE_{\pi^*}\left[ Q_1^{\B,\pi^*} - Q_1^\B \right]\right| = O_p\left( n^{-\frac{\alpha}{2\alpha+2\varsigma+d}} \right), 
\$
which concludes the proof of the lemma. 
\end{proof}

\section{Proof of Multi-Stage Game}

\subsection{Proof of Lemma \ref{lemma:pess-all}}\label{prf:lemma:pess-all}
\begin{proof}
It suffices to show for any $h$ and $\Diamond\in \{\A,\B\}$ that 
\#\label{eq:wtf-pess}
\lim_{n\to \infty} \PP\left(  \btheta_{h}^{\Diamond, j} \in 
\bigcup_{\substack{ \tilde \bbeta^{\Diamond} \in \ci^{\Diamond, \r}, \\ \tilde \btheta^\Diamond_{j}\in \ci_{j}^{\Diamond,\pi}(\tilde \btheta^\Diamond_{j+1/2}), \\ \text{for any $j\in \cC_h^{\Diamond}$} }  }
\ci_h^{\Diamond,j,\pi}(\tilde \btheta_{h+1/2}^{\Diamond}) \right ) = 1, 
\#
where $\cC_h^\A = \{h+1/2, h+1, \ldots, H-1/2\}$ and $\cC_h^\B = \{h+1/2, h+1, \ldots, H\}$. In the meanwhile, we denote by $\tilde \btheta_{H}^\A = \tilde \bbeta^\A$ and $\tilde \btheta_{H+1/2}^\B = \tilde \bbeta^\B$ for notational convenience. For the simplicity of presentation, we only show that \eqref{eq:wtf-pess} holds for $\Diamond  = \A$ as follows. We use mathematical induction as follows. 

\vskip5pt
\noindent\textbf{Base.} We want to show that $\lim_{n\to\infty} \PP( \bbeta^{\A} \in \ci^{\A, \r}) = 1$, i.e., 
\#\label{eq:wtf-pess-a1}
\lim_{n\to\infty} \PP\left( \hat L^{\A,\r}(\bbeta^{\A}, \hat \bff^{\A}) - \hat L^{\A,\r}(\hat \bbeta^{\A}, \hat \bff^{\A}) \leq \eta^{\A,\r}  \right) = 1. 
\#
To show that \eqref{eq:wtf-pess-a1} holds, we note that
\#\label{eq:fiuwerhfuie}
& \hat L^{\A,\r}(\bbeta^{\A}, \hat \bff^{\A}) - \hat L^{\A,\r}(\hat \bbeta^{\A}, \hat \bff^{\A}) \\
& \qquad = \left( \hat L^{\A,\r}(\bbeta^{\A}, \hat \bff^{\A}) - L^{\A,\r}(\bbeta^{\A}, \hat \bff^{\A}) \right) + \left( L^{\A,\r}(\bbeta^{\A}, \hat \bff^{\A}) - L^{\A,\r}(\bbeta^{\A}, \bff^{\A}) \right) \\
& \qquad \qquad + \left( L^{\A,\r}(\bbeta^{\A}, \bff^{\A}) - L^{\A,\r}(\hat \bbeta^{\A}, \hat \bff^{\A}) \right) + \left( L^{\A,\r}(\hat \bbeta^{\A}, \hat \bff^{\A}) - \hat L^{\A,\r}(\hat \bbeta^{\A}, \hat \bff^{\A}) \right) \\
& \qquad \leq \underbrace{\left( \hat L^{\A,\r}(\bbeta^{\A}, \hat \bff^{\A}) - \hat L^{\A,\r}(\hat \bbeta^{\A}, \hat \bff^{\A}) \right) - \left( L^{\A,\r}(\bbeta^{\A}, \hat \bff^{\A}) - L^{\A,\r}(\hat \bbeta^{\A}, \hat \bff^{\A})\right)}_{\text{term (A)}} \\
& \qquad \qquad + \underbrace{\left( L^{\A,\r}(\bbeta^{\A}, \hat \bff^{\A}) - L^{\A,\r}(\bbeta^{\A}, \bff^{\A}) \right)}_{\text{term (B)}}, 
\#
where we use the fact that $L^{\A,\r}(\bbeta^{\A}, \bff^{\A}) \leq L^{\A,\r}(\hat \bbeta^{\A}, \hat \bff^{\A})$ in the last inequality. 
Similar to the proof of Lemma \ref{lemma:pess} in \S\ref{prf:lemma:pess}, we upper bound term (A) using Bernstein's inequality as follows, 
\#\label{eq:ta-bound}
\text{term (A)} = O_p( n^{-\frac{2\alpha}{2\alpha+2\varsigma+d}}). 
\#
In the meanwhile, note that by Theorem \ref{thm:chen2012}, we have $\|\hat f^\A - f^\A\| = O_p( n^{-\frac{\alpha}{2\alpha+2\varsigma+d}})$. Thus, it holds that 
\#\label{eq:tb-bound}
\text{term (B)} = O_p( n^{-\frac{2\alpha}{2\alpha+2\varsigma+d}}). 
\#
Now, combining \eqref{eq:fiuwerhfuie}, \eqref{eq:ta-bound}, and \eqref{eq:tb-bound}, we obtain that 
\$
\hat L^{\A,\r}(\bbeta^{\A}, \hat \bff^{\A}) - \hat L^{\A,\r}(\hat \bbeta^{\A}, \hat \bff^{\A}) = O_p( n^{-\frac{2\alpha}{2\alpha+2\varsigma+d}}), 
\$
which implies that \eqref{eq:wtf-pess-a1} holds. 

\vskip5pt
\noindent\textbf{Hypothesis.} We assume that it holds for any $i\in\{H-1/2, H-1, \ldots, h+1/2\}$ that 
\#\label{eq:induction-hypo-11}
\lim_{n\to\infty} \PP\left( \btheta^{\A}_i \in \ci^{\A, \pi}_i(\btheta^{\A}_{i+1/2})\right) = 1. 
\#

\vskip5pt
\noindent\textbf{Induction.} With the hypothesis in \eqref{eq:induction-hypo-11}, it only suffices to show that $\lim_{n\to\infty} \PP( \btheta^{\A}_h \in \ci^{\A, \pi}_h(\btheta^{\A}_{h+1/2})) = 1$, i.e., 
\#\label{eq:fkjwneor}
\hat L_h^{\A,j,\pi}(\btheta_h^{\A,j}, \hat \bff_h^{\A,j}; \btheta_{h+1/2}^{\A}) - \hat L_h^{\A,j,\pi}(\hat \btheta_h^{\A,j}, \hat \bff_h^{\A,j}; \btheta_{h+1/2}^{\A}) \leq \eta_h^{\A,j} \qquad \text{for any $j\in [3]$}. 
\#
Similar to \eqref{eq:fiuwerhfuie}, we obtain the following upper bound for any $j\in [3]$, 
\$
& \hat L_h^{\A,j,\pi}(\btheta_h^{\A,j}, \hat \bff_h^{\A,j}; \btheta_{h+1/2}^{\A}) - \hat L_h^{\A,j,\pi}(\hat \btheta_h^{\A,j}, \hat \bff_h^{\A,j}; \btheta_{h+1/2}^{\A}) \\
& \qquad \leq \left( \hat L_h^{\A,j,\pi}(\btheta_h^{\A,j}, \hat \bff_h^{\A,j}; \btheta_{h+1/2}^{\A}) - \hat L_h^{\A,j,\pi}(\hat \btheta_h^{\A,j}, \hat \bff_h^{\A,j}; \btheta_{h+1/2}^{\A}) \right) \\
& \qquad \qquad - \left( L_h^{\A,j,\pi}(\btheta_h^{\A,j}, \hat \bff_h^{\A,j}; \btheta_{h+1/2}^{\A}) - L_h^{\A,j,\pi}(\hat \btheta_h^{\A,j}, \hat \bff_h^{\A,j}; \btheta_{h+1/2}^{\A})\right) \\
& \qquad \qquad + \left( L_h^{\A,j,\pi}(\btheta_h^{\A,j}, \hat \bff_h^{\A,j}; \btheta_{h+1/2}^{\A}) - L_h^{\A,j,\pi}(\btheta_h^{\A,j}, \bff_h^{\A,j}; \btheta_{h+1/2}^{\A}) \right), 
\$
which can be further upper bounded following similar arguments as in \eqref{eq:ta-bound} and \eqref{eq:tb-bound} as follows, 
\$
\hat L_h^{\A,j,\pi}(\btheta_h^{\A,j}, \hat \bff_h^{\A,j}; \btheta_{h+1/2}^{\A}) - \hat L_h^{\A,j,\pi}(\hat \btheta_h^{\A,j}, \hat \bff_h^{\A,j}; \btheta_{h+1/2}^{\A}) \leq \eta_h^{\A,j}. 
\$
This justifies \eqref{eq:fkjwneor}, which also justifies \eqref{eq:wtf-pess}. Thus we finish the mathematical induction. Then by a union bound argument, we have $\lim_{n\to\infty} \PP(Q_1^{\A,\pi}\in \ci_1^{\A,\Q,\pi}, Q_1^{\B,\pi}\in \ci_1^{\B,\Q,\pi} \text{ for any $\pi\in \Pi$}) = 1$, which concludes the proof of the lemma. 
\end{proof}

\subsection{Proof of Lemma \ref{lemma:bound-all}}\label{prf:lemma:bound-all}
\begin{proof}
Note that \#\label{eq:uiwehf11}
\left|\EE_{\pi^*}\left[ Q_1^{\A,\pi^*} - Q_1^\A \right]\right| & = \EE_{\pi^*}\left[ (\theta_1^{\A} - \tilde \theta_1^{\A}) A_1  + (\gamma_h^{\A}-\tilde \gamma_h^{\A}) B_{1/2} + (\omega_h^{\A}-\tilde\omega_h^{\A}) A_1 B_{1/2} \right] \\
& \leq \left\| \theta_1^{\A} - \tilde \theta_1^{\A} \right\|_{\pi^*} + \left\| \gamma_1^{\A} - \tilde \gamma_1^{\A} \right\|_{\pi^*} + \left\| \omega_1^{\A} - \tilde\omega_1^{\A} \right\|_{\pi^*},
\#
where $\tilde\btheta_1^\A = (\tilde \theta_1^{\A}, \tilde \gamma_1^{\A}, \tilde \omega_1^{\A})^\top$ is associated with $Q_1^\A$. 
We use mathematical induction to upper bound \eqref{eq:uiwehf11} as follows. 

\vskip5pt
\noindent\textbf{Base.} We upper bound $\EE[\|\tilde \bbeta^\A - \bbeta^{\A} \|^2]$ as follows. 
Note that
\$
\hat L^{\A,\r}(\tilde \bbeta^{\A}, \bff^\A) = \frac{1}{n} \sum_{i=1}^n \left\| \Phi_i(\bff^\A) \tilde \bbeta^\A_i + \alpha_i(\bff^\A)
\right\|^2, 
\$
where $\Phi_i(\bff^\A) \in \RR^{3\times 3}$ and $\tilde \bbeta^\A_i = \tilde \bbeta^\A(S^i_1, U_1^{\A,i})$.
In the meanwhile, by Bernstein's inequality, we have
\#\label{eq:feiuwbr11}
& \left(L^{\A,\r}(\tilde \bbeta^\A, \bff^\A) - L^{\A,\r}(\bbeta^\A, \bff^\A) \right) - \left(\hat L^{\A,\r}(\tilde \bbeta^\A, \bff^\A) - \hat L^{\A,\r}(\bbeta^\A, \bff^\A) \right) \\
& \qquad = O_p \left( \frac{1}{n} + \sqrt{\frac{\EE\left[(\|\Phi(\bff^\A) \tilde \bbeta^\A + \alpha(\bff^\A)\|^2 - \|\Phi(\bff^\A) \bbeta^\A +\alpha(\bff^\A)\|^2 )^2\right]}{n}} \right). 
\#
Note that
\#\label{eq:uehfuier11}
& \EE\left[\left(\left\|\Phi(\bff^\A) \tilde \bbeta^\A + \alpha(\bff^\A)\right\|^2 - \left\|\Phi(\bff^\A) \bbeta^\A + \alpha(\bff^\A)\right\|^2 \right)^2\right] \leq c \cdot \EE\left[\left\|\Phi(\bff^\A) (\tilde \bbeta^\A - \bbeta^\A) \right\|^2  \right]. 
\# 
By combining \eqref{eq:feiuwbr11} and \eqref{eq:uehfuier11}, we have
\#\label{eq:uwenfrervf11}
& \left(L^{\A,\r}(\tilde \bbeta^\A, \bff^\A) - L^{\A,\r}(\bbeta^\A, \bff^\A) \right) - \left(\hat L^{\A,\r}(\tilde \bbeta^\A, \bff^\A) - \hat L^{\A,\r}(\bbeta^\A, \bff^\A) \right) \\
& \qquad = O_p \left( n^{-1} + \sqrt{n^{-1}\cdot\EE[\|\Phi(\bff^\A) (\tilde \bbeta^\A-\bbeta^\A) \|^2 ]} \right). 
\#
In the meanwhile, we have 
\#\label{eq:fwheiurheru}
& \left(\hat L^{\A,\r}(\tilde \bbeta^\A, \bff^\A) - \hat L^{\A,\r}(\bbeta^\A, \bff^\A) \right) - \left(\hat L^{\A,\r}(\tilde \bbeta^\A, \hat \bff^\A) - \hat L^{\A,\r}(\bbeta^\A, \hat \bff^\A) \right) = O_p\left( n^{-\frac{2\alpha}{2\alpha+2\varsigma+d}} \right). 
\#
Combining \eqref{eq:uwenfrervf11} and \eqref{eq:fwheiurheru}, it holds for any $\tilde \bbeta^\A\in \ci^{\A,\r}$ that
\#\label{eq:iehwifer11}
& L^{\A,\r}(\tilde \bbeta^\A, \bff^\A) - L^{\A,\r}(\bbeta^\A, \bff^\A) = O_p \left( n^{-1} + \sqrt{n^{-1}\cdot \EE[\|\Phi(\bff^\A) (\tilde \bbeta^\A-\bbeta^\A)\|^2 ]} + n^{-\frac{2\alpha}{2\alpha+2\varsigma+d}} \right),
\#
In the meanwhile, since $L^{\A,\r}(\cdot)$ is a quadratic functional, it holds that
\#\label{eq:ehurfe11}
\EE\left[\left\|\Phi(\bff^\A) (\tilde \bbeta^\A - \bbeta^{\A}) \right\|^2 \right] \leq L^{\A,\r}(\tilde \bbeta^\A, \bff^\A) - L^{\A,\r}(\bbeta^\A, \bff^\A). 
\#
Now, by combining \eqref{eq:iehwifer11} and \eqref{eq:ehurfe11}, we have
\$
\EE\left[\left\|\Phi(\bff^\A) (\tilde \bbeta^\A - \bbeta^{\A}) \right\|^2 \right] = O_p\left(n^{-\frac{2\alpha}{2\alpha+2\varsigma+d}} \right). 
\$
By Assumption \ref{ass:lower-bound-phi-all}, we know that 
\$
\EE\left[\|\tilde \bbeta^\A - \bbeta^{\A} \|^2 \right] = O_p\left(n^{-\frac{2\alpha}{2\alpha+2\varsigma+d}} \right). 
\$

\vskip5pt
\noindent\textbf{Hypothesis.} We assume that it holds for any $i\in\{H-1/2, H-1, \ldots, h+1/2\}$ that 
\#\label{eq:uehurfehu}
\EE\left[\|\tilde \btheta^\A_i - \btheta^{\A}_i \|^2 \right] = O_p\left( (H-i)^3 \cdot n^{-\frac{2\alpha}{2\alpha+2\varsigma+d}} \right), 
\#
where $\tilde \btheta^\A_{i} \in 
\bigcup_{ \tilde \bbeta^{\A} \in \ci^{\A, \r}, \tilde \btheta^\A_{k}\in \ci_{k}^{\A,\pi}(\tilde \btheta^\A_{k+1/2}) \text{ for any $k\in \cC_i^{\A}$}}
\ci_{i}^{\A,k,\pi}(\tilde \btheta_{i+1/2}^{\A})$.

\vskip5pt
\noindent\textbf{Induction.} We want to upper bound $\EE[\|\tilde \btheta^{\A,j}_h - \btheta^{\A,j}_h \|^2 ]$ for any $j\in [3]$ as follows. 
By Bernstein's inequality, we have 
\#\label{eq:hfweiurh}
& \left(L^{\A,j,\pi}_h(\tilde \btheta^{\A,j}_h, \bff^{\A,j}_h; \btheta^{\A}_{h+1/2}) - L^{\A,j,\pi}_h(\btheta^{\A,j}_h, \bff^{\A,j}_h; \btheta^{\A}_{h+1/2}) \right) \\
& \qquad - \left(\hat L^{\A,j,\pi}_h(\tilde \btheta^{\A,j}_h, \bff^{\A,j}_h; \btheta^{\A}_{h+1/2}) - \hat L^{\A,j,\pi}_h(\btheta^{\A,j}_h, \bff^{\A,j}_h; \btheta^{\A}_{h+1/2}) \right) \\
& \qquad = O_p \left( \frac{(H-h)^2}{n} + \sqrt{\frac{\EE\left[\left(\|\Phi_h(\bff^\A) \tilde \btheta^{\A,j}_h + \alpha_h(\bff^\A) \|^2 - \|\Phi_h(\bff^\A) \btheta^{\A,j}_h + \alpha_h(\bff^\A) \|^2 \right)^2\right]}{n}} \right). 
\#
We upper bound the term $\hat L^{\A,j,\pi}_h(\tilde \btheta^{\A,j}_h, \bff^{\A,j}_h; \btheta^{\A}_{h+1/2}) - \hat L^{\A,j,\pi}_h(\btheta^{\A,j}_h, \bff^{\A,j}_h; \btheta^{\A}_{h+1/2})$ as follows. 
We assume that $\tilde \btheta^{\A,j}_h \in \ci_{h}^{\A,j,\pi}(\btheta^{\A,\dagger}_{h+1/2})$ for some $\btheta^{\A,\dagger}_{h+1/2} \in \bigcup_{ \tilde \bbeta^{\A} \in \ci^{\A, \r}, \tilde \btheta^\A_{k}\in \ci_{k}^{\A,\pi}(\tilde \btheta^\A_{k+1/2}) \text{ for any $k\in \cC_{h+1/2}^{\A}$}} \ci_{h+1/2}^{\A,k,\pi}(\tilde \btheta_{h+1}^{\A})$ and $\btheta^{\A,j}_h \in \ci_{h}^{\A,j,\pi}(\btheta^{\A,\ddagger}_{h+1/2})$ for some $\btheta^{\A,\ddagger}_{h+1/2} \in \bigcup_{ \tilde \bbeta^{\A} \in \ci^{\A, \r}, \tilde \btheta^\A_{k}\in \ci_{k}^{\A,\pi}(\tilde \btheta^\A_{k+1/2}) \text{ for any $k\in \cC_{h+1/2}^{\A}$}}\ci_{h+1/2}^{\A,k,\pi}(\tilde \btheta_{h+1}^{\A})$. 

Note that
\#\label{eq:fiuwehfuweir}
& \hat L^{\A,j,\pi}_h(\tilde \btheta^{\A,j}_h, \bff^{\A,j}_h; \btheta^{\A}_{h+1/2}) - \hat L^{\A,j,\pi}_h(\btheta^{\A,j}_h, \bff^{\A,j}_h; \btheta^{\A}_{h+1/2}) \\
& \qquad = \hat L^{\A,j,\pi}_h(\tilde \btheta^{\A,j}_h, \hat \bff^{\A,j}_h; \btheta^{\A}_{h+1/2}) - \hat L^{\A,j,\pi}_h(\btheta^{\A,j}_h, \hat \bff^{\A,j}_h; \btheta^{\A}_{h+1/2}) + O_p\left((H-h)^2 n^{-\frac{2\alpha}{2\alpha+2\varsigma+d}} \right). 
\#
In the meanwhile, we have
\#\label{eq:fuhweiurfwh}
& \hat L^{\A,j,\pi}_h(\tilde \btheta^{\A,j}_h, \hat \bff^{\A,j}_h; \btheta^{\A}_{h+1/2}) - \hat L^{\A,j,\pi}_h(\btheta^{\A,j}_h, \hat \bff^{\A,j}_h; \btheta^{\A}_{h+1/2}) \\
& \qquad \leq \left | \hat L^{\A,j,\pi}_h(\tilde \btheta^{\A,j}_h, \hat \bff^{\A,j}_h; \btheta^{\A}_{h+1/2}) - \hat L^{\A,j,\pi}_h(\tilde \btheta^{\A,j}_h, \hat \bff^{\A,j}_h; \btheta^{\A,\dagger}_{h+1/2}) \right | \\
& \qquad \qquad + \left | \hat L^{\A,j,\pi}_h(\tilde \btheta^{\A,j}_h, \hat \bff^{\A,j}_h; \btheta^{\A,\dagger}_{h+1/2}) - \hat L^{\A,j,\pi}_h(\hat \btheta^{\A,j}_h, \hat \bff^{\A,j}_h; \btheta^{\A,\dagger}_{h+1/2}) \right |\\
& \qquad \qquad + \left | \hat L^{\A,j,\pi}_h(\hat \btheta^{\A,j}_h, \hat \bff^{\A,j}_h; \btheta^{\A,\dagger}_{h+1/2})  - \hat L^{\A,j,\pi}_h(\hat \btheta^{\A,j}_h, \hat \bff^{\A,j}_h; \btheta^{\A,\ddagger}_{h+1/2}) \right | \\
& \qquad \qquad + \left | \hat L^{\A,j,\pi}_h(\hat \btheta^{\A,j}_h, \hat \bff^{\A,j}_h; \btheta^{\A,\ddagger}_{h+1/2})  - \hat L^{\A,j,\pi}_h(\btheta^{\A,j}_h, \hat \bff^{\A,j}_h; \btheta^{\A,\ddagger}_{h+1/2}) \right |\\
& \qquad \qquad + \left | \hat L^{\A,j,\pi}_h(\btheta^{\A,j}_h, \hat \bff^{\A,j}_h; \btheta^{\A,\ddagger}_{h+1/2}) - \hat L^{\A,j,\pi}_h(\btheta^{\A,j}_h, \hat \bff^{\A,j}_h; \btheta^{\A}_{h+1/2}) \right |. 
\#
The first term on the RHS of \eqref{eq:fuhweiurfwh} can be upper bounded by \eqref{eq:uehurfehu} as follows, 
\#\label{eq:tydfwytedtf1}
\left | \hat L^{\A,j,\pi}_h(\tilde \btheta^{\A,j}_h, \hat \bff^{\A,j}_h; \btheta^{\A}_{h+1/2}) - \hat L^{\A,j,\pi}_h(\tilde \btheta^{\A,j}_h, \hat \bff^{\A,j}_h; \btheta^{\A,\dagger}_{h+1/2}) \right | = O_p\left( (H-h-1/2)^3 \cdot n^{-\frac{2\alpha}{2\alpha+2\varsigma+d}} \right). 
\#
The second term on the RHS of \eqref{eq:fuhweiurfwh} can be upper bounded by the definition of the confidence regions as follows, 
\#\label{eq:tydfwytedtf2}
\left | \hat L^{\A,j,\pi}_h(\tilde \btheta^{\A,j}_h, \hat \bff^{\A,j}_h; \btheta^{\A,\dagger}_{h+1/2}) - \hat L^{\A,j,\pi}_h(\hat \btheta^{\A,j}_h, \hat \bff^{\A,j}_h; \btheta^{\A,\dagger}_{h+1/2}) \right | \leq \eta_h^{\A,j}. 
\#
The third term on the RHS of \eqref{eq:fuhweiurfwh} can be upper bounded by \eqref{eq:uehurfehu} as follows, 
\#\label{eq:tydfwytedtf3}
\left | \hat L^{\A,j,\pi}_h(\hat \btheta^{\A,j}_h, \hat \bff^{\A,j}_h; \btheta^{\A,\dagger}_{h+1/2})  - \hat L^{\A,j,\pi}_h(\hat \btheta^{\A,j}_h, \hat \bff^{\A,j}_h; \btheta^{\A,\ddagger}_{h+1/2}) \right | = O_p\left( (H-h-1/2)^3 \cdot n^{-\frac{2\alpha}{2\alpha+2\varsigma+d}} \right). 
\#
The forth term on the RHS of \eqref{eq:fuhweiurfwh} can be upper bounded by definition of the confidence regions and Lemma \ref{lemma:pess-all} as follows, 
\#\label{eq:tydfwytedtf4}
\left | \hat L^{\A,j,\pi}_h(\hat \btheta^{\A,j}_h, \hat \bff^{\A,j}_h; \btheta^{\A,\ddagger}_{h+1/2})  - \hat L^{\A,j,\pi}_h(\btheta^{\A,j}_h, \hat \bff^{\A,j}_h; \btheta^{\A,\ddagger}_{h+1/2}) \right | \leq \eta_h^{\A,j}. 
\#
The last term on the RHS of \eqref{eq:fuhweiurfwh} can be upper bounded by \eqref{eq:uehurfehu} as follows, 
\#\label{eq:tydfwytedtf5}
\left | \hat L^{\A,j,\pi}_h(\btheta^{\A,j}_h, \hat \bff^{\A,j}_h; \btheta^{\A,\ddagger}_{h+1/2}) - \hat L^{\A,j,\pi}_h(\btheta^{\A,j}_h, \hat \bff^{\A,j}_h; \btheta^{\A}_{h+1/2}) \right | = O_p\left( (H-h-1/2)^3 \cdot n^{-\frac{2\alpha}{2\alpha+2\varsigma+d}} \right). 
\#
Now, by plugging \eqref{eq:fuhweiurfwh}, \eqref{eq:tydfwytedtf1}, \eqref{eq:tydfwytedtf2}, \eqref{eq:tydfwytedtf3}, \eqref{eq:tydfwytedtf4}, and \eqref{eq:tydfwytedtf5} into \eqref{eq:fiuwehfuweir}, we have
\#\label{eq:brvreu}
& \hat L^{\A,j,\pi}_h(\tilde \btheta^{\A,j}_h, \bff^{\A,j}_h; \btheta^{\A}_{h+1/2}) - \hat L^{\A,j,\pi}_h(\btheta^{\A,j}_h, \bff^{\A,j}_h; \btheta^{\A}_{h+1/2}) = O_p\left( (H-h)^3 \cdot n^{-\frac{2\alpha}{2\alpha+2\varsigma+d}} \right).
\#
Note that
\#\label{eq:uehfuier11dede}
& \EE\left[\left(\|\Phi_h(\bff^\A) \tilde \btheta^{\A,j}_h + \alpha_h(\bff^\A) \|^2 - \|\Phi_h(\bff^\A) \btheta^{\A,j}_h + \alpha_h(\bff^\A) \|^2 \right)^2 \right] \leq c (H-h)^2 \cdot \EE\left[\left\|\Phi_h(\bff^\A) (\tilde \btheta^{\A,j}_h - \btheta^{\A,j}_h) \right\|^2  \right]. 
\#
Combining \eqref{eq:hfweiurh}, \eqref{eq:brvreu}, and \eqref{eq:uehfuier11dede}, we have 
\$
& \left(L^{\A,j,\pi}_h(\tilde \btheta^{\A,j}_h, \bff^{\A,j}_h; \btheta^{\A}_{h+1/2}) - L^{\A,j,\pi}_h(\btheta^{\A,j}_h, \bff^{\A,j}_h; \btheta^{\A}_{h+1/2}) \right) \\
& \qquad = O_p\left( \frac{(H-h)^2}{n} + (H-h) \sqrt{\frac{1}{n} \cdot \EE\left[\left\|\Phi_h(\bff^\A) (\tilde \btheta^{\A,j}_h - \btheta^{\A,j}_h) \right\|^2  \right]} + (H-h)^3\cdot n^{-\frac{2\alpha}{2\alpha+2\varsigma+d}} \right), 
\$
which implies that 
\$
\EE\left[\left\|\Phi_h(\bff^\A) (\tilde \btheta^{\A,j}_h - \btheta^{\A,j}_h) \right\|^2  \right] = O_p\left( (H-h)^3 \cdot n^{-\frac{2\alpha}{2\alpha+2\varsigma+d}} \right)
\$
following a similar argument as in the base case. Since the matrix $\Phi_h(\bff^\A; \btheta^{\A}_{h+1/2})$ is lower bounded, we know that 
\$
\EE\left[\left\|\tilde \btheta^{\A,j}_h - \btheta^{\A,j}_h \right\|^2  \right] = O_p\left( (H-h)^3 \cdot n^{-\frac{2\alpha}{2\alpha+2\varsigma+d}} \right), 
\$
which implies that
\$
\sqrt{\EE\left[\|\tilde \btheta^\A_h - \btheta^{\A}_h \|^2 \right]} = O_p\left( (H-h)^{3/2} \cdot n^{-\frac{\alpha}{2\alpha+2\varsigma+d}} \right).
\$
Thus, we have 
\$
\left |\EE_{\pi}\left[ Q_1^{\A,\pi} - Q_1^\A \right] \right | = O_p\left( H^2 \cdot n^{-\frac{\alpha}{2\alpha+2\varsigma+d}} \right), \qquad \left |\EE_{\pi}\left[ Q_1^{\B,\pi} - Q_1^\B \right] \right | = O_p\left( H^2 \cdot n^{-\frac{\alpha}{2\alpha+2\varsigma+d}} \right)
\$
by a union bound argument, which concludes the proof of the lemma. 
\end{proof}

\subsection{Proof of Theorem \ref{thm:main-H}}\label{prf:thm:main-H}
\begin{proof}
We introduce two supporting lemmas. 
The following lemma shows that with the sizes of the confidence regions in \eqref{eq:iowuehfurew}, the true action-value functions $Q_1^{\A,\pi}$ and $Q_1^{\B,\pi}$ lie in the confidence regions.

\begin{lemma}\label{lemma:pess-all}
Under the assumptions stated in Theorem \ref{thm:chen2012}, it holds that  $\lim_{n\to\infty} \PP(Q_1^{\A,\pi}\in \ci_1^{\A,\Q,\pi}, Q_1^{\B,\pi}\in \ci_1^{\B,\Q,\pi} \text{ for any $\pi\in \Pi$}) = 1$. 
\end{lemma}
\begin{proof}
See \S\ref{prf:lemma:pess-all} for a detailed proof. 
\end{proof}

The following lemma shows that with the sizes of the confidence regions in \eqref{eq:iowuehfurew}, for any functions in the confidence regions, they are close to the true action-value functions $Q_1^{\A,\pi}$ and $Q_1^{\B,\pi}$. 

\begin{lemma}\label{lemma:bound-all}
Under Assumption \ref{ass:lower-bound-phi-all} and the assumptions stated in Theorem \ref{thm:chen2012}, for any policy $\pi\in \Pi$, $Q_1^\A\in \ci_1^{\A,\Q,\pi}$, and $Q_1^\B\in \ci_1^{\B,\Q,\pi}$, we have 
\$
\left |\EE_{\pi}\left[ Q_1^{\A,\pi} - Q_1^\A \right] \right | = O_p\left( H^{2} \cdot n^{-\frac{\alpha}{2\alpha+2\varsigma+d}} \right), \qquad \left |\EE_{\pi}\left[ Q_1^{\B,\pi} - Q_1^\B \right] \right | = O_p\left( H^{2} \cdot n^{-\frac{\alpha}{2\alpha+2\varsigma+d}} \right). 
\$
\end{lemma}
\begin{proof}
See \S\ref{prf:lemma:bound-all} for a detailed proof. 
\end{proof}

Note that 
\$
& J(\pi^{*,\A}, \pi^{*,\B}) - J(\hat \pi^\A, \hat \pi^\B) \\
& \qquad = \EE_{\pi^*}\left[ (Q_1^{\A,\pi^*} + Q_1^{\B,\pi^*})(S_1, A_1, B_{1/2}, U_1) \right] - \EE_{\hat \pi}\left[ (Q_1^{\A,\hat \pi} + Q_1^{\B,\hat \pi})(S_1, A_1, B_{1/2}, U_1) \right] \\
& \qquad \leq \EE_{\pi^*}\left[ (Q_1^{\A,\pi^*} + Q_1^{\B,\pi^*})(S_1, A_1, B_{1/2}, U_1) \right] - \min_{Q_1^\A \in \ci_1^{\A,\Q,\hat \pi}, Q_1^\B \in \ci_1^{\B,\Q,\hat \pi}}\EE_{\hat \pi}\left[ (Q_1^\A + Q_1^\B)(S_1, A_1, B_{1/2}, U_1) \right] \\
& \qquad \leq \EE_{\pi^*}\left[ (Q_1^{\A,\pi^*} + Q_1^{\B,\pi^*})(S_1, A_1, B_{1/2}, U_1) \right] - \min_{Q_1^\A \in \ci_1^{\A,\Q,\pi^*}, Q_1^\B \in \ci_1^{\B,\Q,\pi^*}} \EE_{\pi^*}\left[ (Q_1^\A + Q_1^\B)(S_1, A_1, B_{1/2}, U_1) \right] \\
& \qquad \leq \max_{Q_1^\A \in \ci_1^{\A,\Q,\pi^*}, Q_1^\B \in \ci_1^{\B,\Q,\pi^*}} \left | \EE_{\pi^*}\left[ (Q_1^{\A,\pi^*} - Q_1^\A + Q_1^{\B,\pi^*} - Q_1^\B)(S_1, A_1, B_{1/2}, U_1) \right]  \right |, 
\$
where in the first and second inequalities, we use Lemma \ref{lemma:pess-all} and the optimality of $(\hat \pi^\A, \hat \pi^\B)$, respectively.
Further, by Lemma \ref{lemma:bound-all}, we have 
\$
J(\pi^{*,\A}, \pi^{*,\B}) - J(\hat \pi^\A, \hat \pi^\B) = O_p\left( H^2 \cdot n^{-\frac{\alpha}{2\alpha+2\varsigma+d}} \right),
\$
which concludes the proof of the theorem. 

\end{proof}

\end{APPENDICES}


\bibliographystyle{ims}
\bibliography{rl_ref}

\end{document}